%% file: main.tex
\documentclass{statsoc}

\usepackage{fancyhdr}

\usepackage{geometry} 
\usepackage{amsmath}
\usepackage{amsfonts}
\usepackage{bbm}
\usepackage{natbib}
\usepackage{graphicx}
\usepackage{thmtools,thm-restate}
\usepackage{multibib}

\usepackage{etoolbox}

\makeatletter
\patchcmd{\@makecaption}
  {\parbox}
  {\advance\@tempdima-\fontdimen2} 
  {}{}
\makeatother 

\input{macros}

\newcites{cont}{References}
\newcites{appdx}{References}

\title[From Denoising Diffusions to Denoising Markov Models]{From Denoising Diffusions to Denoising Markov Models}

\author{Joe Benton}
\coaddress{Joe Benton, Department of Statistics, University of Oxford, 24-29 St Giles', Oxford, OX1 3LB, UK. E-mail: \texttt{benton@stats.ox.ac.uk}}
\address{Department of Statistics, University of Oxford, Oxford, UK}
\author{Yuyang Shi}
\address{Department of Statistics, University of Oxford, Oxford, UK}
\author{Valentin De Bortoli}
\address{ENS, Paris, France}
\author{George Deligiannidis}
\address{Department of Statistics, University of Oxford, Oxford, UK}
\author[Benton et al.]{Arnaud Doucet}
\address{Department of Statistics, University of Oxford, Oxford, UK}

\begin{document}

\thispagestyle{fancy}

\begin{abstract}
Denoising diffusions are state-of-the-art generative models exhibiting remarkable empirical performance. They work by diffusing the data distribution into a Gaussian distribution and then learning to reverse this noising process to obtain synthetic datapoints. The denoising diffusion relies on approximations of the logarithmic derivatives of the noised data densities using score matching. Such models can also be used to perform approximate posterior simulation when one can only sample from the prior and likelihood. We propose a unifying framework generalising this approach to a wide class of spaces and leading to an original extension of score matching. We illustrate the resulting models on various applications.
\end{abstract}

\keywords{denoising diffusions, generative models, posterior simulation, score matching, unifying framework}

\input{cont}
\bibliographystylecont{chicago}
\bibliographycont{cont}

\newpage

\input{appdx}

\newpage

\bibliographystyleappdx{chicago}
\bibliographyappdx{appdx}

\end{document}

%% file: macros.tex
\newcommand{\extW}{\overline{W}}
\newcommand{\extX}{\overline{X}}
\newcommand{\extY}{\overline{Y}}
\newcommand{\extZ}{\overline{Z}}

\newtheorem{lemma}{Lemma}
\newtheorem{theorem}{Theorem}
\newtheorem{definition}{Definition}
\newtheorem{assumption}{Assumption}

\newtheorem{example}{Example}

\DeclareMathOperator*{\Tr}{Tr}

\newcommand{\bb}[1]{\mathbb{#1}}
\renewcommand{\cal}[1]{\mathcal{#1}}
\newcommand{\mb}[1]{\mathbf{#1}}

\newcommand{\A}{\cal{A}}

\newcommand{\D}{\cal{D}}
\newcommand{\E}[2][]{\bb{E}_{#1} \left[ #2 \right]}

\newcommand{\I}{\cal{I}}
\newcommand{\J}{\cal{J}}
\newcommand{\K}{\cal{K}}
\renewcommand{\L}{\cal{L}}
\newcommand{\M}{\cal{M}}
\renewcommand{\P}{\bb{P}}
\newcommand{\p}{\mb{p}}
\newcommand{\N}{\cal{N}}
\newcommand{\Q}{\bb{Q}}
\newcommand{\R}{\bb{R}}
\newcommand{\w}{\mb{w}}
\newcommand{\X}{\cal{X}}
\newcommand{\x}{\mb{x}}

\newcommand{\y}{\mb{y}}
\newcommand{\z}{\mb{z}}
\newcommand{\Z}{\bb{Z}}
\newcommand{\xiobs}{\boldsymbol{\xi}}

\renewcommand{\d}{\mathrm{d}}
\renewcommand{\i}{\mathbf{i}}
\renewcommand{\j}{\mathbf{j}}
\renewcommand{\k}{\mathbf{k}}
\newcommand{\data}{\textup{data}}

\newcommand{\KL}[2]{\textup{KL}(#1 || #2)}

\newcommand{\const}{\textit{const}}
\newcommand{\del}{\partial}

\newcommand{\lr}[1]{\left(#1\right)}
\newcommand{\lrb}[1]{\left[#1\right]}
\newcommand{\lrcb}[1]{\left\{#1\right\}}

%% file: cont.tex
\section{Introduction}

Given a set of samples from an unknown distribution $p_{\data}(\x)$, generative modelling is the task of producing further synthetic samples coming from approximately the same distribution. Over the past decade, a variety of techniques have been developed to tackle this problem, including autoregressive models \citepcont{oord2016wavenet}, generative adversarial networks \citepcont{goodfellow2014generative}, variational autoencoders \citepcont{kingma2014autoencoding} and normalising flows \citepcont{rezende2015variational}. These methods have had significant success in generating perceptually realistic samples from complex data distributions, such as text and image data \citepcont{brown2020language,dhariwal2021diffusion}. A major motivation for the development of generative models is that they can be easily extended for Bayesian inference. In a typical setting, we make an observation $\xiobs^\ast$ based on underlying datapoint $\x$, for example a category label or partial observation of $\x$, and want to sample from the posterior distribution $p_{\data}(\x | \xiobs^\ast)$. We achieve this by learning a conditional generative model for $\x$ given any observation $\xiobs$ based on samples from $p_{\data}(\x,\xiobs)$. This approach is particularly useful in high-dimensional scenarios where traditional sampling methods, such as Markov chain Monte Carlo (MCMC) methods or approximate Bayesian computation (ABC), are typically infeasible.

Recently, denoising diffusion models \citepcont{sohldickstein2015deep, ho2020denoising, song2021score} have emerged as effective generative models for high-dimensional data. They work by incrementally adding noise to the data to transform the data distribution into an easy-to-sample reference distribution, and then learning to invert the noising process, which is achieved using score matching \citep{hyvarinen2005estimation}. Their use for inference has recently seen an explosion of applications, including text-to-speech generation \citepcont{popov2021grad}, image inpainting and super-resolution \citepcont{song2021score, saharia2022image} and protein structure modelling \citepcont{trippe2023diffusion}. 

Most of the current methodology, theory and applications of denoising diffusion models are for diffusion processes on $\R^d$. However, many distributions of interest are defined on different spaces. Recently, \citecont{debortoli2022riemannian} and \citecont{huang2022riemannian} have extended continuous-time methods and the analogy with score matching from $\R^d$ to general Riemannian manifolds in order to model data with strong geometric prior. Several diffusion methods have also been developed for discrete data, such as text, music or graph structures \citepcont{austin2021structured,hoogeboom2021argmax,campbell2022continuous, sun2023score}. Here though, the relationships to score matching, as well as between these various methods and the Euclidean diffusion case, are less clear. All these recent extensions have been somewhat ad hoc, with training objectives needing to be re-derived for each new application.

The main contribution of this paper is to provide a unifying framework for such models, which we call \emph{denoising Markov models}, or DMMs. We demonstrate how to construct and train a DMM for data in any state space satisfying mild regularity conditions. This yields a principled procedure for using these models for unconditional generation and inference on a wider class of spaces than previously considered. Additionally this general framework leads to a principled extension of score matching to general spaces. Finally, we demonstrate the application of our framework on examples in continuous Euclidean space, discrete space, for Riemmanian manifolds and on the simplex.

\section{Background}
A denoising diffusion model is a generative model consisting of two stochastic processes. The \emph{fixed} noising process takes a data point $\x_0$ drawn from a data distribution $q_0 := p_{\data}$ on state space $\X$ and maps it stochastically to some $\x_T \in \X$. The \emph{learned} generative process takes $\x_T \in \X$ drawn according to some initial distribution $p_0$ on $\X$ and maps it back stochastically to some $\x_0 \in \X$. Throughout, we denote the marginals of the noising and generative processes by $q_t(\x)$ and $p_t(\x)$ respectively for $t \in [0,T]$.

The basic idea is to pick a noising process so that $(q_t)_{t \geq 0}$ converges to some easy-to-sample-from distribution $q_{\textup{ref}}$, which we then take to be $p_0$. We learn a generative process which approximates the time-reversal of the noising process. Then, we can generate approximate samples from $q_0$ by sampling $\x_T \sim p_0$ and running the dynamics of the reverse process to produce a sample $\x_0 \sim p_T$, which should be close to $q_0$.

\vspace{-0.5cm}
\subsection{Continuous-time denoising diffusion models on $\R^d$}
\label{sec:diffusionbackground}

The framework for continuous-time diffusion models on $\R^d$ was first set out by \citecont{song2021score}. The ``forward'' noising process $(Y_t)_{t \in [0,T]}$ evolves according to the stochastic differential equation (SDE)
\begin{equation}
\label{eq:forward_SDE}
    \d Y_t = b(Y_t, t) \d t + \d B_t, \qquad Y_0 = \x_0 \sim p_{\data},
\end{equation}
for some chosen function $b: \R^d \times [0,T] \rightarrow \R^d$, and standard Brownian motion $B$. With this set-up, the time-reversed process $X_t = Y_{T-t}$ can be simulated by initialising $X_0 = \x_T \sim q_T$ and running the SDE
\begin{equation}
\label{eq:reverse_SDE}
    \d X_t = \{ - b(X_t, T-t) + \nabla_\x \log q_{T-t}(X_t) \} \d t + \d\hat B_t,
\end{equation}
where $q_t(\x_t)$ denotes the marginals of the forward process and $\hat B$ is another standard Brownian motion \citepcont{anderson1982reverse}. We typically choose our forward process to be an Ornstein--Uhlenbeck process, i.e. $b(\x, t) = -\x/2$, for which $q_T \approx q_{\textup{ref}} := \mathcal{N}(0,I_d)$, the standard Gaussian distribution on $\R^d$, for large $T$.

To simulate the reverse process, we must approximate $\nabla_\x \log q_t(\x)$. We do this by fixing a parametric family of functions $s_\theta(\x, t)$, and then choosing the parameters $\theta$ to minimise the \emph{denoising score matching} objective 
\begin{equation}
\label{eq:score_matching_objective}
    \I_{\textup{DSM}}(\theta) = \frac{1}{2} \int_0^T \E[q_{0,t}(\x_0, \x_t)]{||\nabla_\x \log q_{t|0}(\x_t | \x_0) - s_\theta(\x_t, t)||^2} \; \d t,
\end{equation}
where $q_{0,t}(\x_0,\x_t)$ and $q_{t | 0}(\x_t|\x_0)$ denote the joint and conditional distributions of the SDE \eqref{eq:forward_SDE}. The conditional is available in closed-form for the Ornstein--Uhlenbeck process. This is sensible since $\I_{\textup{DSM}}$ is minimised when $s_\theta(\x,t) = \nabla_\x \log q_t(\x)$ for almost all $x \in \X$ and $t \in [0,T]$ \citepcont{song2021score}. If our score estimate were exact and $p_0 = q_T$, then we would have $p_t = q_{T-t}$ for all $t \in [0,T]$. In practice, we use a neural network to parameterise $s_\theta(\x,t)$ and use stochastic gradient descent to minimise $\I_{\textup{DSM}}(\theta)$.

Once we have a score estimate $s_\theta(\x, t)$, we compute approximate samples from the reverse process by running the approximate reverse process
\begin{equation}
\label{eq:approxreverse}
    \d X_t = \{ - b(X_t, T-t) + s_\theta(X_t,T-t) \} \d t + \d\hat B_t
\end{equation}
starting in $X_0 \sim p_0$ and setting $\x_0 = X_T$. In practice, we use suitable numerical integrators to simulate the approximate reverse process.

Alternatively, the objective $\cal{I}_{\textup{DSM}}$ can be derived from a lower bound on the model log-likelihood (also known as an Evidence Lower Bound, or ELBO) for $q_T(x)$, either using Girsanov's theorem and the chain rule for Kullback--Leibler divergences \citepcont{song2021maximum}, or by combining the Fokker--Planck equation and Feynman--Kac formula with Girsanov's theorem \citepcont{huang2021variational}.

\vspace{-0.5cm}
\subsection{Diffusion models for inference}
\label{sec:inferencebackground}

Denoising diffusions can also be used to sample approximately from a posterior $p_{\data}(\x |\xiobs^\ast)$ when we only have access to samples from the joint distribution $p_{\data}(\x ,\xiobs)$; see e.g. \citepcont{song2021score}. We first draw a sample $(\x_0, \xiobs_0) \sim p_{\data}$, set $Y_0 = \x_0$ and let $(Y_t)_{t \in [0,T]}$ evolve according to Equation (\ref{eq:forward_SDE}). If we condition on $\xiobs_0$, then the process $Y$ has marginals $q_t(\x_t | \xiobs_0) = \int q_{t|0}(\x_t | \x_0) p_{\data}(\x_0 | \xiobs_0) \d \x_0$, where $q_{t|0}(\x_t | \x_0)$ is the transition kernel of the forward diffusion in Equation (\ref{eq:forward_SDE}). So, the time-reversed process $X_t = Y_{T-t}$ conditioned on $\xiobs_0$ can be simulated by initialising $X_0 \sim q_T(\cdot | \xiobs_0)$ and running the SDE
\begin{equation}
    \label{eq:reverse_SDE_conditional}
    \d X_t = \{ - b(X_t, T-t) + \nabla_{\x} \log q_{T-t}(X_t \;|\; \xiobs_0) \} \d t + \d\hat B_t.
\end{equation}
If we have $q_T(\cdot | \xiobs) \approx q_{\textup{ref}}$ for all $\xiobs$ and an approximation $s_\theta(\x, \xiobs, t)$ to $\nabla_{\x} \log q_t(\x | \xiobs)$, we can obtain approximate samples from $q_0(\cdot | \xiobs^\ast) = p_{\data}(\cdot | \xiobs^\ast)$ for any given $\xiobs^\ast$ by initialising $X_0 \sim p_0 := q_{\textup{ref}}$, simulating the reverse dynamics in Equation (\ref{eq:reverse_SDE_conditional}) with $\nabla_{\x} \log q_{T-t}(X_t | \xiobs_0)$ replaced by $s_\theta(X_t, \xiobs^\ast, T-t)$, and setting $\x_0 = X_T$. To learn $s_\theta(\x, \xiobs, t)$, we minimise
\begin{equation*}
   \I_{\textup{DSM}}(\theta) = \frac{1}{2} \int_0^T \E[q(\x_0, \x_t, \xiobs_0)]{||\nabla_\x \log q_{t|0}(\x_t | \x_0) - s_\theta(\x_t, \xiobs_0, t)||^2} \; \d t,
\end{equation*}
where we denote $q(\x_0, \x_t, \xiobs_0) = p_{\data}(\x_0, \xiobs_0)q_{t|0}(\x_t | \x_0)$. This objective is minimised when $s_\theta(\x, \xiobs, t) = \nabla_{\x} \log q_t(\x| \xiobs)$ for almost all $x \in \X$ and $t \in [0,T]$ \citepcont{song2021score}.

\vspace{-0.5cm}
\subsection{Score matching}

The objective $\I_{\textup{DSM}}$ defined in Equation (\ref{eq:score_matching_objective}) can also be interpreted as a score matching objective. Score matching was introduced as a method for fitting unnormalised probability distributions defined on $\R^d$ by \citecont{hyvarinen2005estimation}. It approximates a distribution $q_0(\x)$ with a distribution of the form $p(\x ; \theta) = q(\x; \theta)/Z(\theta)$ by minimising
\begin{equation*}
    \J(\theta) = \frac{1}{2} \E[q_0(\x)]{||\nabla_\x \log q_0(\x) - \nabla_\x \log q(\x; \theta)||^2},
\end{equation*}
known as an \emph{explicit score matching} loss. This objective is intractable since it depends on $\nabla_\x \log q_0(\x)$, but there are methods for rewriting it in an equivalent tractable form, including implicit and denoising score matching \citepcont{hyvarinen2005estimation, vincent2011connection}. Equation (\ref{eq:score_matching_objective}), which corresponds to denoising score matching, can also be written in explicit, implicit or sliced score matching form \citepcont{huang2021variational}.

\section{A general framework for denoising Markov models}
\label{sec:general_objective}

In this section, we set out a general framework for DMMs. First, we explain how to construct a DMM on an arbitrary state space with a forward noising process $Y$ and backward generative process $X$. Second, we derive an expression for the model likelihood in terms of an expectation over an auxiliary process $Z$, defined in terms of $X$ and running forward in time. Third, we derive an ELBO by using Girsanov's theorem to relate the expectation over $Z$ to one over $Y$. Finally, we show how this ELBO can be used to get a tractable training objective. Our argument follows a similar structure to \citecont{huang2021variational}, but we work in terms of generic Markov generators, rather than specific operators corresponding to diffusions on $\R^d$, and so require generalisations of the stochastic process results therein. For simplicity, we present the framework for unconditional generation and then explain how to adapt it for inference. 
\vspace{-0.5cm}
\subsection{Notation and set-up}
\label{sec:setup}
Our data is assumed to be distributed according to $p_{\data}$ on a state space $\X$. We assume only that $\X$ comes with some reference measure $\nu$, with respect to which all probability densities will be defined, and satisfies some regularity conditions given in Appendix \ref{app:assstatespace}. This includes $\R^d$, discrete spaces and Riemannian manifolds (with or without boundary).

Our DMM consists of a noising process $(Y_t)_{t \in [0,T]}$ and a generative process $(X_t)_{t \in [0,T]}$, which are Markov processes. We consider $Y$ fixed and learn $X$ to approximate the reverse of $Y$. Initially, we must fix a class of processes to which $X$ and $Y$ belong and within which we will optimise $X$. The particular class and parameterisation we choose will necessarily depend on $\X$, but a typical choice for $\X = \R^d$ would be a diffusion (see Example \ref{ex:realdiffusiongenerator}), while a typical choice when $\X$ is a finite discrete space may be a continuous-time Markov chain (CTMC) (see Example \ref{ex:CTMCgenerator}). Our notation is depicted in Fig.\ \ref{fig:process_diagram}.

As $X$ and $Y$ are not necessarily time-homogeneous, it is helpful to define the extended processes $\extX$ and $\extY$ by for example setting $X_t = X_T$ for $t \geq T$ and letting $\extX = (X_t, t)_{t \geq 0}$. Then $\extX$, $\extY$ are time-homogeneous Markov chains on the extended space $\cal{S} := \X \times [0, \infty)$.

\begin{figure}
    \centering
    \includegraphics[width=0.6\textwidth]{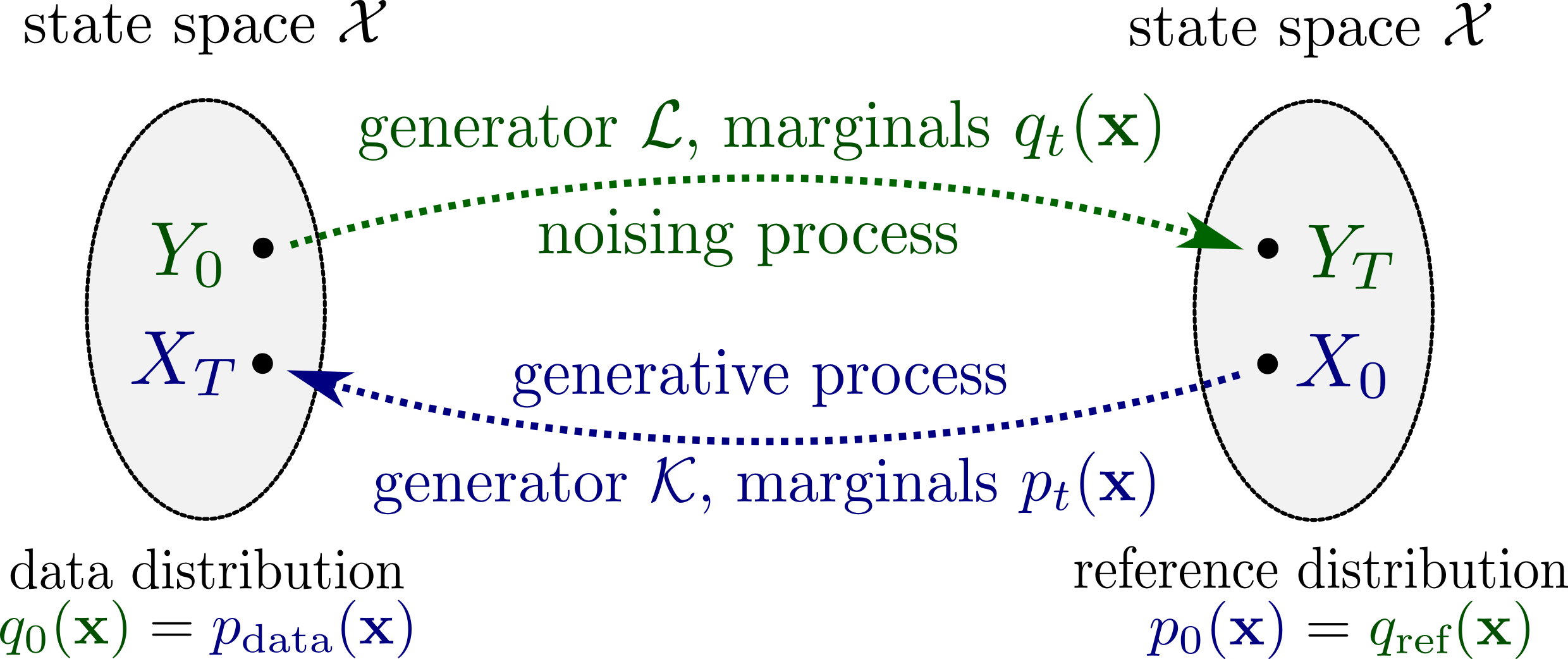}
    \caption{Diagram of notation.}
    \label{fig:process_diagram}
\end{figure}

In general, it is most convenient to define $X$ and $Y$ via the generators of $\extX$ and $\extY$, which we denote by $\K$ and $\L$ respectively. Informally, the generator of a Markov process $\extW$ with state space $\cal{S}$ is an operator $\A$ which acts on a subset $\D(\A)$ of the space of functions $f : \cal{S} \rightarrow \R$ and satisfies $\A f = \lim_{s \rightarrow 0} ({P_s f - f})/{s}$, where $(P_s)_{s \geq 0}$ is the transition semigroup associated to $\extW$ and $P_s f(x) = \E{f(X_s) | X_0 = x}$. For a more formal definition, see Appendix \ref{app:fellerdefinition}.

We denote the time marginals of the processes $X$, $Y$ by $p_t(\x)$, $q_t(\x)$ respectively. We make some smoothness assumptions on $p$, in Appendix \ref{app:asspq}, and assume that $\K$, $\L$ satisfy some regularity conditions, in Appendix \ref{app:assgenerator}. Our assumptions hold for standard models in the literature (Euclidean diffusions, CTMCs and manifold diffusions; see Appendix \ref{app:particularspaces}), plus some that are not covered previously, such as degenerate diffusions. For infinite dimensional spaces, the assumptions of Appendix \ref{app:assstatespace} may fail and more care is needed.

One consequence of our assumptions is that the operator $\K$ decomposes as $\K = \del_t + \hat \K$, where $\hat \K$ operates only on the spatial variables of a function $f$. We can therefore view $\hat \K$ as an operator on functions from $\X$, rather than on functions from $\cal{S}$, and we denote by $\hat \K^\ast$ the adjoint of $\hat \K$ acting on functions on $\X$ (see Appendix \ref{app:dualgeneratordefinition}).
\begin{example}[Euclidean Diffusion]
\label{ex:realdiffusiongenerator}
If $X$ and $Y$ are diffusions on $\R^d$ given by the SDEs
$\d X_t = \mu(X_t, t) \d t + \d \hat B_t$ and $\d Y_t = b(Y_t, t) \d t + \d B_t$, where $B$ and $\hat B$ are Brownian motions, then the corresponding generators are $\K = \partial_t + \mu \cdot \nabla + \frac{1}{2} \Delta$ and $\L = \partial_t + b \cdot \nabla + \frac{1}{2} \Delta$, where $\Delta = \sum_{i=1}^d \frac{\del^2}{\del x_i^2}$ denotes the Laplacian. We then have $\hat \K^\ast = - \mu \cdot \nabla - (\nabla \cdot \mu) + \frac{1}{2} \Delta$ using integration by parts.
\end{example}
\begin{example}[Discrete Space CTMC]
\label{ex:CTMCgenerator}
If $X$ and $Y$ are CTMCs, then $\K = \del_t + A$ and $\L = \del_t + B$, where $A$ and $B$ are the time-dependent generator matrices of $X$ and $Y$. In this case, $\hat \K^\ast = A^T$, the transpose of $A$.
\end{example}

\vspace{-0.5cm}
\subsection{An expression for the model likelihood}

We now derive an expression for the model likelihood $p_T(\x)$. First, under our assumptions, a generalised form of the Fokker--Planck equation, stated precisely in Appendix \ref{app:stochasticprocesstheory}, implies that $\partial_t p = \hat\K^\ast p$ for $\nu$-almost every $\x \in \X$. Typically, the adjoint operator $\hat\K^\ast$ resembles the generator of another process in the same class as $X$ and $Y$. We formalise this idea by making the following assumption.

\begin{assumption}
\label{ass:existenceM}
Let $v(\x,t) = p_{T-t}(\x)$. Then we can write the equation $\partial_t p = \hat\K^\ast p$ in the form $\M v + cv = 0$ for some function $c : \cal{S} \rightarrow \R$, where $\M$ is the generator of another auxiliary Feller process $\extZ = (Z_t, t)_{t \geq 0}$ on $\cal{S}$.
\end{assumption}

\begin{example}[Euclidean Diffusion]
For Euclidean diffusions, the Fokker--Planck equation can be written as $\del_t v = \mu \cdot \nabla v + (\nabla \cdot \mu) v - \frac{1}{2} \Delta v$. Assumption \ref{ass:existenceM} is satisfied with $c = - (\nabla \cdot \mu)$ and $\M = \partial_t - \mu \cdot \nabla + \frac{1}{2} \Delta$,
noting that $\M$ is the generator of the diffusion process $Z$ defined by $\d Z_t = -\mu(Z_t, T-t) \d t + \d B'_t$, where $B'$ is a Brownian motion.
\end{example}

\begin{example}[Discrete Space CTMC] In the CTMC case, if $c_\x = \sum_{\y \in \X} A_{\y\x}$, and $D_{\x\y} = A_{\y\x} - c_\x \mathbbm{1}_{\x = \y}$,
then $\M = \partial_t + D$ is the generator of a CTMC and Assumption \ref{ass:existenceM} is satisfied. Here $c$ has a natural interpretation as a ``discrete divergence''.
\end{example}

In general, we make two smoothness assumptions on $c$ and $v$, given in Appendix \ref{app:assMc}.

Given the Fokker--Planck equation and Assumption \ref{ass:existenceM}, we apply a generalised form of the Feynman--Kac Theorem (see Appendix \ref{app:stochasticprocesstheory}) to $\extZ$ and $v$ to get the following expression for the model likelihood, which generalises that of \citecont{huang2021variational}:
\begin{equation}
\label{eq:model_likelihood}
    p_T(\x) = v(\x, 0) = \bb{E}\bigg[p_0(Z_T) \exp \left\{ \int_0^T c(Z_s, s) \; \d s \right\} \; \bigg| \; Z_0 = \x\bigg].
\end{equation}
This gives an expression in terms of an expectation over the auxiliary process $Z$. We next make this tractable by converting it into an expectation over $Y$.

\subsection{Deriving a tractable lower bound on the model log-likelihood}
\label{deriving_ELBO}

We would like to train our model by finding a reverse process $X$ which maximises the likelihood in Equation (\ref{eq:model_likelihood}). Unfortunately this expression is intractable, but we can find a tractable lower bound for $\log p_T(\x)$ which can then be used as a surrogate objective.

By taking logarithms in Equation (\ref{eq:model_likelihood}) and applying Jensen's inequality, we get
\begin{equation}
\label{eq:ELBO}
    \log p_T(\x) \geq \E[\Q]{\log \frac{\d\P}{\d\Q} + \log p_0(Y_T) + \int_0^T c(Y_s, s) \; \d s \; \bigg| \; Y_0 = \x} =: \cal{E}^\infty
\end{equation}
where $\P$ and $\Q$ are the path measures of the processes $\extZ$ and $\extY$ respectively and $\frac{\d \P}{\d \Q}$ denotes the Radon--Nikodym derivative.

To write $\cal{E}^\infty$ in a tractable form we need to evaluate $\log \frac{\d\P}{\d\Q}$, which we do using a generalisation of Girsanov's theorem. To apply this result, we require that the generators of the auxiliary process and the noising process are related in the following way.

\begin{assumption}
\label{ass:generatorrelation}
There is a bounded measurable function $\beta : \cal{S} \rightarrow (0, \infty)$ such that $\beta^{-1} \M f = \L(\beta^{-1} f) - f \L (\beta^{-1})$ for all $f : S \rightarrow \R$ such that $f \in \D(\M)$ and $\beta^{-1} f \in \D(\L)$. 
\end{assumption}

Since $\M$ is defined in terms of $\K$, we think of Assumption \ref{ass:generatorrelation} as forcing a particular parameterisation of the generative process in terms of $\beta$. In general, not every generative process in the same class as $\L$ will have such a parameterisation. However, the true time-reversal of $\L$ can always be parameterised in this way with $\beta(\x,t) = p_t(\x)$, so this parameterisation is sufficient to capture the optimal generative process. In addition, the objective in Theorem \ref{thm:ELBO} below can often be interpreted and used for a much broader set of generative processes than those which satisfy Assumption \ref{ass:generatorrelation}.


Under Assumption \ref{ass:generatorrelation}, along with a further technical assumption given in Appendix \ref{app:assalpha}, we may apply a generalised form of Girsanov's Theorem (see Appendix \ref{app:stochasticprocesstheory}, and take $\alpha = \beta^{-1}$ in Theorem \ref{generalised_girsanov}) and Dynkin's formula (see Appendix \ref{app:fellerdefinition}) to get
\begin{align*}
    \log \frac{\d\P}{\d\Q} 
    & = \int_0^T \left\{ - \L \log \beta(Y_s, s) - \beta(Y_s, s)\L (\beta^{-1}) (Y_s, s) \right\} \d s + \Q\text{-}\textup{martingale}.
\end{align*}
In addition, we get that $c = \beta \L (\beta^{-1}) - v^{-1} \beta \L(\beta^{-1} v)$ by combining Assumption \ref{ass:generatorrelation} with $f = v$ and Assumption \ref{ass:existenceM}. This allows us to rewrite the ELBO from Equation (\ref{eq:ELBO}) as
\begin{equation*}
    \cal{E}^\infty = \E[\Q]{\log p_0(Y_T) - \int_0^T \Big\{ \frac{\L(\beta^{-1} v)(Y_s, s)}{\beta^{-1}(Y_s, s) v (Y_s, s)} + \L \log \beta (Y_s, s) \Big\} \d s \; \bigg | \; Y_0 = \x}.
\end{equation*}
The final step required to get a tractable expression for $\cal{E}^\infty$ is to remove the function $v$ from this expression. For this, we use the following lemma (see Appendix \ref{app:genobjectiveproofs}).
\begin{restatable}{lemma}{genlem}
\label{lem:gen_lem_1}
Let the generator $\L$ and the functions $\beta$ and $c$ be as above. Then, we have $v^{-1} \beta {\L (\beta^{-1} v)} + \L \log \beta = \beta^{-1} \hat\L^\ast \beta + \hat\L \log \beta$.
\end{restatable}
\begin{theorem}
\label{thm:ELBO}
For DMMs as in Section \ref{sec:setup}--\ref{deriving_ELBO}, the log-likelihood is lower bounded by
\begin{equation}
\label{eq:tractable_ELBO}
    \cal{E}^\infty = \bb{E}_{\Q}\Big[\log p_0(Y_T) \Big | Y_0 = \x\Big] - \int_0^T \bb{E}_{\Q}\bigg[\frac{\hat\L^\ast \beta(Y_s, s)}{\beta(Y_s, s)} + \hat\L \log \beta(Y_s, s) \; \bigg | \; Y_0 = \x\bigg] \d s.
\end{equation}
\end{theorem}

This result extends the corresponding expression for $\R^d$ in \citecont{huang2021variational}. We see the ELBO consists of a term representing the log-likelihood under the reference distribution and an implicit score matching term arising from the change in measure.

\vspace{-0.5cm}
\subsection{Finding suitable training objectives}
\label{sec:tractable_objective}

Based on Theorem \ref{thm:ELBO}, we fit our generative model by maximising the expectation of $\cal{E}^\infty$ with respect to $p_{\data}$. This is equivalent to minimising the objective
\begin{equation}\label{eq:general_ISM}
    \I_{\textup{ISM}}(\beta) = \int_0^T \E[q_t(\x_t)]{ \frac{\hat\L^\ast \beta(\x_t, t)}{\beta(\x_t, t)} + \hat\L \log \beta(\x_t, t)} \d t,
\end{equation}
which we call the \emph{implicit score matching} objective, since it can be interpreted as an extension of implicit score matching from $\R^d$ (see Section \ref{sec:relationship_score_matching} below for more intuition). 

Since $q_t$ and $\hat\L$ are determined by the noising process, which is known and assumed easy to sample from, $\I_{\textup{ISM}}(\beta)$ and its gradient with respect to $\beta$ can be estimated in an unbiased fashion. Since $\beta$ parameterises $\M$ via Assumption \ref{ass:generatorrelation}, and thus $\K$ through Assumption \ref{ass:existenceM}, minimising $\I_{\textup{ISM}}(\beta)$ over $\beta$ is equivalent to learning the generative process.

We also have an equivalent \emph{denoising score matching objective} (see Appendix \ref{app:equiv_objectives}),
\begin{equation}
\label{eq:genDSM}
    \I_{\textup{DSM}}(\beta) = \int_0^T \E[q_{0,t}(\x_0, \x_t)]{\frac{\L( q_{\cdot|0}(\cdot | \x_0)/\beta(\cdot, \cdot))(\x_t, t)}{q_{t|0}(\x_t | \x_0) / \beta(\x_t, t)} - \L \log (q_{\cdot|0}(\cdot | \x_0) /\beta(\cdot, \cdot))(\x_t, t)} \d t.
\end{equation}
Both objectives are minimised when $\beta(\x, t) \propto q_t(\x)$, as shown in Proposition \ref{lem:operator_properties}. $\I_{\textup{DSM}}(\beta)$ can be interpreted as quantifying the difference between $q_{t|0}(\x_t | \x_0)$ and $\beta(\x_t, t)$ via the score matching operator $\Phi(f) = f^{-1}\L f - \L \log f$ introduced in Section \ref{sec:relationship_score_matching} below. These objectives also generalise the following previously studied instances of diffusion models. For all derivations and remarks on the choice of parameterisation, see Appendix \ref{app:particularspaces}.

\begin{example}[Euclidean Diffusion]
\label{ex:realdiffusionparameterisation}
In the setting of Example \ref{ex:realdiffusiongenerator}, Assumption \ref{ass:generatorrelation} reduces to $\nabla \log \beta = b + \mu$, and we have $f^{-1} \L f - \L \log f = \frac{1}{2} \big\| \nabla \log f \big\|^2$. If we substitute $s_\theta(\x_t, t) = \nabla \log \beta(\x_t, t)$, $\I_{\textup{DSM}}(\beta)$ defined in Equation (\ref{eq:genDSM}) reduces to Equation (\ref{eq:score_matching_objective}) and the reverse process is parameterised as in Equation (\ref{eq:approxreverse}). We thus recover the results of \citecont{song2021score} and \citecont{huang2021variational}.
\end{example}

\begin{example}[Discrete Space CTMC]
\label{ex:CTMCobjective}
In the setting of Example \ref{ex:CTMCgenerator}, Assumption \ref{ass:generatorrelation} reduces to $A_{\y\x} = \frac{\beta(\x,t)}{\beta(\y,t)} B_{\x\y}$ for all $\x \neq \y$. We may rewrite $\I_{\textup{ISM}}$ in terms of $A$ to recover the objective of \citecont{campbell2022continuous},
\begin{equation*}
    \mathcal{I}_{\textup{ISM}}(A) = \int_0^T \bb{E}_{q_t(\x_t)}\bigg[- A_{\x_t\x_t} - \sum_{\y \neq \x_t} B_{\x_t\y} \log A_{\y\x_t}\bigg] \d t + \textup{\const}.
\end{equation*}
\end{example}

\begin{example}[Riemannian Manifolds]
\label{ex:manifolddiffusion}
If $\X$ is a Riemannian manifold and we take $\K = \partial_t + \mu \cdot \nabla + \frac{1}{2} \Delta$, $\L = \partial_t + b \cdot \nabla + \frac{1}{2} \Delta$ where $\Delta$ is the Laplace--Beltrami operator associated to $\X$, and perform the reparameterisation $s_\theta(\x_t, t) = \nabla \log \beta(\x_t, t)$, then we recover the framework for training diffusion models on Riemannian manifolds given in \citecont{debortoli2022riemannian} and \citecont{huang2022riemannian}.
\end{example}
\vspace{-0.5cm}
\subsection{Inference}
\label{sec:generalisedinference}
To use DMMs for inference, we follow a similar procedure to Section \ref{sec:inferencebackground}. To noise a sample $(\x_0, \xiobs_0) \sim p_{\data}$, we set $Y_0 = \x_0$ and let $Y$ evolve according to $\L$. To generate $\x_0$ conditioned on an observation $\xiobs^\ast$, we use a generative process $X^{\xiobs^\ast}$ conditioned on $\xiobs^\ast$. We parameterise $X^{\xiobs^\ast}$ in terms of a function $\beta(\x_t, \xiobs^\ast, t)$ which now takes $\xiobs^\ast$ as an input.

We aim to learn $X^{\xiobs^\ast}$ to approximate the time-reversal of $Y$ conditioned on $\xiobs^\ast$. The following extension of Theorem \ref{thm:ELBO} (proved in Appendix \ref{app:genobjectiveproofs}) gives us a way to do this.

\begin{restatable}{theorem}{inferencethm}
\label{thm:inference}
With the above set-up, minimising the objective
\begin{equation*}
    \mathcal{I}_{\textup{DSM}}(\beta) = \int_0^T \E[q(\x_0, \x_t, \xiobs_0)]{\frac{\L(q_{\cdot|0}(\cdot | \x_0) / \beta(\cdot, \xiobs_0, \cdot))(\x_t,t)}{q_{t|0}(\x_t | \x_0) / \beta(\x_t, \xiobs_0, t)} - \L \log (q_{\cdot|0}(\cdot | \x_0) / \beta(\cdot, \xiobs_0, \cdot))(\x_t,t)} \d t
\end{equation*}
is equivalent to maximising a lower bound on the expected model log-likelihood.
\end{restatable}

Theorem \ref{thm:inference} suggests that we may train conditional DMMs by maximising the objective $\I_{\textup{DSM}}(\beta)$ (or the equivalent $\I_{\textup{ISM}}(\beta)$ objective). Since $q_{t|0}(\x_t | \x_0)$ is known, we may do this by calculating an empirical estimate for $\I_{\textup{DSM}}(\beta)$ based on samples $(\x_0, \xiobs_0)$ drawn from $p_{\data}$ and minimising over $\beta$. Then, we generate samples from $p_{\data}(\x_0 | \xiobs^\ast)$ by initialising $X_0^{\xiobs^\ast} \sim p_0$, simulating the reverse process with generator $\K$ parameterised by $\beta = \beta(\cdot, \xiobs^\ast, \cdot)$, and setting $\x_0 = X_T^{\xiobs^\ast}$.

\section{Score matching on general state-spaces}
\label{sec:relationship_score_matching}

When $X$ and $Y$ are Euclidean diffusions, the objective $\I_{\textup{DSM}}(\beta)$ in Equation (\ref{eq:genDSM}) becomes the score matching objective in Equation (\ref{eq:score_matching_objective}). Similarly, the objective $\I_{\textup{ISM}}(\beta)$ from Equation (\ref{eq:general_ISM}) reduces to the implicit score matching objective introduced by \citecont{hyvarinen2005estimation}. This suggests we can view Equations (\ref{eq:general_ISM}) and (\ref{eq:genDSM}) as generalisations of score matching objectives to arbitrary state spaces.

Given state space $\X$ on which we have a Markov process generator $\L$ and an unknown distribution $q_0(\x)$ we wish to approximate, the corresponding \emph{generalised implicit score matching} method learns an approximation $\varphi(\x)$ to $q_0(\x)$ by minimising
\begin{equation*}
    \J_{\textup{ISM}}(\varphi) = \bb{E}_{q_0(\x)} \bigg[ \frac{\hat\L^\ast \varphi(\x)}{\varphi(\x)} + \hat\L \log \varphi(\x)\bigg].
\end{equation*}
We can show that $\J_{\textup{ISM}}$ is equivalent to the \emph{generalised explicit score matching objective}
\begin{equation*}
    \J_{\textup{ESM}}(\varphi) = \bb{E}_{q_0(\x)} \bigg[ \frac{\L (q_0 / \varphi)(\x)}{(q_0(\x) / \varphi(\x))} - \L \log (q_0 / \varphi)(\x) \bigg].
\end{equation*}

In addition, we define the corresponding \emph{generalised denoising score matching} method, which learns an approximation $\varphi_\tau(\x_\tau)$ to the noised distribution $q_\tau(\x_\tau)$, formed by sampling $\x_0 \sim q_0(\cdot)$ and $\x_\tau \sim q_{\tau|0}(\;\cdot \;| \x_0)$, where $q_{\tau|0}$ is the transition probability associated to $\L$ run for time $\tau$. It does this by minimising the objective
\begin{equation*}
    \J_{\textup{DSM}}(\varphi_\tau) = \E[q_{0,\tau}(\x_0, \x_\tau)]{\frac{\L(q_{\tau|0}(\cdot| \x_0) / \varphi_\tau(\cdot))(\x_\tau)}{q_{\tau|0}(\x_\tau | \x_0) / \varphi_\tau(\x_\tau)} - \L \log (q_{\tau|0}(\cdot | \x_0) / \varphi_\tau(\cdot))(\x_\tau)}.
\end{equation*}
$\J_{\textup{DSM}}$ is equivalent to both $\J_{\textup{ISM}}$ and $\J_{\textup{ESM}}$ when used to learn the smoothed distribution $q_\tau(\x_\tau)$ (see Appendix \ref{app:equiv_objectives}). All three objectives extend the corresponding score matching objectives introduced for $\R^d$ by \citecont{hyvarinen2005estimation} and \citecont{vincent2011connection}. They also coincide with the extension of score matching for Riemannian manifolds of \citecont{mardia2016score}.

To illustrate further intuitions behind our objective functions, we define the \emph{score matching operator} $\Phi(f) = f^{-1}\L f - \L \log f$. Note that the time component of $\Phi$ cancels, so we can view it as an operator on $\X$. With this notation, the generalised explicit score matching objective becomes $\J_{\textup{ESM}}(\varphi) = \E[q_0(\x)]{\Phi(q_0/\varphi)(\x)}$. For Euclidean diffusions, $\Phi(f) = \frac{1}{2} ||\nabla \log f ||^2$ (see Example \ref{ex:realdiffusionparameterisation}). In the general case, we view $\Phi(f)$ as measuring the magnitude of a logarithmic gradient of $f$. We interpret the objectives $\J_{\textup{DSM}}$ and $\J_\textup{ESM}$ as trying to fit $\varphi$ to $q_0$ by minimising this logarithmic gradient of the ratio $q_0 / \varphi$.
\begin{restatable}{proposition}{operatorlem}
\label{lem:operator_properties}
Let $Y$ be a Feller process with semigroup operators $(Q_t)_{t \geq 0}$, generator $\L$ and associated score matching operator $\Phi$. Then:
\begin{enumerate}
    \item $\Phi(f) \geq 0$ for all $f$ in the domain of $\Phi$, with equality if $f$ is constant;
    \item for any probability measures $\pi_1, \pi_2$ on $\X$ and $t \geq 0$,
    \begin{equation*}
        \frac{\d}{\d t} \KL{\pi_1 Q_t}{\pi_2 Q_t} = - \E[\pi_1 Q_t]{\Phi\lr{\frac{\d (\pi_1 Q_t)}{\d (\pi_2 Q_t)}}},
    \end{equation*}
    where $\KL{\pi_1 Q_t}{\pi_2 Q_t}$ denotes the Kullback--Leibler divergence between $\pi_1 Q_t$, $\pi_2 Q_t$.
\end{enumerate}
\end{restatable}

Proposition \ref{lem:operator_properties}(a) shows that $\Phi$ is always non-negative, so $\J_{\textup{{ESM}}}$ is minimised if $\varphi(\x) \propto q_0(\x)$. Thus minimising any of our generalised score matching objectives should typically correspond to learning an approximation to $q_0$. Note though that if $Q_t$ is not ergodic and $\pi_1$, $\pi_2$ are different invariant distributions of $Q_t$ then Proposition \ref{lem:operator_properties}(b) implies that $\Phi(\d \pi_1 / \d \pi_2) = 0$ $\pi_1$-a.e., even though $\d \pi_1 / \d \pi_2$ is not constant. This suggests that generalised score matching may fail if the noising process is not ergodic. Proposition \ref{lem:operator_properties}(b) was proved for score matching on $\R^d$ by \citecont{lyu2009interpretation}. It suggests we can interpret score matching as finding an approximation $\varphi$ which minimises the decrease in KL divergence between $q_0$ and $\varphi$ caused by adding an infinitesimal amount of noise to both according to $\L$.

Our generalised score matching methods give a principled way to extend score matching to fit unnormalised probability distributions on arbitrary spaces. Other extensions of score matching have been explored, including to arbitrary sub-domains of $\R^d$ \citepcont{yu2022generalized}, ratio matching \citepcont{hyvarinen2007some} and marginalisation with generalised score matching \citepcont{lyu2009interpretation}. However, these methods lack the generality of our framework and do not respect the intuition coming from $\R^d$ that Proposition \ref{lem:operator_properties}(b) should hold. There are also many other density estimation methods that seek to learn ratios of density functions, including noise-contrastive estimation, which also approximates score matching under certain conditions \citepcont{gutmann2012bregman}.

\section{Relationship to discrete time models}
\label{sec:discreteapprox}

Denoising diffusion models were originally introduced in discrete time by \citecont{sohldickstein2015deep}. In this setting, the noising and generative processes are Markov chains $\x_{0:T}=(\x_{t_k})_{k=0}^N$ observed at a sequence of times $0 = t_0 < t_1 < \dots < t_N = T$, with fixed forwards transition kernel $\tilde q(\x_{t_k} | \x_{t_{k-1}})$ and learned backwards kernel $\tilde p_\theta(\x_{t_{k-1}} | \x_{t_k})$.
To fit discrete time diffusion models, \citecont{sohldickstein2015deep} minimise the following Kullback--Leibler divergence with respect to $\theta$:
\begin{equation}
\label{eq:discrete_time_ELBO}
    \KL{\tilde{q}(\x_{0:T})}{\tilde{p}_\theta(\x_{0:T})} = \sum_{k=1}^{N} \E [\tilde q(\x_{t_{k-1}},\x_{t_k})]{\log \frac{\tilde q(\x_{t_k}| \x_{t_{k-1}})}{\tilde p_\theta(\x_{t_{k-1}} | \x_{t_k})}}+\const.
\end{equation}

Given any DMM with generators $\K, \L$ and marginals $p_t, q_t$ as in Section \ref{sec:general_objective}, we define its \emph{natural discretisation} to be the discrete-time model with $\tilde q(\x_{t_k} | \x_{t_{k-1}}) = q_{{t_k}|t_{k-1}}(\x_{t_k} | \x_{t_{k-1}})$ and $\tilde p_\theta(\x_{t_{k-1}} | \x_{t_k}) = p_{T - t_{k-1} | T - t_k}(\x_{t_{k-1}} | \x_{t_k})$. Then, the Kullback--Leibler divergence (\ref{eq:discrete_time_ELBO}) for the natural discretisation can be viewed as a first-order approximation to $\I_{\textup{ISM}}$ for the continuous-time model.
\begin{restatable}{lemma}{approxlem}
\label{lem:approx}
Suppose $X$, $Y$ are fixed generative and noising processes with marginals $p$, $q$ as in Section \ref{sec:general_objective}, and suppose that they are related as in Assumptions \ref{ass:existenceM} and \ref{ass:generatorrelation} for some sufficiently regular function $\beta$. Then for any $0 < s < t < T$ with $\gamma = t - s$,
\begin{equation*}
    \gamma \; \bb{E}_{q_s(\x_s)}\bigg[\frac{\hat\L^\ast \beta(\x_s)}{\beta(\x_s)} + \hat\L \log \beta (\x_s) \bigg] = \E[q_{s,t}(\x_s, \x_t)]{\log \frac{q_{t|s}(\x_t | \x_s)}{p_{T-s | T-t}(\x_s | \x_t)}} + o(\gamma). 
\end{equation*}
\end{restatable}
Applying this lemma on each interval $[t_k, t_{k+1}]$, we get the following theorem.
\begin{restatable}{theorem}{approxthm}
\label{thm:approx}
For any DMM, the objective (\ref{eq:discrete_time_ELBO}) for its natural discretisation is equivalent to the natural discretisation of $\I_{\textup{ISM}}$ to first order in $\overline{\gamma} = \max_{k=0, \dots, N-1}|t_{k+1} - t_k|$.
\end{restatable}
This theorem generalises to arbitrary state spaces a result of \citecont{ho2020denoising}, which demonstrated the equivalence of minimizing (\ref{eq:discrete_time_ELBO}) and the score matching objective for Euclidean state spaces. For the proofs of Lemma \ref{lem:approx} and Theorem \ref{thm:approx}, see Appendix \ref{app:discretetimeproofs}.

Lemma \ref{lem:approx} also implies a general equivalence between one-step denoising autoencoders and score matching. \citecont{vincent2011connection} discussed this equivalence for autoencoders using Gaussian noise in $\R^d$, but our methods allow us to extend this correspondence to arbitrary state spaces and noising processes. For more details, see Appendix \ref{app:autoencoder}.

\section{Experiments}
\label{sec:experiments}

We now present experiments demonstrating DMMs on several tasks and data spaces, for unconditional generation and conditional simulation. All details are in Appendix \ref{app:experimentaldetails}.

\vspace{-0.5cm}
\subsection{Inference on $\R^d$ using diffusion processes}
First, we use diffusion processes in $\R^d$ to perform approximate Bayesian inference for real-valued parameters. We consider $p_{\data}(\xiobs | \x) = \prod_{i=1}^N p_{\data}(\xi_i | \x)$, where $p_{\data}(\xi_i | \x)$ is the $g$-and-$k$ distribution with parameters $\x = (A, B, g, k)$ and $d=4$, and we let $p_{\data}(\x)$ be uniform on $[0,10]^4$. The $g$-and-$k$ distribution is a 4-parameter distribution in which $A,B,g,k$ control the location, scale, skewness and kurtosis respectively.

We fix our noising process to be an Ornstein--Uhlenbeck process, and parameterise our reverse process as in Example \ref{ex:realdiffusionparameterisation}, with $s_\theta(\x, \xiobs, t)$ being given by a fully connected neural network. To train the model, we sample $(\x_0, \xiobs_0) \sim p_{\data}$ and minimise the denoising score matching objective from Section \ref{sec:generalisedinference} via stochastic gradient descent on $\theta$.

To test our model, we first consider the case where there are a true set of underlying parameters $\x_{\textup{true}} = (3,1,2,0.5)$. We generate an observation $\xiobs_0 \sim p_{\data}(\xiobs_0 | \x_{\textup{true}})$ with $N = 250$, sample from the approximate posterior using our DMM and plot the result in Fig.\ \ref{fig:gandk_density_250}. We compare our method with the semi-automatic ABC (SA-ABC) \citepcont{nunes2015abctools} and Wasserstein SMC (W-SMC) \citepcont{bernton2019approximate} methodologies, as well as Sequential Neural Posterior, Likelihood and Ratio Estimation approaches (SNPE, SNLE and SNRE) (see e.g. \citecont{lueckmann2021benchmarking}). We see in Fig.\ \ref{fig:gandk_density_250} that the DMM achieves more accurate posterior estimation for all parameters, except the kurtosis parameter $k$ for which W-SMC is more accurate. Among the other neural network-based approaches, SNPE appears most competitive on this task, but is less accurate than the DMM especially for parameters $g$ and $k$. Additional experimental results comparing DMMs to other simulation-based inference methods can be found in \citepcont{sharrock2022sequential,geffner2023compositional}. 



\begin{figure}[t]
    \centering
    \includegraphics[width=\textwidth]{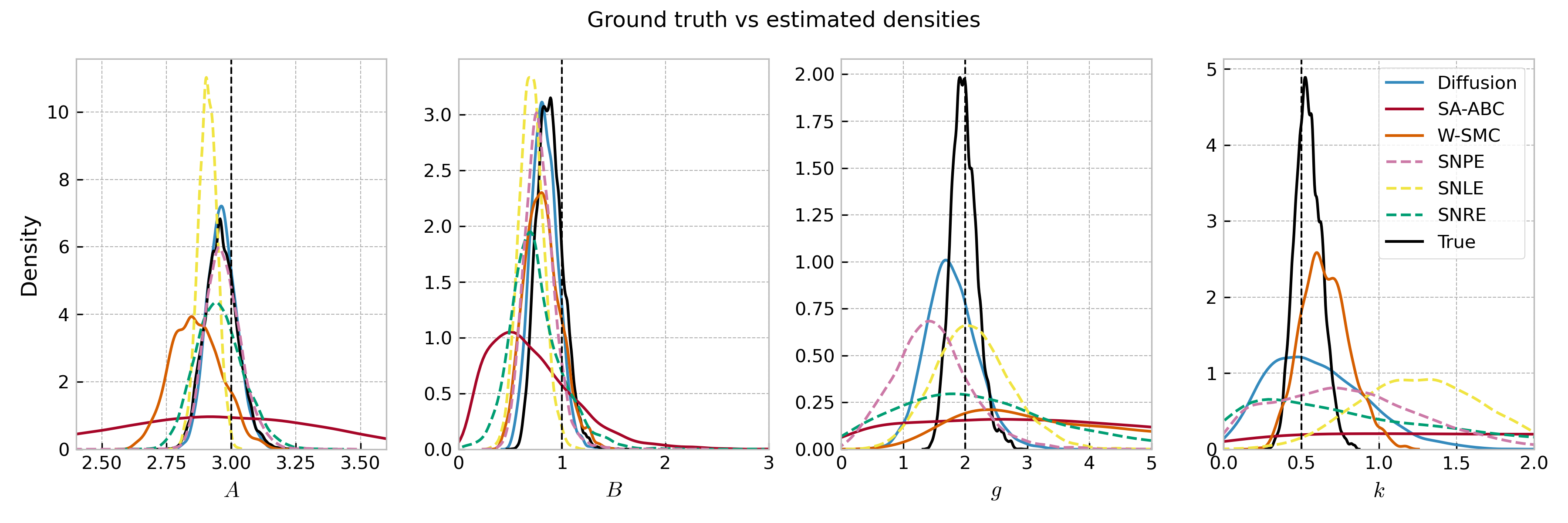}
    \caption{Posterior kernel density estimates of samples generated using our DMM, SA-ABC, W-SMC, SNLE, SNPE and SNRE for the $g$-and-$k$ distribution, with $\x_{\textup{true}}=(3,1,2,0.5)$ and $N=250$.}
    \label{fig:gandk_density_250}
\end{figure}

Next, we demonstrate that our model can perform inference for a range of observation values $\xiobs^\ast$ simultaneously. We generate a series of 512 parameter values $\x_0$ drawn from $p_{\data}(\x_0)$ and draw an observation $\xiobs_0$ from $p_{\data}(\xiobs_0 | \x_0)$ with $N = 10000$ for each $\x_0$. Then, we generate 8 samples $\x_0'$ from our approximation to the posterior $p_{\data}(\x_0 | \xiobs_0)$ for each $\xiobs_0$. We plot each component of the pairs $(\x_0, \x_0')$ in Fig.\ \ref{fig:gandk_true_parameter_vs_samples_10000}. We see our model is able to infer the original parameters across a range of parameter values.

\begin{figure}[t]
    \centering
    \includegraphics[trim={0, 0, 0, 0.4cm}, clip, width=\textwidth]{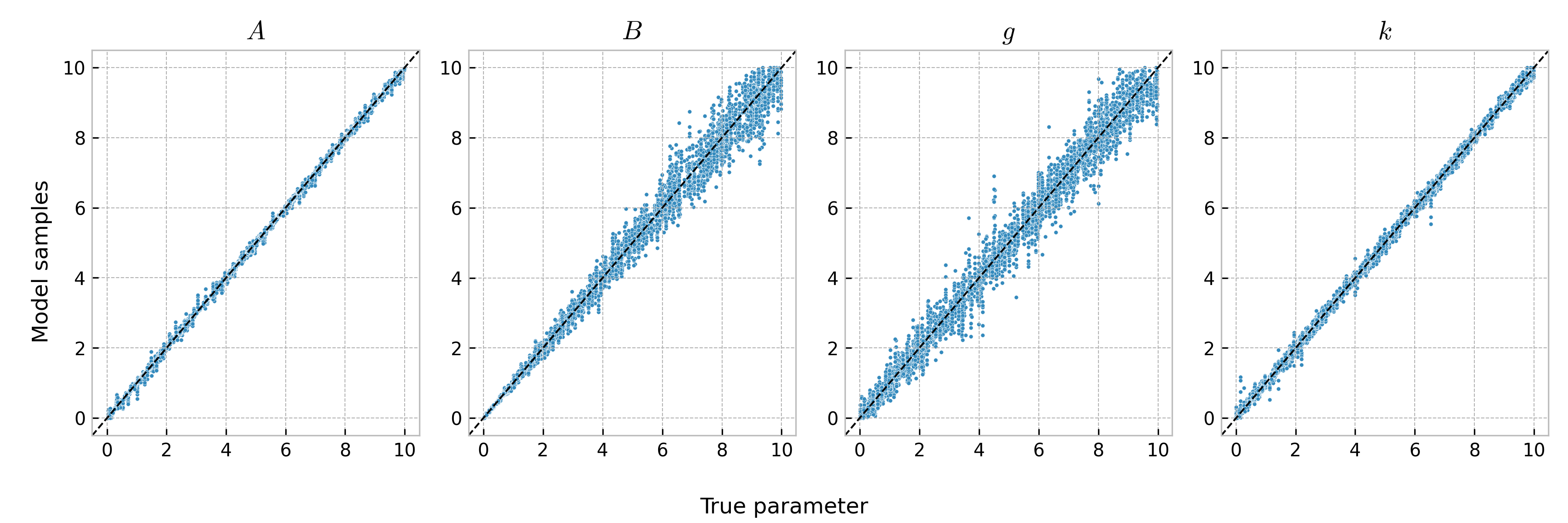}
    \caption{Comparison of posterior samples $\x_0'$ from our DMM approximation to $p_{\data}(\cdot | \xiobs_0)$ and the true parameter value $\x_0$ for a range of $\x_0$ in the prior distribution, with $N = 10000$.}
    \label{fig:gandk_true_parameter_vs_samples_10000}
\end{figure}


\vspace{-0.5cm}
\subsection{Image inpainting and super-resolution using discrete-space CTMCs}

Second, we demonstrate that our framework is applicable for large-scale Bayesian inverse problems, such as super-resolution and inpainting for images. For these problems, the prior $p_{\data}(\x)$ is the distribution of images. Most ABC techniques such as SA-ABC and W-SMC are not applicable as they require an analytical expression for this prior, whereas DMMs do not rely on such an expression. 

We consider performing image inpainting for MNIST digit images, where each image $\x_0$ has $28\times28$ pixels with values in $\{0,\dots,255\}$, and the observed incomplete image $\xiobs_0$ has the middle $14\times14$ pixels missing. Since our state space $\X = \{0,\dots,255\}^{28 \times 28}$ is discrete, we use the set-up of Example \ref{ex:CTMCgenerator} and let the generator of our noising process factor over pixel dimensions. We use the denoising parameterisation of the reverse process (see Appendix \ref{app:discretespaces}) and train by minimising the form of the objective in Example \ref{ex:CTMCobjective}.

To test our model, we plot the reconstructed image samples for a number of digits in Fig.\ \ref{fig:mnist_inpaint_grid}. We observe that the samples we obtain are consistent with conditioning and appear to be realistic, but also display diversity in the shape of the strokes. In Appendix \ref{app:mnistexperiment}, we also compare our method to a continuous state space approach.

In addition, we train a conditional discrete-space DMM to perform super-resolution on ImageNet images to demonstrate that this method provides perceptually high quality samples even in very high-dimensional scenarios. For details, see Appendix \ref{app:imagenetexperimentdetails}.

\begin{figure}
    \centering
    \includegraphics[height=3.5cm]{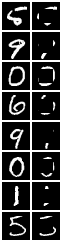} \includegraphics[height=3.5cm]{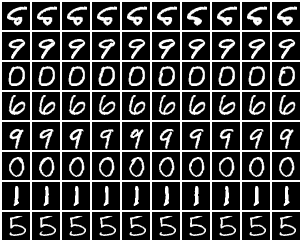}
    \hspace{15mm}
    \includegraphics[height=3.5cm]{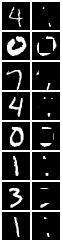} \includegraphics[height=3.5cm]{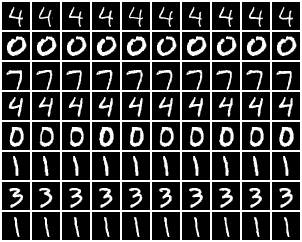}
    \caption{Samples from the MNIST inpainting task. The first column in each set plots the ground truth images, and the second column has the centre $14\times14$ pixels missing.}
    \label{fig:mnist_inpaint_grid}
\end{figure}

\vspace{-0.5cm}
\subsection{Modelling distributions on $SO(3)$ using manifold diffusions}

Thirdly, we demonstrate that DMMs can approximate distributions on manifolds using two tasks on $SO(3)$. Since $SO(3)$ is a Lie group and so a Riemannian manifold, we use the framework from Example \ref{ex:manifolddiffusion}. As our noising process, we use Brownian motion with generator $\L = \del_t + \frac{1}{2} \Delta$. We can explicitly calculate the transition kernels $q_{t|0}(\x_t|\x_0)$ for this process, allowing us to use the denoising score matching objective. We parameterise this objective in terms of a neural network approximation $s_\theta(\x,t)$ of the score. This is in contrast to \citecont{debortoli2022riemannian}, in which the explicit transition kernels are not used for sampling the forward process or in the loss function, both of which require further approximations.

First we check that our DMM can learn simple mixtures of wrapped normal distributions $p_{\data}(\x) = \frac{1}{M} \sum_{m=1}^{M} \mathcal{N}^W (\x | \mu_m, \sigma^2_m)$, where $\mathcal{N}^W (\x \;|\; \mu_m, \sigma^2_m)$ is the wrapped normal distribution on $SO(3)$ with expectation $\mu_m$ and variance $\sigma^2_m$ \citepcont{debortoli2022riemannian}. We plot samples from our resulting DMM in Fig.\ \ref{fig:main_so3}. We see that our model provides a good fit to $p_\data(\x)$, covering all modes. In Appendix \ref{app:wrappednormalexperiment}, we provide additional results and show that we can also sample from the class conditional density $p_{\data}(\x | m)$.

Second, we consider a more realistic pose estimation task on the SYMSOL dataset, which requires predicting the 3D orientation of various symmetric 3D solids based on 2D views \citepcont{murphy2021implicit}. Due to the rotational symmetries, a key challenge is to predict all possible poses when only one possibility is presented in training.  We use a conditional DMM where $\xiobs$ is the 2D image view. Fig.\ \ref{fig:symsolresults} shows two sets of samples from our model conditioned on 2D images of two different solids. We see that our model learns to sample from the ground truth accurately and infer the full set of rotational symmetries for different views $\xiobs$. For further experimental details and plots, see Appendix \ref{app:symsolexperiment}.

\begin{figure}
    \centering
    \includegraphics[trim={4cm 1.1cm 1.8cm 1cm}, clip, width=\textwidth]{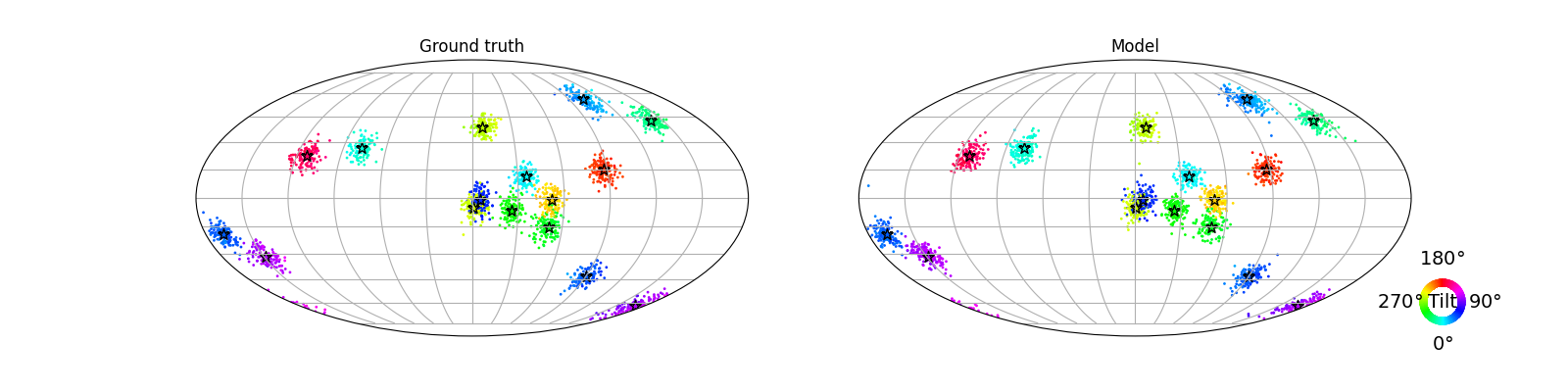} \\
    \caption{Samples from the ground truth and our DMM approximation to the mixture of wrapped normal distributions. Each sample is denoted by a point, whose position represents the axis of rotation and whose colour represents the angle of rotation. Stars denote the true cluster means.}
    \label{fig:main_so3}
\end{figure}

\begin{figure}
    \centering
    \hspace*{-0.1cm}\raisebox{-0.5\height}{\includegraphics[trim={0 0 1.8cm 0}, clip, width=1.5cm]{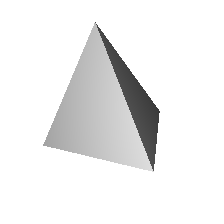}}
    \raisebox{-0.5\height}{\includegraphics[trim={2.5cm 0.7cm 1.8cm 1cm}, clip, width=12.5cm]{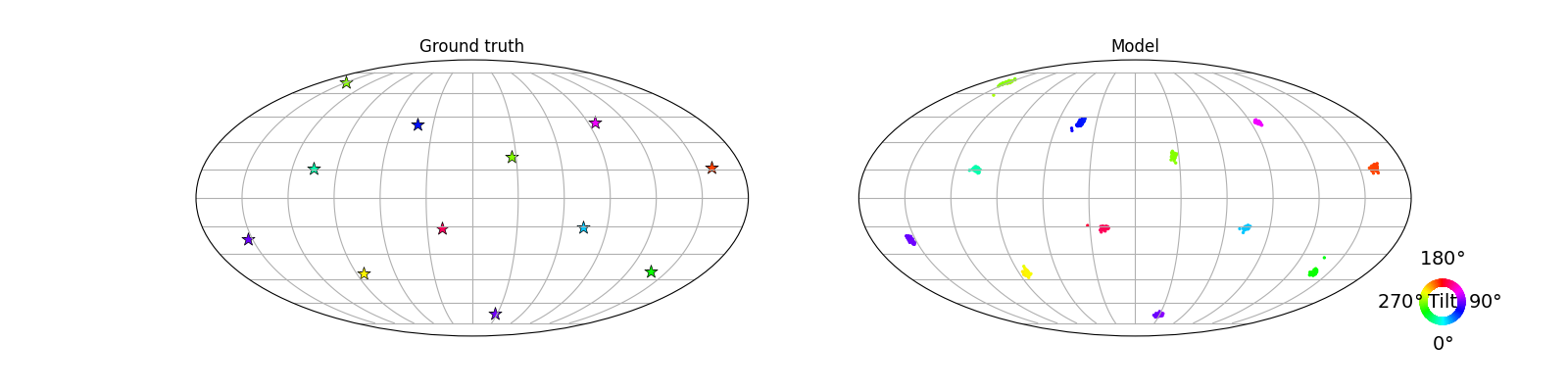}} \\
    \hspace*{-0.1cm}\raisebox{-0.5\height}{\includegraphics[trim={0 0 1.8cm 0}, clip, width=1.5cm]{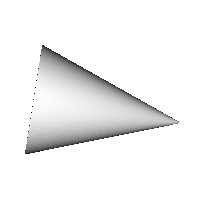}}
    \raisebox{-0.5\height}{\includegraphics[trim={2.5cm 1.1cm 1.8cm 0.7cm}, clip, width=12.5cm]{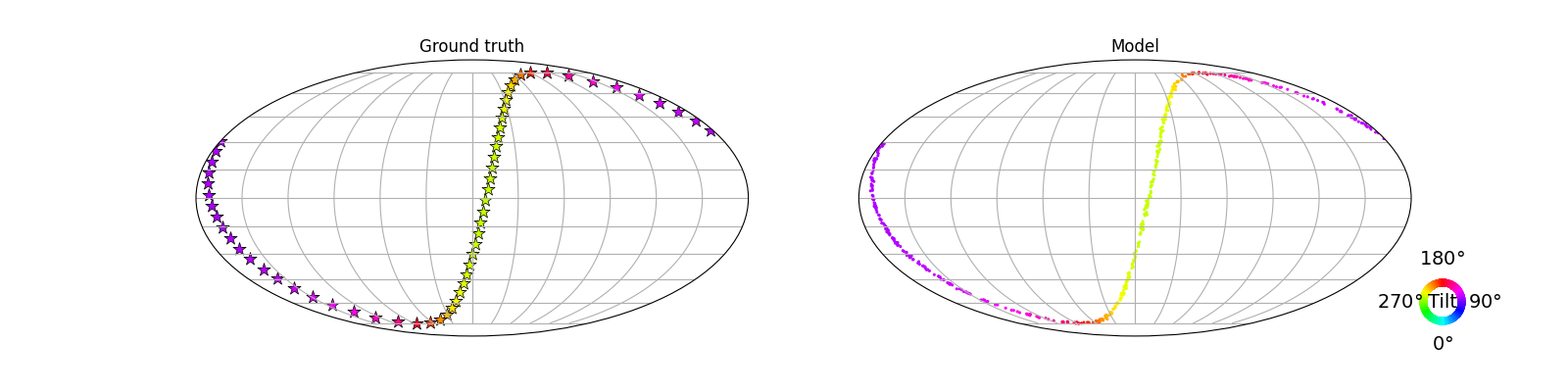}} \\
    \caption{Samples from the ground truth (plotted as stars, middle) and our pose estimation DMM (right) conditioned on 2D views of two shapes (left). The axis of rotation and rotation angle are represented by position and colour respectively.}
    \label{fig:symsolresults}
\end{figure}

\vspace{-0.5cm}
\subsection{Approximation of distributions over measures using Wright--Fisher diffusions}

Finally, we present an example of learning to approximate a distribution over measures on a finite state space $E = \{1, \dots, N\}$. In this case $\X = \cal{P}(E)$, the space of measures on $E$. This is of particular interest in compositional data analysis \citepcont{greenacre2021compositional}. Elements of $\X$ can be parameterised by tuples of real numbers $\p = (p_1, \dots, p_N) \in [0,1]^N$ such that $\sum_{i=1}^N p_i = 1$. We could approximate the data distribution using a diffusion model on $\R^N$, but such a model would not reflect the fact that our distribution should be supported on a submanifold, the simplex. Using the standard setup for manifold diffusions as in Example \ref{ex:manifolddiffusion} would not respect the boundary of the simplex. Other methods have been presented in the literature, but they rely on either reflected diffusions \citepcont{lou2023reflected} or on projections of the simplex \citepcont{richemond2022categorical}.

We therefore use Wright--Fisher diffusions, a process used in population genetics to model the evolution of allele frequencies, as our class of generative processes. A Wright--Fisher process has generator $\L = \del_t + \frac{1}{2} \sum_{i,j = 1}^N p_i \lr{\delta_{ij} - p_j} \frac{\del^2}{\del p_i \del p_j} + \sum_{i,j=1}^N q_{ij}p_i \frac{\del}{\del p_j}$, where $(q_{ij})_{i,j = 1, \dots, N}$ is some matrix such that $\sum_{j=1}^N q_{ij} = 0$ for each $i = 1, \dots, N$. The process takes values in the space of measures on $E$, and so respects the structure of our data distribution \citepcont{ethier1993transition}. For specific choices of $q_{ij}$, the process converges to a known invariant distribution and we can calculate the implicit score matching loss. For details of the theoretical setup, see Appendix \ref{app:wrightfisher}.

We evaluate the proposed method by modelling $p_{\data}(\x) = \frac{1}{M} \sum_{m=1}^{M} \textup{Dirichlet}(\alpha_m)$, a mixture of Dirichlet distributions with parameters $\alpha_m\in\R^N$, for various values of $N$. Fig.\ \ref{fig:mix_dir} shows two visualisations of samples drawn from our DMM compared to ground truth samples in dimension $N=3$. Our model is able to accurately approximate $p_{\data}(\x)$. For further evaluations and experimental details, see Appendix \ref{app:wrightfisherexperiments}.

\begin{figure}
    \centering
    \includegraphics[trim={0 540 0 540}, clip, width=0.45\textwidth]{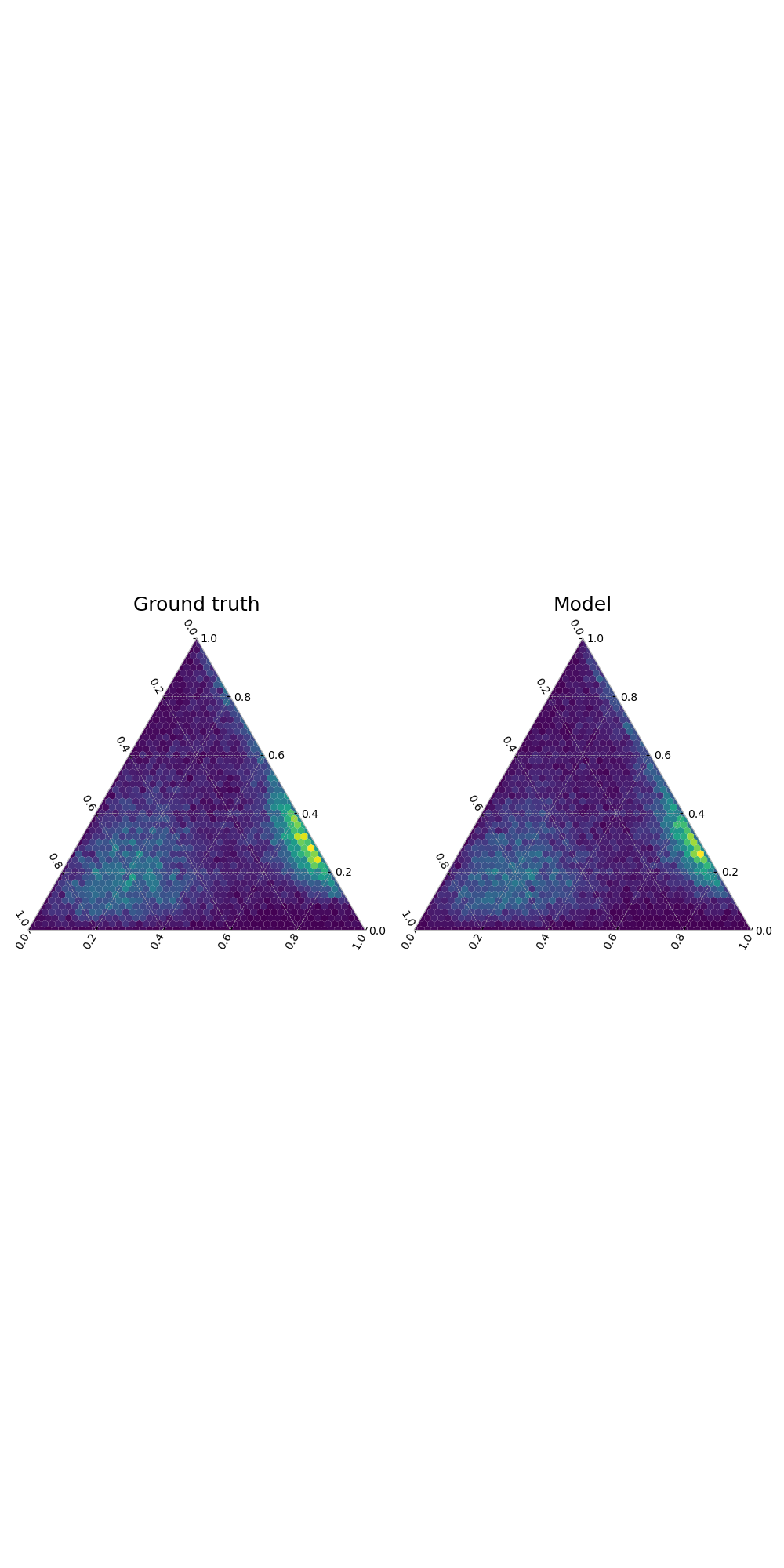}
    \includegraphics[width=0.54\textwidth]{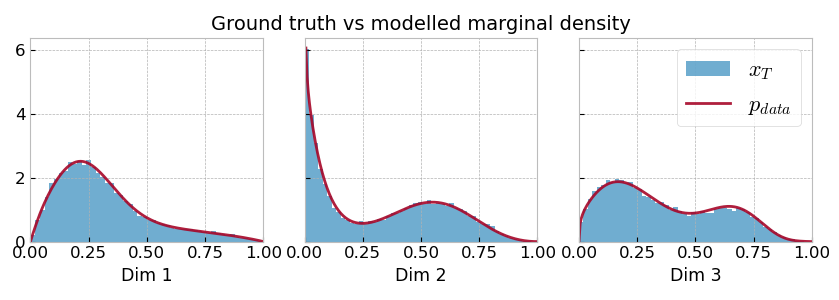}
    \caption{Histograms of samples from our simplex DMM and the ground truth mixture of Dirichlet distributions for dimension $N = 3$, plotted over the whole space as a ternary plot (left) and over the marginals per dimension (right).}
    \label{fig:mix_dir}
\end{figure}

\section{Discussion}
We have provided here a general framework which allows us to extend denoising diffusion models to general state-spaces. The resulting DMMs can be trained with principled objectives and used for inference, generalizing along the way score matching ideas. Their applicability and performance have been demonstrated on a range of problems. From a methodological point of view, the proposed framework is general enough to accommodate, for example, general noising processes, mixed continuous/discrete processes and some infinite-dimensional settings with finite representations (though our assumptions on the state space (see Appendix \ref{app:assstatespace}) may fail to hold in the infinite-dimensional setting so more care is required).

However, we still lack a proper theoretical understanding of these models. Under realistic assumptions on the data distribution,  \citecont{debortoli2022convergence} and \citecont{chen2022sampling} show that diffusion models on $\R^d$ can in theory learn essentially any distribution given a good enough score approximation and infinite data. However finite sample guarantees are currently absent. Moreover, $p_{\data}$ is typically an empirical measure as we only have access to a finite set of datapoints, so $q_t$ is a mixture of Gaussians for an Ornstein--Uhlenbeck noising diffusion and its score $\nabla \log q_t$ is thus available. If we were simulating samples using the exact time reversal of this diffusion, we would simply recover the empirical distribution. It is because we are approximating the time-reversal and in particular using an approximation of the scores that we are able to obtain novel samples. It is not yet clear why the approximation of the score using neural networks appears to provide perceptually realistic samples for many applications.

The effectiveness of such methods for inference, even in scenarios where standard MCMC or ABC techniques are not applicable \citepcont{sharrock2022sequential,geffner2023compositional}, 
may also be considered surprising. One perspective on the training process is that it involves the model constructing its own summary statistics that allow it to perform inference effectively on the training observations. It is not yet well understood why the summary statistics the model learns appear empirically effective, or what sorts of summary statistics our training procedure biases the model towards.

Overall, this contribution shows how the range of existing models relate to each other and may help applying DMMs in practice to a large variety of problems. However, our understanding of such models is still incomplete and deserves further attention.

\section*{Acknowledgments}
Joe Benton was supported by the EPSRC Centre for Doctoral Training in Modern Statistics and Statistical Machine Learning (EP/S023151/1) and Yuyang Shi by the Huawei UK Fellowship Programme. Arnaud Doucet acknowledges support of the UK Dstl and EPSRC grant EP/R013616/1. This is part of the collaboration between US DOD, UK MOD and UK EPSRC under the Multidisciplinary University Research Initiative. He also acknowledges support from the EPSRC grants CoSines (EP/R034710/1) and Bayes4Health (EP/R018561/1).

%% file: appdx.tex
\appendix

\section{Background on Feller processes}
\label{app:fellerbackground}

We recall some basic definitions and properties associated with Feller processes which we use for the derivations in Section \ref{sec:general_objective}. Our principal source is \citeappdx{dong2003feller}.

\subsection{Definition of a Feller process}
\label{app:fellerdefinition}

Let $S$ be a locally compact, separable metric space and let $C_0(S)$ denote the set of continuous functions $f : S \rightarrow \R$ such that for any $\epsilon > 0$ there exists a compact $K \subseteq S$ such that $|f(x)| < \epsilon$ for all $x \not \in K$. Also, let $||f||$ denote the supremum norm on $C_0(S)$.

\begin{definition}[Feller process]
A time-homogeneous Markov process $(X_t)_{t \geq 0}$ with state space $S$ and associated transition semigroup $(P_t)_{t \geq 0}$ is a Feller process if:
\begin{itemize}
    \item $P_t f \in C_0(S)$ for all $f \in C_0(S)$ and $t \geq 0$.
    \item $||P_t f|| \leq ||f||$ for all $f \in C_0(S)$.
    \item $P_t f(x) \rightarrow f(x)$ as $t \rightarrow 0$ for all $x \in S$ and $f \in C_0(S)$.
\end{itemize}
\end{definition}

\begin{definition}[Generator of a Feller process]
Suppose $X$ is a Feller process on $S$ as above and $f$ is a function in $C_0(S)$. If the limit
\begin{equation*}
    \A f := \lim_{s \rightarrow 0} \frac{P_s f - f}{s}
\end{equation*}
exists in $C_0(S)$, we say that $f$ is in the domain of the generator of $X$. We call the operator $\A$ defined in this way the generator of $X$ and denote its domain by $\D(\A)$.
\end{definition}

In the main text, we are concerned with Feller processes $\extX$, $\extY$ defined on the extended space $\cal{S} = \X  \times [0,\infty)$ which are constructed by taking a time-inhomogeneous Markov process $X$ on $\X$ and defining $\extX = (X_t, t)_{t \geq 0}$. In this setting, we have the following variant of Dynkin's formula.

\begin{lemma}[Dynkin's formula] If $\extX = (X_t, t)_{t \geq 0}$ is a Feller process on $\cal{S}$ with generator $\A$ and $f \in \D(\A)$, then
\begin{equation*}
    M_t^f = f(X_t, t) - f(X_0, 0) - \int_0^t \A f(X_s, s)~\d s
\end{equation*}
is a martingale with respect to the natural filtration of $\extX$.
\end{lemma}

\begin{proof}
See Theorem 27.20 in \citeappdx{dong2003feller}.
\end{proof}

\subsection{Adjoint of a generator}
\label{app:dualgeneratordefinition}

Given a state space $S$ and a reference measure $\nu$ on $S$, we can define an inner product on $C_0(S)$ by letting
\begin{equation*}
    \langle f , h \rangle = \int_S fh~\d \nu
\end{equation*}
for all $f,h \in C_0(S)$ such that the integral exists. This induces a Hilbert space structure on $C_0(S)$ and allows us to make the following definition, from \citeappdx{yosida1965functional}.

\begin{definition}[Adjoint of an operator]
\label{def:dual}
Given operator $\A$ with domain $\D(\A)$ contained in $C_0(S)$, we define the adjoint operator $\A^\ast$ acting at function $f \in C_0(S)$ by
\begin{equation*}
    \langle \A^\ast f, h \rangle = \langle f, \A h \rangle \hspace{5mm} \text{ for all } h \in \D(\A).
\end{equation*}
The domain $\D(\A^\ast)$ of $\A^\ast$ is the set of all functions $f$ such that there exists some function $\A^\ast f$ for which the above holds.
\end{definition}

\section{Assumptions for Section \ref{sec:general_objective}}
\label{app:assumptions}

Here, we list the assumptions under which our derivations in Section \ref{sec:general_objective} hold. Note that these assumptions can be verified in several relevant cases (see Appendix \ref{app:particularspaces}).

\subsection{Assumptions on the state space $\X$}
\label{app:assstatespace}

\begin{assumption}
\label{ass:statespace}
The state space $\X$ is a locally compact, separable metric space. In addition, there exists a reference measure $\nu$ on $\X$ with respect to which all relevant probability distributions are absolutely continuous.
\end{assumption}

\subsection{Assumptions on the marginals $p$ and $q$}
\label{app:asspq}

\begin{assumption}
\label{ass:regularity_of_p} We have $p_t \in \D(\hat \K^\ast)$ for each $t \in [0,T]$, where $\hat \K^\ast$ is the adjoint of the spatial part of the operator $\K$. In addition, $p$ is differentiable with respect to $t$ and $\partial_t p$ is bounded.
\end{assumption}

\subsection{Assumptions on the generators $\K$ and $\L$}
\label{app:assgenerator}

\begin{assumption}
\label{ass:feller}
$\extX$ and $\extY$ are Feller processes with associated transition semigroups $(P_t)_{t \geq 0}$, $(Q_t)_{t \geq 0}$ and generators $\K, \L$ respectively.
\end{assumption}

\begin{assumption}
\label{ass:decomposition}
$\K$ decomposes as $\K = \partial_t + \hat\K$, where $\hat\K f$ is defined only in terms of the spatial arguments of $f$, so we may view it as an operator on (a subset of) $C_0(\X)$.
\end{assumption}

\begin{assumption}
\label{ass:domainsubspace}
There exists a subset $\D_0 \subseteq \D(\hat \K) \cap L^2(\X, \nu)$ which is dense in $L^2(\X, \nu)$, satisfies $\hat \K h \in \D_0$ for all $h \in \D_0$ and such that every function in $\D_0$ is bounded and has compact support.
\end{assumption}

\subsection{Assumptions on $\M$ and $c$}
\label{app:assMc}

\begin{assumption}
\label{ass:Mandc}
The function $c : \cal{S} \rightarrow \R$ is bounded, and the function $v : \cal{S} \rightarrow \R$ is bounded, in  $\D(\M)$ and satisfies $\int_0^T \E{|\M v (Z_s, s)|^2} \d s < \infty$.
\end{assumption}

\subsection{Assumptions on $\beta$}
\label{app:assalpha}

\begin{assumption}
\label{ass:alpha}
The functions $\beta^{-1}$, $\beta^{-1} v$, $\log \beta$ and $\log v$ are in $\D(\L)$, $\beta^{-1}$ and $\beta \L (\beta^{-1})$ are both bounded, and $\beta \in \D(\hat \L ^\ast)$.
\end{assumption}

\section{Stochastic process theory}
\label{app:stochasticprocesstheory}

We provide full statements of the general stochastic process results used in Section \ref{sec:general_objective}. For completeness, we also provide proofs of the given results adapted to our setting.

\begin{theorem}[Fokker--Planck]
\label{generalised_FP}
Let $(X_t)_{t \in [0,T]}$ be a Markov process with generator $\K$ and marginals $p_t$ satisfying the assumptions in Appendix \ref{app:assumptions}. Then $p$ satisfies the forward Kolmogorov equation $\partial_t p = \hat\K^\ast p$ for $\nu$-almost every $\x$.
\end{theorem}

\begin{proof} For any $h \in \D_0$, by Assumptions \ref{ass:regularity_of_p} and \ref{ass:domainsubspace} we may write
\begin{align*}
    \langle \partial_t p - \hat\K^\ast p, h \rangle &= \int_\X (\partial_t p)h - p (\hat\K h) \; \d\nu \\
    &= \partial_t \E{h(X_t)} - \E{\hat\K h(X_t)}.
\end{align*}
Applying Dynkin's formula to $f(\x,t) = h(\x)$, taking expectations and using Fubini's theorem, we see that
\begin{equation*}
    \E{h(X_t)} - \E{h(X_0)} = \int_0^t \E{\hat\K h(X_s)} \; \d s.
\end{equation*}
Differentiating with respect to $t$, we deduce that $\langle \partial_t p - \hat\K^\ast p, h \rangle = 0$. Since this holds for all $h \in \D_0$ and $\D_0$ is dense in $L^2(\X, \nu)$, we conclude that $\partial_t p - \hat\K^\ast p = 0$ holds $\nu$-a.e. as required.
\end{proof}

\begin{theorem}[Feynman--Kac]
\label{generalised_FK}
Let $\extZ = (Z_t, t)_{t \geq 0}$ be a Feller process on $\cal{S}$ with generator $\M$. Suppose that we are given functions $v, c: \cal{S} \rightarrow \R$ and $h : \X \rightarrow \R$ such that $\M, v, c$ solve equation $\M v + c v = 0$ (as in Assumption \ref{ass:existenceM}) with boundary condition $v(\cdot, T) = h(\cdot)$. Suppose also that Assumption \ref{ass:Mandc} is satisfied. Then we have
\begin{equation*}
    v(\x, \tau) = \E{h(Z_T) \exp \left\{ \int_\tau^T c(Z_s, s) \; \d s \right\} \; \Bigg|\; Z_\tau = \x}
\end{equation*}
for all $0 \leq \tau \leq T$.
\end{theorem}

\begin{proof} This result is well-known in the case of Euclidean diffusion processes \citepappdx{karatzas1991brownian}. In the general case, the proof relies on the theory of semimartingales (see for example \cite{metivier1982semimartingales}). Fix $\tau \in [0,T]$ and for all $t \in [\tau, T]$ define
\begin{equation*}
    \label{FK_martingale}
    S_t = v(Z_t, t) \exp \left\{ \int_\tau^t c(Z_s, s) \d s \right\}
\end{equation*}
along with
\begin{equation*}
    V_t = v(Z_t, t), \hspace{5mm} U_t = \exp\left\{\int_\tau^t c(Z_s, s) \d s\right\}.
\end{equation*}
Each of these processes is clearly a semimartingale, and so we may define $\d S_t$, $\d U_t$ and $\d V_t$ accordingly \citepappdx{metivier1982semimartingales}. The following lemma will allow us to express $\d S_t$ in terms of $\d U_t$ and $\d V_t$.

\begin{lemma}[Integration by parts for semimartingales]
\label{product_rule}
If $U$ and $V$ are semimartingales and at least one is continuous then we have
\begin{equation*}
    \d (U_t V_t) = U_{t-} \d V_t + U_{t-} \d V_t + \d [U, V]^c_t,
\end{equation*}
where $[\cdot,\cdot]^c_t$ denotes the quadratic covariation.
\end{lemma}

\begin{proof}
This is Theorem 2.7.4(ii) of \citeappdx{pulido2011semimartingales}, or follows from applying Theorem 27.1 of \citeappdx{metivier1982semimartingales} to the function $\varphi(U,V) = UV$.
\end{proof}

Since $v \in \D(\M)$ by Assumption \ref{ass:Mandc}, by Dynkin's formula we have that $V$ is a semimartingale and we may decompose
\begin{equation*}
    \d V_t = \M v \d t + \d M^v_t
\end{equation*}
where $M_t^v$ is a martingale. Also, since $c(x,t)$ is bounded by Assumption \ref{ass:Mandc}, $U$ is a continuous, adapted, previsible process of finite variation and satisfies
\begin{equation*}
    \d U_t = c(Z_t, t) \exp\left\{\int_\tau^t c(Z_s, s) \; \d s\right\} \d t.
\end{equation*}
In addition, note that $\d [U, V]_t^c = 0$ since $U$ is continuous and of finite variation. Therefore, by Lemma \ref{product_rule}, we can calculate
\begin{align*}
    S_t - S_\tau &= \int_\tau^t U_{s-} \d V_s + \int_\tau^t V_{s-} \d U_s + [U, V]^c_t \\
    &= \int_\tau^t U_s \big\{ \M v + cv \big\} \d s + \int_\tau^t U_s \; \d M_s^v \\
    &= \int_\tau^t U_s \; \d M_s^v
\end{align*}
where we have used that $\M v + cv = 0$ in the last line. Therefore, $S$ can be expressed as a stochastic integral with respect to the martingale $M^v$.

The conditions we have imposed through Assumption \ref{ass:Mandc} on $c$ and $v$ imply that $U$ is bounded and $M^v$ is square-integrable. It follows, for example from Theorem 24.4.5 in \citepappdx{metivier1982semimartingales}, that $S$ is a local martingale and hence, since it is also bounded, a true martingale. We then have that
\begin{align*}
    v(\x, \tau) &= \E{S_\tau | Z_\tau = \x} = \E{S_T | Z_\tau = \x} \\
    &= \E{h(Z_T) \exp \left\{ \int_\tau^T c(Z_s, s) \d s \right \} \Bigg | Z_\tau = \x}
\end{align*}
as required.
\end{proof}

\begin{theorem}[Girsanov]
\label{generalised_girsanov}
Let $\extY = (Y_t, t)_{t \geq 0}$ and $\extZ = (Z_t, t)_{t \geq 0}$ be Feller processes on $\cal{S}$ with generators $\L$, $\M$ and path measures $\Q$, $\P$ respectively, such that $Y_0$ and $Z_0$ have the same law. Suppose also that there exists a bounded, measurable function $\alpha : \cal{S} \rightarrow (0, \infty)$ in $\D(\L)$ such that $\alpha^{-1} \L \alpha$ is bounded, and such that
\begin{equation}
\label{eq:girsanovgeneratoreqn}
    \alpha \M f = \L(f\alpha) - f \L \alpha
\end{equation}
for all functions $f$ such that $f \in \D(\M)$ and $f\alpha \in \D(\L)$. Then we have
\begin{equation}
\label{eq:change_of_measure}
    \frac{\d \P}{\d \Q}(\omega) = \frac{\alpha(\omega_T, T)}{\alpha(\omega_0, 0)} \exp \Big\{ - \int_0^T \frac{\L \alpha (\omega_s, s)}{\alpha(\omega_s, s)} \; \d s \Big\}.
\end{equation}
\end{theorem}

\begin{proof} This essentially follows from the work of \citeappdx{palmowski2002technique}. Using their terminology, their Proposition 3.2 implies $\alpha$ is a good function, so the RHS of Equation (\ref{eq:change_of_measure}) is a martingale and we may define a measure $\tilde \P$ by
\begin{equation*}
    \frac{\d \tilde \P}{\d \Q}(\omega) = \frac{\alpha(\omega_T, T)}{\alpha(\omega_0, 0)} \exp \Big\{ - \int_0^T \frac{\L \alpha (\omega_s, s)}{\alpha(\omega_s, s)} \; \d s \Big\}.
\end{equation*}
Under the measure $\tilde \P$, the canonical process $(\omega_t)_{t \in [0,T]}$ is still Markov. By the proof of their Theorem 4.2, we see that
\begin{equation*}
    \tilde D^f_t = f(Y_t, t) - \int_0^t \M f(Y_s, s) \d s
\end{equation*}
is a martingale for all sufficiently smooth functions $f$, implying that $\M$ is the generator of $(\omega_t)_{t \in [0,T]}$ under $\tilde \P$. It follows that $(\omega_t)_{t \in [0,T]}$ has the same law under $\tilde \P$ as $\extZ$ does under $\Q$, which is sufficient to prove the result since $\extY$ and $\extZ$ are Feller.
\end{proof}

\section{Proof from Section \ref{sec:general_objective}}
\label{app:genobjectiveproofs}

We give the proofs of Lemma \ref{lem:gen_lem_1} and Theorem \ref{thm:inference} from Section \ref{sec:general_objective}.

\genlem*

\begin{proof} Let us define $\hat \M$ to be the operator such that $\M = \hat \M + \del_t$. Then, since $\hat{\M} + c = \M + c- \del_t = \hat\K^\ast$, for any sufficiently rapidly decaying test function $f$ we have
\begin{align*}
    \langle \hat \M f, 1 \rangle + \langle cf, 1 \rangle & = \langle \hat \K^\ast f , 1 \rangle = \langle f, \hat\K 1 \rangle = 0,
\end{align*}
so $\langle \hat \M f, 1 \rangle = - \langle c , f \rangle$. Assumption \ref{ass:generatorrelation}, which states that $\beta^{-1} \M f = \L(\beta^{-1} f) - f \L (\beta^{-1})$ for all sufficiently rapidly decaying $f$, can be rearranged to $\hat \M f = \beta \hat \L (\beta^{-1} f) - \beta f \hat \L (\beta^{-1})$. So, it follows that
\begin{align*}
    \langle c, f \rangle & = - \langle \beta \hat \L (\beta^{-1} f) , 1 \rangle + \langle \beta f \hat \L (\beta^{-1}), 1 \rangle \\
    & = - \langle f, \beta^{-1} \hat \L^\ast \beta \rangle + \langle f, \beta \hat \L (\beta^{-1}) \rangle
\end{align*}
for any sufficiently rapidly decaying $f$. We conclude that $\beta^{-1} \hat \L^\ast \beta = \beta \hat \L (\beta^{-1}) - c$.

Next, using Assumption \ref{ass:generatorrelation} with $f=v$ we can write
\begin{align*}
    v^{-1} \beta \L(\beta^{-1} v) & = \beta \L (\beta^{-1}) + v^{-1} \M v \\
    & = \beta \L (\beta^{-1}) - c \\
    & = - \del_t \log \beta + \beta \hat \L (\beta^{-1}) - c \\
    & = - \del_t \log \beta + \beta^{-1} \hat \L^\ast \beta.
\end{align*}
Finally, note that $- \L \log \beta + \hat\L \log \beta = - \del_t \log \beta$. Combining this with the final line above, we get the desired result.
\end{proof}

\inferencethm*

\begin{proof}
Applying Theorem \ref{thm:ELBO} to the generative process $X^{\xiobs^\ast}$ conditioned on observation $\xiobs^\ast$,
\begin{multline*}
    \log p_T(\x_0 | \xiobs^\ast) \geq \bb{E}_{\Q}\Big[\log p_0(Y_T) \Big | Y_0 = \x_0 \Big] \\ -\int_0^T \bb{E}_{\Q}\bigg[\frac{\hat\L^\ast \beta(\x_t, \xiobs^\ast, t)}{\beta(\x_t, \xiobs^\ast, t)} + \hat\L \log \beta(\x_t, \xiobs^\ast, t) \; \bigg | \; Y_0 = \x_0 \bigg] \d t.
\end{multline*}
Replacing $\xiobs^\ast$ by $\xiobs_0$, letting $(\x_0, \xiobs_0) \sim p_\data$ and taking expectations, we get
\begin{multline*}
    \E[p_{\data}(\x_0, \xiobs_0)]{\log p_T(\x_0 | \xiobs_0)} \geq \bb{E}_{q_T(\x_T)}\Big[\log p_0(\x_T)\Big] \\ - \int_0^T \bb{E}_{q(\x_t, \xiobs_0)}\bigg[\frac{\hat\L^\ast \beta(\x_t, \xiobs_0, t)}{\beta(\x_t, \xiobs_0, t)} + \hat\L \log \beta(\x_t, \xiobs_0, t) \bigg] \d t.
\end{multline*}
For any given $\xiobs$, we have
\begin{multline*}
    \bb{E}_{q(\x_t | \xiobs)}\bigg[\frac{\hat\L^\ast \beta(\x_t, \xiobs, t)}{\beta(\x_t, \xiobs, t)} + \hat\L \log \beta(\x_t, \xiobs, t) \bigg] \\ = \E[q(\x_0, \x_t | \xiobs)]{\frac{\L( q_{\cdot|0}(\cdot | \x_0, \xiobs) / \beta(\cdot, \xiobs, \cdot))(\x_t, t)}{q_{t|0}(\x_t | \x_0, \xiobs) / \beta (\x_t, \xiobs, t)} - \L \log (q_{\cdot|0}(\cdot | \x_0, \xiobs) / \beta (\cdot, \xiobs, \cdot))(\x_t, t)} + \const
\end{multline*}
by the argument of Appendix \ref{app:equiv_objectives} (see below), where the constant depends only on the dynamics of the forward process. Substituting $\xiobs_0$ for $\xiobs$ and taking expectations over $\xiobs_0 \sim p_\data$, noting that $q_{t|0}(\x_t | \x_0, \xiobs_0) = q_{t|0}(\x_t | \x_0)$, we get
\begin{multline*}
    \bb{E}_{q(\x_t, \xiobs_0)}\bigg[\frac{\hat\L^\ast \beta(\x_t, \xiobs_0, t)}{\beta(\x_t, \xiobs_0, t)} + \hat\L \log \beta(\x_t, \xiobs_0, t) \bigg] \\ = \E[q(\x_0, \x_t, \xiobs_0)]{\frac{\L( q_{\cdot|0}(\cdot | \x_0) / \beta(\cdot, \xiobs_0, \cdot))(\x_t, t)}{q_{t|0}(\x_t | \x_0) / \beta(\x_t, \xiobs_0, t)} - \L \log (q_{\cdot|0}(\cdot | \x_0) / \beta (\cdot, \xiobs_0, \cdot))(\x_t, t)} + \const.
\end{multline*}
It follows that
\begin{multline*}
    \E[p_{\data}(\x_0, \xiobs_0)]{\log p_T(\x_0 | \xiobs_0)} \geq \bb{E}_{q_T(\x_T)}\Big[\log p_0(\x_T)\Big] \\ - \int_0^T \E[q(\x_0, \x_t, \xiobs_0)]{\frac{\L( q_{\cdot|0}(\cdot | \x_0) / \beta(\cdot, \xiobs_0, \cdot))(\x_t, t)}{q_{t|0}(\x_t | \x_0) / \beta(\x_t, \xiobs_0, t)} - \L \log (q_{\cdot|0}(\cdot | \x_0) / \beta (\cdot, \xiobs_0, \cdot))(\x_t, t)} \d t \\ + \const.
\end{multline*}
The first term on the RHS and the constant are independent of the dynamics of the reverse process. Hence minimising
\begin{equation*}
    \mathcal{I}_{\textup{DSM}}(\beta) = \int_0^T \E[q(\x_0, \x_t, \xiobs_0)]{\frac{\L(q_{\cdot|0}(\cdot | \x_0) / \beta(\cdot, \xiobs_0, \cdot))(\x_t,t)}{q_{t|0}(\x_t | \x_0) / \beta(\x_t, \xiobs_0, t)} - \L \log (q_{\cdot|0}(\cdot | \x_0) / \beta(\cdot, \xiobs_0, \cdot))(\x_t,t)} \d t
\end{equation*}
is equivalent to maximising a lower bound on $\E[p_{\data}(\x_0, \xiobs_0)]{\log p_T(\x_0 | \xiobs_0)}$, which is the expected model log-likelihood.
\end{proof}

\section{Equivalence of generalised score matching objectives}
\label{app:equiv_objectives}

First, we show that $\cal{I}_{\textup{ISM}}$ and $\I_\textup{DSM}$ are equivalent training objectives.
\begin{align*}
    & \hspace{5mm} \E[q_{0,t}(\x_0, \x_t)]{\frac{\L( q_{\cdot|0}(\cdot | \x_0) / \beta(\cdot, \cdot))(\x_t, t)}{q_{t|0}(\x_t | \x_0) / \beta(\x_t, t)} - \L \log (q_{\cdot|0}(\cdot | \x_0) / \beta(\cdot, \cdot))(\x_t, t)} \\
    & = \int_\X \int_\X q_{0,t}(\x_0, \x_t) \left\{ \frac{\L( q_{\cdot|0}(\cdot | \x_0)/ \beta(\cdot, \cdot))(\x_t, t)}{q_{t|0}(\x_t | \x_0) / \beta(\x_t, t)} - \L \log (q_{\cdot|0}(\cdot | \x_0) / \beta(\cdot, \cdot))(\x_t, t) \right\} \d\nu(\x_0) \d\nu(\x_t) \\
    & = \int_\X q_0(\x_0) \int_\X \left\{ \beta(\x_t, t) \hat\L\left( \frac{q_{t|0}(\cdot | \x_0)}{\beta(\cdot, t)}\right)(\x_t) - q_{t|0}(\x_t | \x_0) \hat \L \log \left( \frac{q_{t|0}(\cdot | \x_0)}{\beta(\cdot, t)}\right)(\x_t) \right\} \d\nu(\x_t) \d\nu(\x_0) \\
    & = \int_\X q_0(\x_0) \int_\X q_{t|0}(\x_t|\x_0) \left\{\frac{\hat\L^\ast \beta(\x_t, t)}{\beta(\x_t,t)} + \hat\L \log \beta(\x_t,t) \right\} \d\nu(\x_t) \d\nu(\x_0) + \const \\
    & = \E[q_t(\x_t)]{\frac{\hat\L^\ast \beta(\x_t)}{\beta(\x_t)} + \hat\L \log \beta(\x_t)} + \const,
\end{align*}
where the constants depend only on the dynamics of the forward process and so are fixed during training. Integrating from $t = 0$ to $t = T$, we conclude that $\I_{\textup{ISM}}$ and $\I_{\textup{DSM}}$ are equivalent.

There is also an \emph{explicit score matching} form of the general DMM training objective as follows: 
\begin{equation*}
    \I_{\textup{ESM}}(\beta) = \int_0^T \E[q_t(\x_t)]{\frac{\L(q_\cdot(\cdot) / \beta(\cdot, \cdot))(\x_t, t)}{q_t(\x_t) / \beta(\x_t, t)} - \L \log (q_\cdot(\cdot) / \beta(\cdot, \cdot))(\x_t, t)} \d t.
\end{equation*}
To see that this is equivalent to $\I_{\textup{ISM}}$ and $\I_{\textup{DSM}}$, observe
\begin{align*}
    & \hspace{5mm} \E[q_t(\x_t)]{\frac{\L(q_\cdot(\cdot) / \beta(\cdot, \cdot))(\x_t, t)}{q_t(\x_t) / \beta(\x_t,t )} - \L \log (q_\cdot(\cdot) / \beta(\cdot, \cdot))(\x_t, t)} \\
    & = \int_\X q_t(\x_t) \left\{ \frac{\L(q_\cdot(\cdot) / \beta(\cdot, \cdot))(\x_t, t)}{q_t(\x_t) / \beta(\x_t, t)} - \L \log (q_\cdot(\cdot) / \beta(\cdot, \cdot))(\x_t, t) \right\} \d\nu(\x_t) \\
    & = \int_\X \left\{ \beta(\x_t, t) \hat \L\left( \frac{q_\cdot(\cdot)}{\beta(\cdot, \cdot)} \right) - q_t(\x_t) \hat \L \log \left( \frac{q_\cdot(\cdot)}{\beta(\cdot, \cdot)} \right) \right\} \d\nu(\x_t) \\
    & = \int_\X q_t(\x_t) \left\{ \frac{\hat \L^\ast \beta(\x_t, t)}{\beta(\x_t, t)} + \hat \L \log \beta(\x_t, t) \right\} \d \nu(\x_t) + \const \\
    & = \E[q_t(\x_t)]{\frac{\hat\L^\ast \beta(\x_t)}{\beta(\x_t)} + \hat\L \log \beta(\x_t)} + \const
\end{align*}
and integrate from $t=0$ to $t=T$.

\section{Application to particular spaces}
\label{app:particularspaces}

In this section, we show how our general framework can be applied in some particular cases of interest, namely to Euclidean diffusion processes, continuous-time Markov Chains on finite discrete state spaces, diffusions on Riemannian manifolds and the Wright--Fisher diffusion on the simplex.

A recurring theme we see in each example is that the default parameterisation given by our framework in terms of $\beta$ is sub-optimal, either because we expect it to lead to numerical instabilities when optimising the training objective, or because it only captures a restricted subset of the class of reverse processes we are interested in. However, in each case it turns out to be possible to reparameterise the generative process in a way which captures a wider class of processes and lets us interpret the training objective on this wider class. This allows us to optimise our generative process over this wider class of processes. In addition this reparameterisation typically leads to a form of the objective that we expect to be more numerically stable in practice.

\subsection{Real vector spaces}
\label{app:realdiffusion}

We show how our framework recovers the setup of \citeappdx{song2021score}, described in Section \ref{sec:diffusionbackground}, in the case where $\K$ and $\L$ are the Euclidean diffusion processes given in Example \ref{ex:realdiffusiongenerator}. For convenience, we recall that $X$ and $Y$ satisfy the SDEs
\begin{equation}
\label{eq:diffusionSDEs}
    \d X_t = \mu(X_t, t) \d t + \d \hat B_t, \hspace{5mm}
    \d Y_t = b(Y_t, t) \d t + \d B_t,
\end{equation}
respectively, and the corresponding generators are
\begin{equation*}
    \K = \partial_t + \mu \cdot \nabla + \frac{1}{2} \Delta, \hspace{5mm}
    \L = \partial_t + b \cdot \nabla + \frac{1}{2} \Delta.
\end{equation*}

First, we check the assumptions made in Appendix \ref{app:assumptions}. If we let our reference measure $\nu$ be the Lebesgue measure, then Assumption \ref{ass:statespace} holds. Assumption 5 is satisfied whenever $b$ and $\mu$ are Lipschitz functions \citeappdx[Corollaries 19.27 and 19.31]{schilling2012brownian}, and Assumption \ref{ass:decomposition} follows given the form of $\K$ above. For Assumption \ref{ass:domainsubspace} we take $\D_0 = C_c^\infty(\R^d)$, the set of infinitely differentiable functions with compact support, and note that this is dense in $L^2(\X, \nu)$. Finally, we assume that the reverse process and $p_0$ are sufficiently regular that Assumptions \ref{ass:regularity_of_p}, \ref{ass:Mandc} and \ref{ass:alpha} hold. 

Using integration by parts, we can calculate the adjoint of $\hat \K$. We have
\begin{align*}
    \int f \hat \K h \d \nu & = \int f \lr{\mu \cdot \nabla h + \frac{1}{2} \Delta h } \d \nu \\
    & = - \int h \nabla \cdot (f \mu) \d \nu - \frac{1}{2} \int \nabla f \cdot \nabla h \d \nu \\
    & = \int h \lr{- \mu \cdot \nabla f - (\nabla \cdot \mu) f + \frac{1}{2} \Delta f } \d \nu,
\end{align*}
assuming $f$ and $h$ are sufficiently regular that all boundary terms are zero. Therefore,
\begin{equation*}
    \hat\K^\ast = - \mu \cdot \nabla - (\nabla \cdot \mu) + \frac{1}{2}\Delta.
\end{equation*}
We see that Assumption \ref{ass:existenceM} holds if we let $c = - (\nabla \cdot \mu)$ and
\begin{equation*}
    \M = \partial_t - \mu \cdot \nabla + \frac{1}{2}\Delta,
\end{equation*}
noting that this is the generator of another diffusion process $\extZ$ satisfying the SDE
\begin{equation*}
    \d Z_t = - \mu(Z_t, T-t) \d t + \d B'_t.
\end{equation*}

Given this form of $\L$ and $\M$, Assumption \ref{ass:generatorrelation} then becomes
\begin{equation*}
    - \beta^{-1} \mu \cdot \nabla f + \frac{1}{2} \beta^{-1} \Delta f = b \cdot \nabla (\beta^{-1} f) + \frac{1}{2} \Delta (\beta^{-1} f) - f b \cdot \nabla (\beta^{-1}) - \frac{1}{2} f \Delta (\beta^{-1}),
\end{equation*}
which reduces to
\begin{equation}
\label{eq:reverseparamdiffusion}
    \nabla \log \beta = \mu + b,
\end{equation}
for some bounded measurable function $\beta$. This puts a restriction on the class of reverse processes $\K$ we may use; the condition that the drift $\mu$ must be expressible as $-b + \nabla \log \beta$ for some $\beta$ is not automatically satisfied. However, the true time-reversal of the forward process will satisfy this property. In addition, we will show that we may reparameterise the training objective so that it can be interpreted for a broader class of reverse processes.

Assuming for the moment that Assumption \ref{ass:generatorrelation} does hold, we can evaluate
\begin{align*}
    \Phi(f) & = \frac{\L f}{f} - \L \log f \\
    & = \frac{b \cdot \nabla f}{f} + \frac{1}{2} \frac{\Delta f}{f} - b \cdot \nabla \log f - \frac{1}{2} \Delta \log f \\
    & = \frac{1}{2} \big\| \nabla \log f \big\|^2,
\end{align*}
and so the denoising score matching objective becomes
\begin{equation}
\label{eq:realdiffusionDSM}
    \mathcal{I}_{\textup{DSM}}(\beta) = \frac{1}{2} \int_0^T \E[q_{0,t}(\x_0, \x_t)]{\big\|\nabla \log q_{t|0}(\x_t | \x_0) - \nabla \log \beta(\x_t, t) \big\|^2} \d t.
\end{equation}

Looking at Equations (\ref{eq:reverseparamdiffusion}) and (\ref{eq:realdiffusionDSM}) suggests that it is more natural to parameterise the reverse process in terms of $s_\theta(\x,t) = \nabla \log \beta(\x,t)$ instead of $\beta(\x,t)$. Making this substitution, the objective becomes
\begin{equation*}
    \mathcal{I}_{\textup{DSM}}(\theta) = \frac{1}{2} \int_0^T \E[q_{0,t}(\x_0, \x_t)]{\big\|\nabla \log q_{t|0}(\x_t | \x_0) - s_\theta(\x_t,t) \big\|^2} \d t,
\end{equation*}
recovering the objective of \citeappdx{song2021score}.

Parameterising in terms of $s_\theta(\x, t)$ rather than $\beta(\x,t)$ is preferable for a couple of reasons. First, $s_\theta(\x, t)$ is targeting the score $\nabla \log q_t(\x)$, while $\beta(\x,t)$ is targeting $q_t(\x)$, and we expect the former to typically be an easier target. Second, while Equation (\ref{eq:realdiffusionDSM}) only makes sense when the forward and backward processes are related via Assumption \ref{ass:generatorrelation}, the objective in Equation (\ref{eq:score_matching_objective}) is valid for any forward and backward diffusion processes as in Equation (\ref{eq:diffusionSDEs}). Hence reparameterising allows us to capture a wider class of reverse processes in our optimisation.

\subsection{Discrete state spaces}
\label{app:discretespaces}

Next, we show how to apply our framework when $X$ and $Y$ are continuous-time Markov chains on a finite discrete state space as in Example \ref{ex:CTMCgenerator}. With a particular choice of parameterisation, we end up recovering the set-up of \citeappdx{campbell2022continuous}.

Recall that we start with $\K = \del_t + A$ and $\L = \del_t + B$, where $A$ and $B$ are the time-dependent generator matrices of $X$ and $Y$ respectively. From this it follows immediately that $\hat \K^\ast = A^T$. We will use the counting measure as our reference measure $\nu$.

On a finite discrete space, all functions are bounded and have compact support, and $\D(\hat \K) = \D(\hat \L) = C_0(\cal{S})$ is the set of all functions on $\X$. Assumptions \ref{ass:statespace}, \ref{ass:feller}, \ref{ass:decomposition} and \ref{ass:domainsubspace} follow immediately. In addition, we assume that the reverse process and $p_0$ are sufficiently regular that Assumptions \ref{ass:regularity_of_p}, \ref{ass:Mandc} and \ref{ass:alpha} always hold.

In order for Assumption \ref{ass:existenceM} to hold, we need to find $\M$ and $c$ such that $\M + c = \del_t + \hat \K ^\ast$ (viewed as operators). Since $\M$ should be the generator of another CTMC, we write $\M = \del_t + D$ for some generator matrix $D$. We then require $D + c = A^T$, where $c$ is viewed as a diagonal matrix and $D$ must have zero row sums. This holds if and only if we take
\begin{equation*}
    c_\x = \sum_{\y \in \X} A_{\y\x}, \hspace{5mm} D_{\x\y} = A_{\y\x} - c_\x \mathbbm{1}_{\x = \y}.
\end{equation*}

With this choice of $\M$, Assumption \ref{ass:generatorrelation} becomes
\begin{equation*}
    \beta^{-1}(\x, t) \sum_{\z \in \X} D_{\x\z} f(\z) = \sum_{\z \in \X} B_{\x\z} \; \beta^{-1}(\z, t) f(\z) - f(\x) \sum_{\z \in \X} B_{\x\z} \; \beta^{-1}(\z, t)
\end{equation*}
for all $\x \in \X$. If we pick two distinct $\x,\y$ and set $f(\z) = \mathbbm{1}_{\z = \y}$ in the above, we deduce
\begin{equation*}
    \beta^{-1}(\x, t) D_{\x\y} = \beta^{-1}(\y, t) B_{\x\y} \hspace{5mm} \text{for all } \x \neq \y.
\end{equation*}
Hence for Assumption 2 to hold, we require
\begin{equation}
\label{eq:discreverseparam1}
    A_{\y\x} = \frac{\beta(\x,t)}{\beta(\y,t)} B_{\x\y} \hspace{5mm} \text{for all } \x \neq \y.
\end{equation}
An elementary check also shows that this condition is sufficient for Assumption \ref{ass:generatorrelation} to hold for a given choice of $\beta$.

With this parameterisation, the implicit score matching objective becomes
\begin{align*}
    \I_{\textup{ISM}}(\beta) &= \int_0^T \E[q_t(\x_t)]{\frac{B^T \beta(\x_t)}{\beta(\x_t)} + B \log \beta(\x_t)} \d t \\
    &= \int_0^T \E[q_t(\x_t)]{\sum_{\y \in \X} \lrcb{B_{\y\x_t} \frac{\beta(\y, t)}{\beta(\x_t, t)} + B_{\x_t\y}\log \frac{\beta(\y, t)}{\beta(\x_t, t)}} } \d t.
\end{align*}

Unfortunately, fitting $\beta$ directly using this objective is typically likely to perform poorly. This can be seen for a couple of reasons. Firstly, the optimal value of $\beta(\x,t)$ is $q_t(\x)$, and so learning $\beta(\x,t)$ should be roughly as hard as targeting the marginals of the forward process directly. Secondly, the presence of $\beta$ in the denominators can lead to numerical instabilities in regions where the forward process has low density.

Fortunately, we have at least a couple of methods for avoiding these problems available. The first is to find an equivalent formulation of the objective in terms of the generator of the reverse process, and then learn this generator using a denoising parameterisation. For $\x \neq \y$, we have
\begin{align*}
    B_{\x\y}\log \frac{\beta(\y, t)}{\beta(\x, t)} & = B_{\x\y} \log \frac{B_{\x\y}}{A_{\y\x}} \\
    & = - B_{\x\y} \log A_{\y\x} + \const,
\end{align*}
where the constant depends only on the dynamics of the forward process, which are fixed.
We can therefore write
\begin{align*}
    \mathcal{I}_{\textup{ISM}}(A) & = \int_0^T \E[q_t(\x_t)]{ B_{\x_t\x_t} + \sum_{\y \neq \x_t} A_{\x_t\y} - \sum_{\y \neq \x_t} B_{\x_t\y} \log A_{\y\x_t}} \d t + \const \\
    & = \int_0^T \E[q_t(\x_t)]{- A_{\x_t\x_t} - \sum_{\y \neq \x_t} B_{\x_t\y} \log A_{\y\x_t}} \d t + \const,
\end{align*}
recovering the objective of \citeappdx{campbell2022continuous}. In addition, we can parameterise the reverse generator $A$ via
\begin{equation}
\label{eq:discretedenoisingparam}
    A_{\x\y}(\theta) = B_{\y\x} \sum_{\x_0} \frac{q_{t|0}(\y | \x_0)}{q_{t|0}(\x | \x_0)} p_\theta^{(t)}(\x_0 | \x_t) \hspace{5mm} \text{for } \x \neq \y,
\end{equation}
where $p_\theta^{(t)}(\x_0 | \x_t)$ is some learned estimate of the original datapoint $\x_0$ given the noised observation $\x_t$, and $\theta$ denotes the learnable parameters. This parameterisation should be more stable, as it avoids potentially exploding denominators, and we expect predicting the original datapoint given the noised datapoint to be an easier goal than learning the marginals $q_t(\x)$. See \citecont{campbell2022continuous} for more details on this denoising parameterisation.

The second method is to reparameterise our objective in terms of the ratios $s_\theta(\x, \y, t) = \beta(\y, t) / \beta(\x, t)$. Doing this, the training objective becomes
\begin{equation}
\label{eq:discrete_score_objective}
    \I_{\textup{ISM}}(\theta) = \int_0^T \E[q_t(\x_t)]{\sum_{\y \in \X} \big\{ B_{\y\x_t}s_\theta(\x_t, \y; t) - B_{\x_t\y} \log s_\theta(\y, \x_t; t) \big\}} \d t.
\end{equation}
In addition, the generative process is now parameterised in terms of $s_\theta(\x, \y, t)$ via
\begin{equation}
\label{eq:discreverseparam2}
    A_{\x\y} = B_{\y\x}s_\theta(\x,\y;t) \hspace{5mm} \text{for } \x \neq \y.
\end{equation}
Importantly, this objective matches the generalised objective from Section \ref{sec:general_objective} when the noising and generative processes are related by Assumption \ref{ass:generatorrelation}, and is still minimised when $s_\theta(\x,\y;t) = q_t(\y)/q_t(\x)$.

This parameterisation is potentially beneficial for a couple of reasons. Firstly, by removing $\beta(\x, t)$ from the denominators, we expect that objective should be more numerically stable. Secondly, this parameterisation captures a wider class of potential reverse processes, since $A$ is now given in terms of $B$ via Equation (\ref{eq:discreverseparam2}), which is less restrictive than Equation (\ref{eq:discreverseparam1}). 

As discussed further in Section \ref{sec:relationship_score_matching}, the integrand in Equation (\ref{eq:discrete_score_objective}) may be viewed as a score matching objective for discrete state space. It shares certain similarities with ratio matching techniques \citepappdx{hyvarinen2007some}, in particular targeting the ratios $\beta(\y,t)/\beta(\x,t)$. However, as far as we are aware this particular objective is not directly equivalent to any previously studied score matching objective in discrete state space \citepappdx{hyvarinen2007some, lyu2009interpretation, sohldickstein2011new}.

\subsection{Riemannian manifolds}
\label{app:riemannianmanifolds}

Consider the case where $\X$ is a Riemannian manifold with metric tensor $g$ and $\nu$ is the volume measure induced by $g$ (so that Assumption \ref{ass:statespace} holds). A diffusion in $\X$ may be defined through its generator, so we let the noising and generative processes have generators
\begin{equation*}
    \K = \partial_t + \mu \cdot \nabla + \frac{1}{2} \Delta, \hspace{5mm}
    \L = \partial_t + b \cdot \nabla + \frac{1}{2} \Delta.
\end{equation*}
respectively, where $\Delta$ is the Laplace-Beltrami operator defined in local coordinates by
\begin{equation*}
    \Delta f = \frac{1}{\sqrt{|g|}} \partial_i \big(\sqrt{|g|} g^{ij} \partial_j f \big)
\end{equation*}
and $|g|$ denotes the determinant of the metric tensor. For such processes, Assumption 5 is satisfied under mild regularity conditions on the manifold and the coefficients of the generators, as detailed by \citeappdx{molchanov1968strong}. As in the Euclidean diffusion case, Assumption \ref{ass:decomposition} follows from the given form of $\K$, for Assumption \ref{ass:domainsubspace} we may take $\D_0 = C^\infty_c(\X)$ and note that this is dense in $L^2(\X, \nu)$ \citepappdx[Section 4.4]{taylor2011partial}, and we assume that the reverse process and $p_0$ are sufficiently regular that Assumptions \ref{ass:regularity_of_p}, \ref{ass:Mandc} and \ref{ass:alpha} hold.

To calculate the adjoint operator of $\hat \K$, we recall that the canonical volume element on $\X$ induced by $g$ is given by
\begin{equation*}
    \d \omega = \sqrt{|g|} \; \d x^1 \wedge \cdots \wedge \d x^n
\end{equation*}
and the divergence of a vector field $a : \X \rightarrow T\X$ on a Riemannian manifold is given by
\begin{equation*}
    \nabla \cdot a = \frac{1}{\sqrt{|g|}} \partial_i(a^i \sqrt{|g|}).
\end{equation*}
Then, using the generalised Stokes' Theorem, we have
\begin{align*}
    \langle f, \mu \cdot \nabla h \rangle &= \int_\X f \mu^i \; (\partial_i h) \sqrt{|g|} \; \d x^1 \wedge \cdots \wedge \d x^n \\
    &= - \int_\X h \; \partial_i(\mu^i f \sqrt{|g|}) \; \d x^1 \wedge \cdots \wedge \d x^n \\
    &= \langle - (\nabla \cdot \mu)f - (\mu \cdot \nabla f), h \rangle,
\end{align*}
where we assume $f$ and $h$ are sufficiently smooth that we may disregard boundary terms. In addition, we have
\begin{align*}
    \langle f, \Delta h \rangle &= \int_\X f \partial_i(\sqrt{|g|} g^{ij} \partial_j h) \; \d x^1 \wedge \cdots \wedge \d x^n \\
    &= - \int_\X \sqrt{|g|} g^{ij} (\partial_i f) (\partial_j h) \; \d x^1 \wedge \cdots \wedge \d x^n \\
    &= \langle \Delta f, h \rangle.
\end{align*}
We conclude that the adjoint operator is given by
\begin{equation*}
    \hat \K ^\ast = - \mu \cdot \nabla - (\nabla \cdot \mu) + \frac{1}{2} \Delta.
\end{equation*}
Then, as in the Euclidean diffusion case we see that Assumption \ref{ass:existenceM} holds if we let $c = -(\nabla \cdot \mu)$ and
\begin{equation*}
    \M = \partial_t - \mu \cdot \nabla + \frac{1}{2}\Delta,
\end{equation*}
noting that $\M$ is also the generator of a diffusion process $Z$ on $\X$. We also find that Assumption \ref{ass:generatorrelation} reduces to the condition $\nabla \log \beta = \mu + b$, as before.

Assuming this holds, we can evaluate
\begin{align*}
    \Phi(f) & = \frac{\L f}{f} - \L \log f \\
    & = \frac{b \cdot \nabla f}{f} + \frac{1}{2} \frac{\Delta f}{f} - b \cdot \nabla \log f - \frac{1}{2} \Delta \log f \\
    & = \frac{1}{2} \big\| \nabla \log f \|_{g(x)}^2,
\end{align*}
where $\| \cdot \|_{g(x)}$ denotes the norm on the tangent space $T_x\X$ induced by $g$.

Finally, as in the Euclidean diffusion we make a reparameterisation $s_\theta(\x, t) = \nabla \log \beta(\x,t)$ in order to sidestep Assumption \ref{ass:generatorrelation} and provide an easier training target. The resulting denoising score matching objective is
\begin{equation*}
    \mathcal{I}_{\textup{DSM}}(\theta) = \frac{1}{2} \int_0^T \E[q_{0,t}(\x_0, \x_t)]{\big\|\nabla \log q_{t|0}(\x_t | \x_0) - s_\theta(\x_t,t) \big\|_{g(\x_t)}^2} \d t,
\end{equation*}
which reproduces the result of \citeappdx{debortoli2022riemannian} and \citeappdx{huang2022riemannian}. Notably, we find that all the relevant formulae in the manifold case are essentially the same as in the Euclidean diffusion case, except for the inclusion of the metric tensor.

\subsection{Wright--Fisher diffusions}
\label{app:wrightfisher}

Suppose we wish to approximate a distribution $p_{\data}(\cdot)$ over the space $\X = \cal{P}(E)$ of measures on a finite set $E = \{1, \dots, N\}$. A natural class of stochastic processes on $\X$ are the Wright--Fisher diffusions, a model used in population genetics to describe the evolution of allele frequencies in a population over time \citepcont{ethier1993transition}.

We can parameterise measures in $\X$ by tuples of real numbers $\p = (p_1, \dots, p_N) \in [0, 1]^N$ such that $\sum_{i=1}^N p_i = 1$. With this parameterisation, the Wright--Fisher diffusion has generator
\begin{equation*}
    \L = \del_t + \frac{1}{2} \sum_{i,j = 1}^N p_i \lr{\delta_{ij} - p_j} \frac{\del^2}{\del p_i \del p_j} + \sum_{i,j=1}^N q_{ij}p_i \frac{\del}{\del p_j}, \label{eq:WF_generator}
\end{equation*}
with domain $\D(\L) = \{F|_{\cal{P}(E)} : F(p_1, \dots, p_N) \in C^2(\R^N) \}$, where $(q_{ij})_{i,j = 1,\dots, N}$ is some matrix, potentially depending on $\p$ and $t$, such that $\sum_{j=1}^N q_{ij} = 0$ for each $i = 1,\dots,N$.

If we take $q_{ij} = \frac{1}{2}\vartheta_j > 0$ for all $\p \in \X, t \in [0,T]$ and $i \neq j$, then this process is ergodic and its invariant distribution is $\textup{Dirichlet}(\Theta)$, the Dirichlet distribution with parameters $\Theta = (\vartheta_1, \dots, \vartheta_N)$ \citepcont{ethier1993transition}. Moreover, the transition function of the process can be expressed as
\begin{equation}
\label{eq:WFeigenfunction}
    P(t, \p, \cdot) = \sum_{n=0}^\infty d^\Theta_n(t) \sum_{\alpha \in (\Z_+^N): |\alpha| = n} \binom{n
    }{\alpha} \prod_{i=1}^N p_i^{\alpha_i} \textup{Dirichlet}(\alpha + \Theta)(\cdot)
\end{equation}
where $d_n^\Theta(t)$ are smooth functions of $t$ given explicitly in \cite{ethier1993transition}. It follows that if we take $\vartheta_j > 2$ for all $j$ and we start the process in the interior of the simplex, then the process almost surely does not hit the boundary and the marginals of the forward process always vanish and have zero derivative at the boundary (since this holds for any Dirichlet distribution where all parameters are greater than 2).

Note that $\X$ is compact and hence locally compact and separable. Since we can view $\X$ as a subset of a linear subspace of $\R^N$, it also has a natural Lebesgue measure, which we take as the reference measure $\nu$. Hence we satisfy Assumption \ref{ass:statespace}.

We let our noising process have generator $\L$ as above and our generative process have generator
\begin{equation*}
    \K = \del_t + \frac{1}{2} \sum_{i,j = 1}^N p_i \lr{\delta_{ij} - p_j} \frac{\del^2}{\del p_i \del p_j} + \sum_{i,j=1}^N r_{ij}p_i \frac{\del}{\del p_j},
\end{equation*}
where $(r_{ij})_{i,j = 1, \dots, N}$ is another matrix with zero row sums. The forward process $Y$ is then Feller from \citetappdx[Theorem 3.4]{ethier1993fleming}. It follows that the extended forward process $\extY$ is also Feller and, since the process is pathwise continuous on a compact state space, this implies that the extended backward process $\extX$ is also Feller, so Assumption \ref{ass:feller} holds. Assumption \ref{ass:decomposition} follows from the given form of $\K$, and for Assumption \ref{ass:domainsubspace}, we can take $\D_0 = \{F|_{\cal{P}(E)} : F(p_1, \dots, p_N) \in C^\infty(\R^N) \}$. As usual, we assume that the reverse process and $p_0$ are sufficiently regular that Assumptions \ref{ass:regularity_of_p}, \ref{ass:Mandc} and \ref{ass:alpha} hold.

In order to calculate the adjoint operator $\hat \K ^\ast$, we require the following lemma, which is essentially a form of the integration by parts formula for the space $\X$.

\begin{lemma}
\label{lem:IbPsimplex}
Suppose we have $F : \R^N \rightarrow \R^N$ such that for all $x \in \X$, $F(x) \cdot \mb{1} = 0$, where $\mb{1}$ is the unit vector in the $(1,\dots,1)^T$ direction. In addition, suppose that $F(x) = 0$ for all $x \in \del \X$. Then
\begin{equation*}
    \int_\X \sum_{j=1}^N \frac{\del F_j}{\del p_j}(\p) \;\d \nu(\p) - \frac{1}{N} \int_\X \sum_{j,k = 1}^N \frac{\del F_k}{\del p_j}(\p) \;\d \nu (\p) = 0.
\end{equation*}
\end{lemma}

\begin{proof}
Since $F(x) \cdot \mb{1} = 0$ for all $x \in \X$, we can view $F$ as a function from $\X$ to $T\X$, the tangent bundle of $\X$. Then, since $F(x) = 0$ for $x \in \del \X$, by the generalised Stokes' theorem we have
\begin{equation*}
    \int_\X \nabla_\X \cdot F \;\d \nu = 0,
\end{equation*}
where $\nabla_\X \cdot F$ denotes the manifold divergence on $\X$. Finally, $\nabla_{\X} \cdot F = \nabla \cdot F - \mb{1} \cdot \nabla \lr{F \cdot \mb{1}}$, where $\nabla$ is the standard gradient operator on $\R^N$, and so the result follows.
\end{proof}

First, we need to calculate the adjoint of $\hat \K$. To deal with the first order term, we use Lemma \ref{lem:IbPsimplex} with $F_j = r_{ij}p_i f h$ for $i = 1, \dots, N$ in turn to get
\begin{align*}
    \int_\X \sum_{j=1}^N \frac{\del}{\del p_j}(r_{ij}p_i f h) \;\d \nu(\p) - \frac{1}{N} \int_\X \sum_{j,k = 1}^N \frac{\del}{\del p_j}(r_{ik}p_i f h) \;\d \nu (\p) = 0,
\end{align*}
whenever $fh = 0$ on $\del \X$. Since $\sum_{j=1}^N r_{ij} = 0$, the second term vanishes. Thus, summing over $i$ we get
\begin{multline*}
     \int_\X \sum_{i,j=1}^N p_i f h \frac{\del r_{ij}}{\del p_j} \; \d\nu(\p) + \int_\X \sum_{i=1}^N r_{ii} f h \;\d \nu(\p) \\ + \int_{\X} \sum_{i, j=1}^N r_{ij}p_i h \frac{\del f}{\del p_j} \;\d\nu(\p) + \int_{\X} \sum_{i, j=1}^N r_{ij}p_i f \frac{\del h}{\del p_j} \;\d\nu(\p) = 0,
\end{multline*}
from which we deduce that
\begin{equation}
\label{eq:WFdualpart1}
    \int_{\X} f \lr{\sum_{i,j = 1}^N r_{ij} p_i \frac{\del}{\del p_j}} h \; \d \nu(\p) = - \int_{\X} h \lr{\sum_{i,j=1}^N p_i \frac{\del r_{ij}}{\del p_j} + \sum_{i=1}^N r_{ii} + \sum_{i,j = 1}^N r_{ij} p_i \frac{\del}{\del p_j}} f \; \d \nu(\p)
\end{equation}
whenever $fh = 0$ on $\del \X$.

To deal with the second order term, we use Lemma \ref{lem:IbPsimplex} with $F_j = p_i(\delta_{ij} - p_j)(f \del_i h)$ for each $i = 1, \dots, N$ in turn to get
\begin{equation*}
    \int_\X \sum_{j=1}^N \frac{\del}{\del p_j}(p_i(\delta_{ij} - p_j)(f \del_i h)) \;\d \nu(\p) - \frac{1}{N} \int_\X \sum_{k,j = 1}^N \frac{\del}{\del p_j}(p_i(\delta_{ik} - p_k)(f \del_i h)) \;\d \nu (\p) = 0,
\end{equation*}
whenever $f \del_i h = 0$ on $\del \X$ for each $i = 1, \dots, N$. Expanding the LHS, we get
\begin{align*}
    & \hspace{5mm} \int_\X \sum_{j=1}^N (\delta_{ij} - p_i - \delta_{ij}p_i)(f \del_i h) \d\nu(\p) + \int_\X \sum_{j=1}^N p_i(\delta_{ij} - p_j)((\del_j f)(\del_i h) + f \del_i \del_j h) \d\nu(\p) \\
    & - \frac{1}{N} \int_{\X} \sum_{j,k=1}^N(\delta_{ij}\delta_{ik} - \delta_{ij} p_k - \delta_{jk}p_i) (f \del_i h) \d\nu(\p) - \frac{1}{N} \int_\X \sum_{j,k=1}^N p_i(\delta_{ik} - p_k) \del_j(f \del_i h) \d\nu(\p).
\end{align*}
Now, the last term is zero since $\sum_{k=1}^N p_i(\delta_{ik} - p_k) = 0$. Simplifying and summing over $i$, we get
\begin{multline*}
    \int_\X \sum_{i=1}^N (1-Np_i) f \del_i h \;\d \nu(\p) + \int_\X f \lr{\sum_{i,j=1}^N p_i(\delta_{ij} - p_j) \frac{\del^2}{\del p_i \del p_j}} h \;\d\nu(\p) \\ +  \int_\X \sum_{i,j=1}^N p_i(\delta_{ij} - p_j)(\del_j f)(\del_i h) \;\d\nu(\p) = 0.
\end{multline*}
By symmetry, we may reverse the roles of $f$ and $h$ in this last equation and subtract the resulting equations to get
\begin{align}
    & \hspace{5mm} \int_\X \sum_{i=1}^N (1-Np_i) f \del_i h \;\d \nu(\p) + \int_\X f \lr{\sum_{i,j=1}^N p_i(\delta_{ij} - p_j) \frac{\del^2}{\del p_i \del p_j}} h \;\d\nu(\p) \nonumber \\
    & = \int_\X \sum_{i=1}^N (1-Np_i) h \del_i f \;\d \nu(\p) + \int_\X h \lr{\sum_{i,j=1}^N p_i(\delta_{ij} - p_j) \frac{\del^2}{\del p_i \del p_j}} f \;\d\nu(\p) \label{eq:WFsecondorder1}
\end{align}
whenever $f \nabla h = h \nabla f = 0$ on $\del \X$. Finally, applying Lemma \ref{lem:IbPsimplex} with $F_i = fh(1 - N p_i)$, we get
\begin{equation*}
    \int_\X \sum_{j=1}^N \frac{\del}{\del p_j}(fh(1 - N p_j)) \;\d \nu(\p) - \frac{1}{N} \int_\X \sum_{i,j = 1}^N \frac{\del}{\del p_j}(fh(1 - Np_i)) \;\d \nu (\p) = 0.
\end{equation*}
whenever $fh = 0$ on $\del \X$. Expanding, we have
\begin{multline*}
    \int_\X \sum_{j=1}^N h(1 - N p_j) \del_j f \;\d\nu(\p) + \int_\X \sum_{j=1}^N f(1-Np_j) \del_j h \;\d\nu(\p) - N^2 \int_\X f h \d \nu(\p) \\ - \frac{1}{N} \int_\X \sum_{i,j = 1}^N (1 - Np_i) \del_j(fh) \;\d \nu (\p) + N \int_\X f h \d \nu(\p) = 0,
\end{multline*}
which simplifies to
\begin{equation}
\label{eq:WFsecondorder2}
     \int_\X \sum_{j=1}^N h(1 - N p_j) \del_j f \;\d\nu(\p) + \int_\X \sum_{j=1}^N f(1-Np_j) \del_j h \;\d\nu(\p) - N(N-1) \int_\X fh \d \nu(\p) = 0.
\end{equation}
Combining Equations (\ref{eq:WFsecondorder1}) and (\ref{eq:WFsecondorder2}), we see
\begin{align}
    \frac{1}{2} \int_\X f \lr{\sum_{i,j=1}^N p_i(\delta_{ij} - p_j) \frac{\del^2}{\del p_i \del p_j}} h \;\d\nu(\p) & = \frac{1}{2} \int_\X h \lr{\sum_{i,j=1}^N p_i(\delta_{ij} - p_j) \frac{\del^2}{\del p_i \del p_j}} f \;\d\nu(\p) \nonumber \\
    & \hspace{5mm} + \int_\X h \lr{\sum_{j=1}^N (1 - Np_j) \frac{\del}{\del p_j}} f \;\d\nu(\p) \nonumber \\
    & \hspace{5mm} - \frac{N(N-1)}{2} \int_\X fh \d \nu(\p). \label{eq:WFdualpart2}
\end{align}

Putting together Equations (\ref{eq:WFdualpart1}) and (\ref{eq:WFdualpart2}), we conclude that the operator
\begin{multline*}
    \K_0 = \frac{1}{2} \sum_{i,j = 1}^N p_i \lr{\delta_{ij} - p_j} \frac{\del^2}{\del p_i \del p_j} + \sum_{j=1}^N (1 - Np_j) \frac{\del}{\del p_j} - \frac{N(N-1)}{2} \\ - \sum_{i,j=1}^N p_i \frac{\del r_{ij}}{\del p_j} - \sum_{i=1}^N r_{ii} - \sum_{i,j = 1}^N r_{ij} p_i \frac{\del}{\del p_j}
\end{multline*}
satisfies $\langle \K_0 f, h \rangle = \langle f, \K h \rangle$ for all functions $f, h \in \{F|_{\cal{P}(E)} : F(p_1, \dots, p_N) \in C^2(\R^N) \}$ such that $f h = f \nabla h = h \nabla f = 0$ on $\del \X$. We conclude that $\hat \K^\ast = \K_0$ and $h \in \D(\hat \K^\ast)$ for all $h$ such that $h = \nabla h = 0$ on $\del \X$. Therefore, we choose to define
\begin{equation*}
    \M = \del_t + \frac{1}{2} \sum_{i,j = 1}^N p_i \lr{\delta_{ij} - p_j} \frac{\del^2}{\del p_i \del p_j} + \sum_{j=1}^N (1 - Np_j) \frac{\del}{\del p_j} - \sum_{i,j = 1}^N r_{ij} p_i \frac{\del}{\del p_j},
\end{equation*}
\begin{equation*}
    c = - \frac{N(N-1)}{2} - \sum_{i,j=1}^N p_i \frac{\del r_{ij}}{\del p_j} - \sum_{i=1}^N r_{ii}.
\end{equation*}
We see that Assumption \ref{ass:existenceM} is satisfied, since $v$ vanishes and has zero derivative on $\del \X$ by our earlier remarks. Recalling that $q_{ij} = \frac{1}{2} \vartheta_j$ for $i \neq j$ and $\sum_{j=1}^N r_{ij} = 0$, if we let
\begin{equation*}
    u_{ij} = \vartheta_j + \frac{p_j}{p_i}(\vartheta_i - 1) - r_{ij}, \quad \text{for } i \neq j
\end{equation*}
and set $u_{ii} = - \sum_{j \neq i} u_{ij}$, then for each $j = 1, \dots, N$ we have
\begin{align*}
    \sum_{i=1}^N u_{ij}p_i & = \sum_{i \neq j}u_{ij}p_i - \sum_{i \neq j} u_{ji}p_j \\
    & = \sum_{i \neq j} \lr{p_i \vartheta_j + p_j (\vartheta_i - 1)} - \sum_{i \neq j} \lr{p_j\vartheta_i + p_i(\vartheta_j - 1)} - \sum_{i=1}^N r_{ij}p_i \\
    & = 1 - Np_j - \sum_{i=1}^N r_{ij}p_i.
\end{align*}
We thus see that $\M$ is the generator of another Wright--Fisher process with transition matrix $(u_{ij})_{i,j=1,\dots,N}$. Hence Assumption \ref{ass:existenceM} is satisfied. To check Assumption \ref{ass:generatorrelation},
\begin{align*}
    \beta \L(\beta^{-1} f) - \beta f \L(\beta^{-1}) & = \frac{\beta}{2} \sum_{i,j=1}^N p_i(\delta_{ij} - p_j)(\beta^{-1}(\del_i \del_j f) + 2 (\del_i f)(\del_j \beta^{-1}) + f(\del_i \del_j \beta^{-1})) \\
    & \hspace{5mm} + \beta \sum_{i,j=1}^N q_{ij}p_i (\beta^{-1} \del_j f + f \del_j \beta^{-1}) \\
    & \hspace{5mm} - \frac{\beta f}{2} \sum_{i,j = 1}^N p_i \lr{\delta_{ij} - p_j} \del_i \del_j \beta^{-1} - \beta f \sum_{i,j=1}^N q_{ij}p_i \del_j \beta^{-1} \\
    & = \frac{1}{2} \sum_{i,j=1}^N p_i(\delta_{ij} - p_j)(\del_i \del_j f) - \sum_{i,j=1}^N p_i(\delta_{ij} - p_j)(\del_i f)(\del_j \log \beta) \\
    & \hspace{5mm} + \sum_{i,j=1}^N q_{ij}p_i (\del_j f).
\end{align*}
Thus Assumption \ref{ass:generatorrelation} holds if and only if
\begin{equation*}
    \sum_{i=1}^N u_{ij}p_i = - \sum_{i=1}^N p_i(\delta_{ij} - p_j) (\del_i \log \beta) + \sum_{i=1}^N q_{ij} p_i
\end{equation*}
for each $j = 1, \dots, N$. This is satisfied if we take
\begin{equation}
\label{eq:reversetransitionmatrix}
    r_{ij} = \vartheta_j + \frac{p_j}{p_i}(\vartheta_i - 1) - p_j \frac{\del (\log \beta)}{\del p_i} - q_{ij}
\end{equation}
for $i \neq j$ and $r_{ii} = - \sum_{j \neq i} r_{ij}$. We choose this parameterisation since if we start the forward process in its invariant distribution and learn $\beta$ so that the generative process is the exact time reversal of the forward process then $\beta(\p, t) \propto q_t(\x) \propto \prod_{i=1}^N p_i^{\vartheta_i - 1}$. In this case, Equation (\ref{eq:reversetransitionmatrix}) reduces to $r_{ij} = \vartheta_j - q_{ij}$, so this parameterisation ensures that if we start the forward process in its invariant distribution and learn the reverse process perfectly then the transition matrix $(r_{ij})_{i,j=1,\dots,n}$ we learn is equal to the transition matrix $(q_{ij})_{i,j=1,\dots,n}$ of the forward process.

We can then calculate the score matching operator $\Phi(f) = f^{-1} \L f - \L \log f$,
\begin{align*}
    \Phi(f) & = \frac{1}{2} \sum_{i,j = 1}^N p_i \lr{\delta_{ij} - p_j} f^{-1} \frac{\del^2 f}{\del p_i \del p_j} - \frac{1}{2} \sum_{i,j = 1}^N p_i \lr{\delta_{ij} - p_j} \frac{\del^2 (\log f)}{\del p_i \del p_j} \\
    & \hspace{15mm} + \sum_{i,j=1}^N q_{ij}p_i f^{-1} \frac{\del f}{\del p_j} - \sum_{i,j=1}^N q_{ij}p_i \frac{\del (\log f)}{\del p_j} \\
    & = \frac{1}{2} \sum_{i,j = 1}^N p_i \lr{\delta_{ij} - p_j} \frac{\del (\log f)}{\del p_i} \frac{\del (\log f)}{\del p_j}.
\end{align*}
However, since we do not have access to the analytic forms of the transition kernel $q_{t|0}(\p_t | \p_0)$ for this model, we must fit $\beta$ using the implicit score matching objective. We thus calculate
\begin{align*}
    \frac{\hat \L^\ast \beta}{\beta} + \hat \L \log \beta & = \frac{1}{2} \sum_{i,j = 1}^N p_i \lr{\delta_{ij} - p_j} \beta^{-1} \frac{\del^2 \beta}{\del p_i \del p_j} + \sum_{j=1}^N (1 - Np_j) \beta^{-1} \frac{\del \beta}{\del p_j} - \frac{N(N-1)}{2} \\
    & \hspace{5mm} - \sum_{i,j=1}^N p_i \frac{\del q_{ij}}{\del p_j} - \sum_{i=1}^N q_{ii} - \sum_{i,j = 1}^N q_{ij} p_i \beta^{-1} \frac{\del \beta}{\del p_j} \\
    & \hspace{5mm} + \frac{1}{2} \sum_{i,j = 1}^N p_i \lr{\delta_{ij} - p_j} \frac{\del^2 (\log \beta)}{\del p_i \del p_j} + \sum_{i,j=1}^N q_{ij}p_i \frac{\del (\log \beta)}{\del p_j} \\
    & = \sum_{i,j = 1}^N p_i \lr{\delta_{ij} - p_j} \frac{\del^2 (\log \beta)}{\del p_i \del p_j} + \frac{1}{2} \sum_{i,j = 1}^N p_i \lr{\delta_{ij} - p_j} \frac{\del (\log \beta)}{\del p_i} \frac{(\del \log \beta)}{\del p_j} \\
    & \hspace{5mm} + \sum_{j=1}^N (1 - Np_j) \frac{\del (\log \beta)}{\del p_j} + \const,
\end{align*}
where we have discarded terms that do not depend on $\beta$. Noting that the loss and the reverse process only depend on $\beta$ through $\del (\log \beta)/ \del p_j$, we reparameterise in terms of $s_\theta^i(\p, t) = p_i \del (\log \beta(\p, t)) / \del p_i$. (We include the extra factor of $p_i$ for numerical stability reasons, since if we start in the stationary distribution then $p_i \del (\log \beta(\p, t)) / \del p_i$ should be of constant scale.) Doing this, the implicit score matching objective becomes
\begin{multline}
    \label{eq:WF_ism}
    \I_{\textup{ISM}}(\theta) = \int_0^T \bb{E}_{q_t(\p_t)}\bigg[\sum_{i,j = 1}^N \lr{\delta_{ij} - p_j} \frac{\del s_\theta^i(\p_t, t)}{\del p_j} + \frac{1}{2} \sum_{i,j = 1}^N (p_j^{-1} \delta_{ij} - 1) s^i_\theta(\p_t, t) s_\theta^j(\p_t, t) \\ + (1 - N) \sum_{j=1}^N s^j_\theta(\p_t, t) \bigg] \d t, 
\end{multline}
and the reverse process is parameterised as the Wright--Fisher diffusion with transition matrix $(r_{ij})_{i,j=1, \dots, N}$ where
\begin{equation*}
    r_{ij}(\p, T-t) = \frac{1}{2} \vartheta_j + \frac{p_j}{p_i}(\theta_i - 1) - \frac{p_j}{p_i} s^i_\theta(\p, t), \quad i \neq j.
\end{equation*}

\section{Proof of properties of the score matching operator}
\label{app:operatorproperties}

We give the proof of the properties of the score matching operator from Proposition \ref{lem:operator_properties}.

\operatorlem*

\begin{proof} Since $\log$ is a concave function, it follows that $\log(Q_t f) \geq Q_t (\log f)$ for all $f$ in the domain of $\Phi$ with equality if $f$ is constant. Hence
\begin{equation*}
    \frac{\log (Q_t f) - \log f}{t} \geq \frac{Q_t (\log f) - \log f}{t}
\end{equation*}
for all $t \geq 0$. Taking the limit $t \downarrow 0$, we deduce that $(\L f)/f \geq \L (\log f)$ which gives the first part of the lemma.

For the second part, we assume that $\pi_1 Q_t$ and $\pi_2 Q_t$ are absolutely continuous with respect to $\nu$ and let $\pi_{1,t}(\x)$ and $\pi_{2,t}(\x)$ respectively denote their densities. Then
\begin{align*}
    \frac{\d}{\d t} \KL{\pi_1 Q_t}{\pi_2 Q_t} & = \frac{\d}{\d t} \int \pi_{1,t}(\x) \log \lr{\frac{\pi_{1,t}(\x)}{\pi_{2,t}(\x)}} \d \nu(\x) \\
    & = \int \del_t \pi_{1,t}(\x) \log \lr{\frac{\pi_{1,t}(\x)}{\pi_{2,t}(\x)}} \d \nu(\x) + \int \pi_{1,t}(\x) \del_t \log \pi_{1,t}(\x) \d \nu(\x) \\
    & \hspace{10mm} - \int \pi_{1,t}(\x) \del_t \log \pi_{2,t}(\x) \d \nu(\x) \\
    & = \int \hat \L^\ast \pi_{1,t}(\x) \log \lr{\frac{\pi_{1,t}(\x)}{\pi_{2,t}(\x)}} \d \nu(\x) + \int \hat \L^\ast \pi_{1,t}(\x) \d \nu(\x) \\
    & \hspace{10mm} - \int \hat \L^\ast \pi_{2,t}(\x) \lr{\frac{\pi_{1,t}(\x)}{\pi_{2,t}(\x)}} \d \nu(\x)
\end{align*}
using the Fokker--Planck equation on each term. Since $\langle \hat \L^\ast \pi_{1,t}, 1\rangle = \langle \pi_{1,t}, \hat \L 1\rangle = 0$, we may drop the second term and write
\begin{align*}
    \frac{\d}{\d t} \KL{\pi_1 Q_t}{\pi_2 Q_t} & = \E[\pi_{1,t}(\x)]{\hat \L \log \lr{\frac{\pi_{1,t}(\x)}{\pi_{2,t}(\x)}}} - \E[\pi_{1,t}(\x)]{\lr{\frac{\pi_{2,t}(\x)}{\pi_{1,t}(\x)}} \hat \L \lr{\frac{\pi_{1,t}(\x)}{\pi_{2,t}(\x)}}} \\
    & = - \E[\pi_{1,t}(\x)]{\Phi\lr{\frac{\pi_{1,t}(\x)}{\pi_{2,t}(\x)}}},
\end{align*}
which is the desired result.
\end{proof}

\section{Discrete-time approximation proofs}
\label{app:discretetimeproofs}

In this section, we give the proofs of Lemma \ref{lem:approx} and Theorem \ref{thm:approx} from Section \ref{sec:discreteapprox}. In order to prove Lemma \ref{lem:approx}, we use a couple of lemmas which we present first.

\begin{lemma}
\label{lem:N}
Given processes $\extX$ and $\extY$ as in Section \ref{sec:general_objective}, define a process $\extW$ by setting $\extW_t = (X_{T-t}, t)$ and denote its generator by $\N$. Then we have
\begin{equation*}
    v \N g = (\M + c)(vg)
\end{equation*}
for all sufficiently rapidly decaying functions $g$.
\end{lemma}

\begin{proof} First, we let $\overleftarrow{\K} = - \del_t + \hat \N$ denote the generator of the time-reversal of $\extX$. Then, the integration by parts formula of \citeappdx{cattiaux2021time} implies that for all sufficiently rapidly decaying test functions $f$ and $g$ we have
\begin{equation*}
    \langle p_t f, \K g + \overleftarrow{\K}g \rangle + \langle p_t, \Gamma(f,g)\rangle = 0,
\end{equation*}
where $\Gamma(f,g) = \K(fg) - f \K g - g \K f$ denotes the carr\'e du champ operator associated to $\K$. We deduce that
\begin{align*}
    \langle f, p_t \overleftarrow{\K} g \rangle & = - \langle p_t f, \K g\rangle - \langle p_t, \K(fg) - f \K g - g \K f \rangle \\
    &= -\langle p_t f, \del_t g \rangle - \langle \hat\K^\ast p_t, fg \rangle + \langle \hat\K^\ast (g p_t), f \rangle \\
    &= -\langle p_t \del_t g , f \rangle - \langle g \del_t p_t, f \rangle + \langle \hat\K^\ast (g p_t), f\rangle \\
    &= \langle \hat\K^\ast (g p_t) - \del_t(p_t g), f \rangle
\end{align*}
where in the third line we have used the Fokker--Planck equation. Since $f$ was arbitrary, it follows that
\begin{equation*}
    p_t \overleftarrow{\K} g = \hat\K^\ast (g p_t) - \del_t(p_t g).
\end{equation*}
Finally if we substitute $t \mapsto T-t$ in this final equation, we get
\begin{equation*}
    v \N g = (\hat\K^\ast + \partial_t)(vg), 
\end{equation*}
which gives the desired result when combined with the definition of $\M$ and $c$ from Assumption \ref{ass:existenceM}.
\end{proof}

\begin{lemma}
\label{lem:zeta}
Suppose $\beta : \cal{S} \rightarrow (0,\infty)$ is a function such that Assumption \ref{ass:generatorrelation} holds. If we define $\zeta = \beta^{-1} v$, then for any function $f$ decaying sufficiently rapidly, $\zeta$ satisfies
\begin{equation*}
    \zeta \N f = \L(f \zeta) - f \L \zeta.
\end{equation*}
\end{lemma}

\begin{proof}
For any sufficiently rapidly decaying $f$ satisfying $f \zeta \in \D(\L)$ and $vf \in \D(\M)$, using Lemma \ref{lem:N} we have
\begin{align*}
    \L(f \zeta) - f \L \zeta &= \L(v \beta^{-1} f) - f \L(\beta^{-1} v) \\
    &= \beta^{-1} \M(vf) - \beta^{-1} f \M v \\
    &= \beta^{-1} v \N f - c \beta^{-1} vf + c \beta^{-1} vf \\
    &= \zeta \N f.
\end{align*}
\end{proof}

Now we can give the proofs of Lemma \ref{lem:approx} and Theorem \ref{thm:approx}.

\approxlem*

\begin{proof}
Let $\P^{\x_s}$ and $\Q^{\x_s}$ denote the path measures of $\extW$ and $\extY$ respectively on the interval $[s,t]$ when we condition on the initial value $\x_s$. Assuming $\beta$ is sufficiently regular so that $\zeta$ is bounded away from zero and infinity and $\zeta^{-1} \L \zeta$ is bounded and continuous in the time variable, by Girsanov's theorem and Lemma \ref{lem:zeta} we have
\begin{equation*}
    \frac{\d\P^{\x_s}}{\d\Q^{\x_s}}(\omega) = \frac{\zeta(\omega_t, t)}{\zeta(\omega_s,s)}\exp\left\{- \int_s^t \frac{\L \zeta(\omega_\tau, \tau)}{\zeta(\omega_\tau, \tau)} \d\tau \right\}.
\end{equation*}
Taking logarithms and writing $\gamma = t - s$, to first order in $\gamma$ for any fixed path $\omega$ this becomes
\begin{equation*}
    \log \frac{\d\P^{\x_s}}{\d\Q^{\x_s}}(\omega) = \log \frac{\zeta(\omega_t, t)}{\zeta(\omega_s,s)} - \gamma \frac{\L \zeta(\omega_s, s)}{\zeta(\omega_s, s)} + o(\gamma).
\end{equation*}
Since the first order terms depend only on the value of the path at its endpoints $(\omega_s, \omega_t)$, we conclude that
\begin{equation*}
    \log \frac{q_{t|s}(\x_t | \x_s)}{p_{T-t|T-s}(\x_t | \x_s)} = - \log \frac{\zeta(\x_t, t)}{\zeta(\x_s, s)} + \gamma \frac{\L \zeta(\x_s, s)}{\zeta(\x_s, s)} + o(\gamma).
\end{equation*}
It follows that
\begin{equation*}
    \log \frac{q_{t|s}(\x_t | \x_s)}{p_{T-s|T-t}(\x_s | \x_t)} = \log \frac{v(\x_t, t)}{v(\x_s, s)} - \log \frac{\zeta(\x_t, t)}{\zeta(\x_s, s)} + \gamma \frac{\L \zeta(\x_s, s)}{\zeta(\x_s, s)} + o(\gamma).
\end{equation*}
Taking expectations and using the definition of the generator as a stochastic derivative, we have
\begin{align*}
    \E[\Q]{\log \frac{q_{t|s}(\x_t | \x_s)}{p_{T-s | T-t}(\x_s | \x_t)}} &= \gamma \E[q_s(\x_s)]{\frac{\L\zeta(\x_s,s)}{\zeta(\x_s,s)} - \L \log \zeta(\x_s,s) + \L \log v(\x_s,s)} + o(\gamma) \\
    &= \gamma \E[q_s(\x_s)]{\frac{\hat\L^\ast\beta(\x_s,s)}{\beta(\x_s,s)} + \hat\L \log \beta(\x_s,s)} + o(\gamma),
\end{align*}
where in the final line we have used Lemma \ref{lem:gen_lem_1}.
\end{proof}

\approxthm*

\begin{proof}
Given time steps $0 = t_0 < t_1 < \dots < t_N = T$, define $\gamma_k = t_{k+1} - t_k$ for $k = 0, \dots, N-1$ and set $\overline{\gamma} = \max_{k=0, \dots, N-1} \gamma_k$. Then the natural discretisation of the objective $\I_{\textup{ISM}}(\beta)$ is given by
\begin{equation*}
    \sum_{k = 0}^{N-1} \gamma_k \E[q_{t_k}(\x_{t_k})]{\frac{\hat \L^\ast \beta(\x_{t_k})}{\beta(\x_{t_k})} + \hat \L \log \beta(\x_{t_k})}.
\end{equation*}
Using Lemma \ref{lem:approx} with $s = t_{k-1}$ and $t = t_k$ for $k = 1, \dots, N$, we get
\begin{align*}
    & \hspace{5mm} \E [\tilde q(\x_{t_{k-1}})]{\KL{\tilde q(\x_{t_{k-1}} | \x_{t_k})}{\tilde p_\theta(\x_{t_{k-1}} | \x_{t_k})}} \\
    & = \E[\tilde q(\x_{t_{k-1}}, \x_{t_k})]{\log \frac{\tilde q(\x_{t_k} | \x_{t_{k-1}})}{\tilde p_\theta(\x_{t_{k-1}} | \x_{t_k})}} + \const \\
    & = \gamma_{k-1} \; \bb{E}_{q_{t_{k-1}}(\x_{t_{k-1})}}\bigg[\frac{\hat\L^\ast \beta(\x_{t_k})}{\beta(\x_{t_k})} + \hat\L \log \beta(\x_{t_k}) \bigg] + o(\gamma_{k-1}) + \const.
\end{align*}
Putting this together, we see that
\begin{equation*}
    \KL{\tilde{q}(\x_{0:T})}{\tilde{p}_\theta(\x_{0:T})} = \sum_{k = 0}^{N-1} \gamma_k \E[q_{t_k}(\x_{t_k})]{\frac{\hat \L^\ast \beta(\x_{t_k})}{\beta(\x_{t_k})} + \hat \L \log \beta(\x_{t_k})} + o(\overline{\gamma}) + \const,
\end{equation*}
so objective (\ref{eq:discrete_time_ELBO}) is equivalent to the natural discretisation of $\I_{\textup{ISM}}$ to first order in $\overline{\gamma}$.
\end{proof}

\section{General equivalence between denoising autoencoders and score matching}
\label{app:autoencoder}

A denoising autoencoder takes a datapoint $\x_0$ drawn from a data distribution $q_0$, noises it according to some density $q_\tau(\x_\tau | \x_0)$ and then tries to reconstruct $\x_0$ given the noised observation $\x_\tau$ \citepappdx{vincent2008extracting}. Traditionally, $q_\tau(\x_\tau | \x_0)$ is taken to be Gaussian with mean $\x_0$ and some standard deviation $\sigma$ and we make a point estimate $f_\theta(\x_\tau)$ for $\x_0$ given $\x_\tau$. The parameters $\theta$ are learned by minimising the MSE error
\begin{equation*}
    \J_{\textup{DAE}}(\theta) = \E[q_{0,\tau}(\x_0, \x_\tau)]{\|f_\theta(\x_\tau) - \x_0 \|^2}.
\end{equation*}

For a general denoising autoencoder on state space $\X$, we allow a probabilistic reconstruction $p_{0|\tau}^{(\theta)}(\x_0 | \x_\tau)$ of $\x_0$ depending on a set of parameters $\theta$, rather than a point estimate. We fit $\theta$ by minimising the objective
\begin{equation*}
    \J_{\textup{DAE}}(\theta) = \E[q_{0,\tau}(\x_0, \x_\tau)]{- \log p_{0|\tau}^{(\theta)}(\x_0 | \x_\tau)}.
\end{equation*}
Note that this reduces to the MSE objective in the case where $\X = \R^d$ and $p_{0|\tau}^{(\theta)}(\x_0 | \x_\tau)$ is Gaussian with mean $f_\theta(\x_\tau)$.

Suppose now that we have a generalised denoising autodencoder where the noising distribution $q_{0,\tau}(\x_0, \x_\tau)$ is given by the endpoints of a Markov process on $\X$ with generator $\L$ and the denoising distribution $p_{0|\tau}^{(\theta)}(\x_0 | \x_\tau)$ is given by the endpoints of a Markov process on $\X$ with generator $\K$. Suppose further that we parameterise the denoising process $\K$ via some function $\beta(\x, t)$ according to Assumptions \ref{ass:existenceM} and \ref{ass:generatorrelation} as in Section \ref{sec:general_objective}. Then Lemma \ref{lem:approx} implies that $\J_{\textup{DAE}}$ is equivalent to first order to the objective 
\begin{equation*}
    \J_{\textup{ISM}}(\beta) = \E[q_\tau(\x_\tau)]{\frac{\hat\L^\ast \beta(\x_\tau, \tau)}{\beta(\x_\tau, \tau)} + \hat\L \log \beta(\x_\tau, \tau)},
\end{equation*}
or alternatively to the corresponding generalised denoising score matching objective as in Section \ref{sec:relationship_score_matching}.

This generalises the result of \cite{vincent2011connection}, which demonstrated an equivalence between denoising autoencoders and denoising score matching in the case of Gaussian noise on $\R^d$. Indeed, we recover their result by considering the case where $q_{\tau | 0}(\x_\tau | \x_0)$ and $p_{0|\tau}^{(\theta)}(\x_0 | \x_\tau)$ are Gaussian, noting that these distributions are naturally induced as the distributions of the endpoints of diffusion processes.

Our work extends this equivalence between denoising autoencoders and generalised score matching as described in Section \ref{sec:relationship_score_matching} to arbitrary state spaces and noising/denoising distributions, provided that the noising and denoising distributions can be viewed as the marginals at the endpoints of Markov processes with known generators.

\section{Experimental details}
\label{app:experimentaldetails}

We give the details of our experimental set-up and results from Section \ref{sec:experiments}. Code for all of our experiments can be found at \texttt{github.com/yuyang-shi/generalized-diffusion}. 

\subsection{Inference on $\R^d$ using diffusion processes}
\label{app:realdiffusionexperiment}

The $g$-and-$k$ distribution with parameters $(A,B,g,k)$ is defined via its quantile function
\begin{equation*}
    F^{-1} (q|A,B,g,k) = A+B\left[1+0.8 \tanh\left(\frac{gz(q)}{2}\right)\right]\left(1+z(q)^{2}\right)^{k}z(q), 
\end{equation*}
where $z(q)$ denotes the $q$th quantile of the standard Gaussian distribution, and we require $B > 0$ and $k > -0.5$. The parameters $A,B,g,k$ control the location, scale, skewness and kurtosis of the distribution respectively \citepcont{prangle2020gk}. The prior on the parameters is uniform on $[0,10]^4$. For the diffusion model, we centre and rescale each parameter linearly to $[-1,1]$ in our implementation, and transform back to $[0,10]$ for reporting.  

As our noising process, we use the Ornstein--Uhlenbeck process $\d Y_t = -\frac{1}{2} Y_t \d t + \d B_t$. This has generator $\L = \partial_t - \frac{1}{2} \x \cdot \nabla + \frac{1}{2} \Delta$ and transition densities $q_{t|0}(\x_t | \x_0)$ which are Gaussian and available analytically. We can sample from the forward process at time $t$ by sampling $\x_0 \sim q_0(\x_0)$ and then $\x_t \sim q_{t | 0}(\x_t | \x_0)$. In practice, we apply a time-rescaling to the noising process following \citeappdx{song2021score}, in order to apply less noise at small times and move more quickly to the reference distribution at large times, by considering 
\begin{equation*}
    \d Y_t = -\frac{1}{2} \beta(t) Y_t \d t + \sqrt{\beta(t)} \d B_t. 
\end{equation*}
The $\beta$ schedule is set to be linear and monotonically increasing, i.e. 
\begin{equation}
    \beta(t)=\beta_{\textup{min}} + (\beta_{\textup{max}} - \beta_{\textup{min}})t.
    \label{eq:betaschedule}
\end{equation}
We set $\beta_{\textup{min}}=0.001$ and $\beta_{\textup{max}}$ is selected using a grid search from $2,4,6,8,10$. 

The reverse process is parameterised in terms of a conditional score network $s_\theta(\x_t,\xiobs, t)$ using multilayer perceptrons (MLPs). We first encode $\x$ and $\xiobs$ into 128-dimensional encodings using two separate MLPs with 3 layers and 512 hidden units in each layer. We then concatenate the two encodings as well as the time $t$ and pass through another MLP with 3 layers and 512 hidden units in each layer. The total number of neural network parameters is approximately 1.9M.
For $N=250$, we take in $\xiobs$ the full set of order statistics as inputs to our network, i.e. we sort the observation $\xiobs$ and take all $n=250$ values. For $N=10000$, we take $n=100$ evenly-spaced order statistics from our observation as inputs, following \citeappdx{fearnhead2012constructing}. 

Since we have access to the analytic transition densities, we train using the denoising score matching objective $\I_{\textup{DSM}}(\theta)$. We use a total of ${10}^6$ training samples $(\x_0, \xiobs_0) \sim p_\data$ during training. We optimise the network using the Adam optimiser with batch size 512 and learning rate 0.0001 with a cosine annealing schedule for 2.5M iterations. For sampling, we use the Euler-Maruyama method with 1000 steps to simulate from the reverse SDE. 

The ground truth posterior density is estimated with MCMC samples generated using the R package \texttt{gk} \citepappdx{prangle2020gk}. We compare our method with the semi-automatic ABC (SA-ABC) and Wasserstein SMC (W-SMC) methodologies using the R packages \texttt{abctools} \citepcont{nunes2015abctools} and \texttt{winference} \citepcont{bernton2019approximate}, as well as with Sequential Neural Posterior \citepappdx{greenberg2019automatic}, Likelihood \citepappdx{papamakarios2019sequential} and Ratio Estimation \citepappdx{durkan2020contrastive} approaches (SNPE, SNLE and SNRE) using the \texttt{sbi} Python package \citepappdx{tejero-cantero2020sbi}. All methods are set to use ${10}^6$ data samples to generate $5000$ posterior samples. We note that the default configurations offered by the \texttt{sbi} package for SNPE, SNLE and SNRE use comparatively smaller neural networks compared to our choice of score network $s_\theta(\x_t,\xiobs, t)$ detailed above. We have correspondingly increased the size of the neural networks for the three methods to approximately the same number of parameters. We also use Neural Spline Flows (NSFs, \citeappdx{durkan2019neural}) for SNPE as it is reported to have superior performance \citepcont{lueckmann2021benchmarking}. Other settings are kept to the default values. 

Compared to SA-ABC and W-SMC methodologies, neural-network based approaches including our DMM model require fitting a neural network and therefore are more computationally expensive at training time. However, our model is able to produce more accurate posterior estimates for fixed $\xiobs_0$, and perform amortised inference across a range of parameter values using the same number of ${10}^6$ data samples. Therefore, it is comparatively more data-efficient. 

As well as the plots in the main text, we also provide a pair plot comparing the approximate posterior from our diffusion model to the ground truth joint distribution in Fig.\ \ref{fig:gandkpairplot}. We see that our model provides results very close to the ground truth for the parameters $A$, $B$ and $g$ and can model the dependency between parameters, but gives a wider estimate in its reproduction of the posterior over $k$.


\begin{figure}[t]
    \centering
    \includegraphics[width=0.85\textwidth]{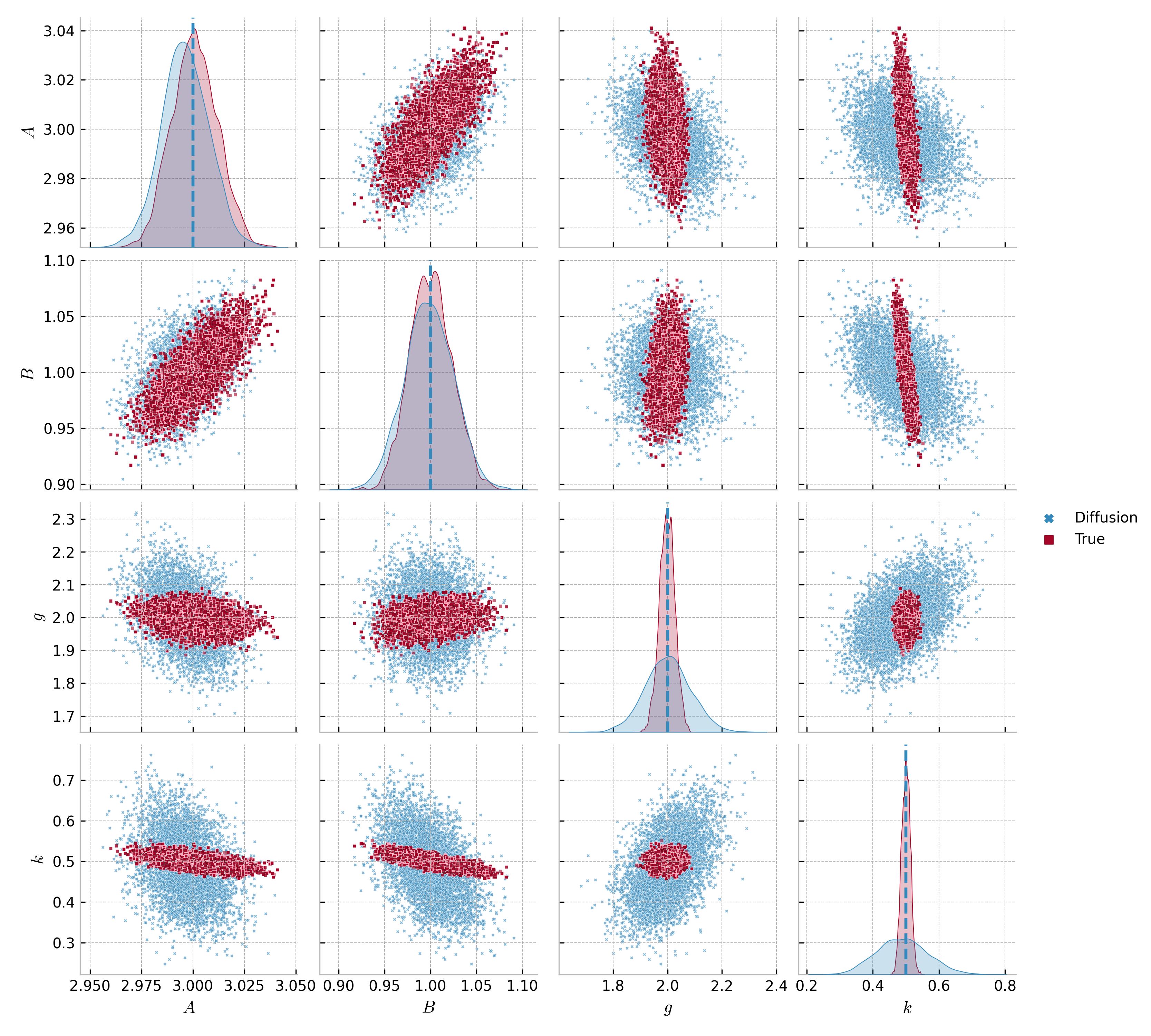}
    \caption{Pair plots of the simulated posterior samples from the diffusion model and the ground truth distribution using MCMC for the $g$-and-$k$ distribution example, with $\x_{\textup{true}}=(3,1,2,0.5)$ and $N=10000$. The off-diagonal plots are the pairwise scatter plots between each component of $\x$, and the diagonal plots reproduce each parameter's marginal kernel density estimate.}
    \label{fig:gandkpairplot}
\end{figure}

\subsection{MNIST digit image inpainting using discrete-space CTMCs}
\label{app:mnistexperiment}

Our implementation in discrete space closely follows that of \citecont{campbell2022continuous}, and we refer to their paper for further details. We denote our states as $\x_0 = (\x_0^1 \dots, \x_0^D)$ and for our noising process we use a CTMC with generator matrix $B := B^{1:D}(\x^{1:D}, \y^{1:D})$ which factorises over the dimensions, so $B^{1:D}(\x^{1:D}, \y^{1:D}) = \sum_{i=1}^D \tilde B(\x^i, \y^i) \mathbbm{1}_{\x^{1:D \setminus i} = \y^{1:D \setminus i}}$ for some rate matrix $\tilde B$ acting on a single dimension. Thus each pixel evolves independently as a CTMC on $\{0, \dots, 255\}$ with rate matrix $\tilde B$. We use the Gaussian rate matrix of \cite{campbell2022continuous} for $\tilde B$, which respects the ordinal structure of our state space and has a discretised Gaussian as its invariant distribution. The transition probabilities for this forward process can be calculated analytically efficiently by diagonalising the matrix and using matrix exponentials. This allows us to sample directly from the forward process at time $t$. 

Since we have access to the forward transition probabilities, we use the denoising parameterisation of the reverse process in terms of $p_\theta^{(t)}(\x_0 | \x_t)$ given in Equation (\ref{eq:discretedenoisingparam}), which we expect to lead to more stable training. We parameterise $p^{(t)}_\theta(\x_0 | \x_t, \xiobs)$ using a convolutional U-net \citepappdx{ho2020denoising}, taking as inputs both $\x_t$ and $\xiobs$ (concatenated in the channel dimension), as well as a sinusoidal embedding of the time $t$. The total number of neural network parameters is approximately 6.1M. The output of the network is defined as the mean and log scale of a logistic distribution for each pixel. The logistic distribution is then discretised into bins $\{0, \dots, 255\}$, and $p^{(t)}_\theta(\x_0 | \x_t, \xiobs)$ is defined as the product of the discretised logistic distributions across dimensions. 

We used the MNIST dataset \citepappdx{lecun2010mnist} which consists of images of handwritten digits. To train our model, we minimise the objective given in Example \ref{ex:CTMCobjective}. For optimisation, we use the Adam optimiser with batch size 128 and learning rate 0.0002 for 1M iterations. 
In order to simulate the reverse process efficiently, we use a tau-leaping approximation with 1000 steps (for more details see \citecont{campbell2022continuous}).

We compare our method to a continuous state space approach, as used for example in \cite{song2021score} and presented in Appendix \ref{app:realdiffusion}. We first normalize the data to range $[-1,1]$, and then learn a continuous-space diffusion model with an Ornstein–-Uhlenbeck noising process. All training configurations are kept the same as the discrete-space DMM. We report the Peak Signal-to-Noise Ratio (PSNR) and Structural Similarity Index Measure (SSIM) for both methods in Table \ref{tab:mnist_discrete_vs_continuous}.
PSNR and SSIM are two image quality metrics which measure the similarity between the generated posterior image and the ground truth. PSNR measures the pixel-by-pixel difference between two images and is a direct transformation of the mean squared error (MSE), whereas SSIM is a structural and more perceptional metric based on luminance, contrast and structure. For the continuous-space diffusion model, we report values for both the raw output samples (rescaled back to original scale), as well as with a further rounding step to the nearest integer in $\{0,\dots,255\}$. The discrete-space and continuous-space models appear to achieve comparable results, with the discrete-space model having a slightly worse PSNR score, but slightly better SSIM score, suggesting comparable perceptual quality.

\begin{table}
    \caption{\label{tab:mnist_discrete_vs_continuous}PSNR and SSIM scores for MNIST 14x14 inpainting using discrete-space and continuous-space DMMs. Higher values denote better performance. }
    \centering
    \fbox{%
    \begin{tabular}{c|c|c|c}
          & Discrete-space & Continuous-space (raw) & Continuous-space (rounded) \\
         \hline
         PSNR & 16.63 & 16.72 & \textbf{16.75} \\
         \hline
         SSIM & \textbf{0.757} & 0.706 & 0.723 \\
    \end{tabular}}
\end{table}

\subsection{Large-scale image super-resolution using discrete-space CTMCs}
\label{app:imagenetexperimentdetails}

\begin{figure}
    \centering
    \includegraphics[height=2.8cm]{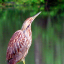}
    \hspace{2mm}
    \includegraphics[height=2.8cm]{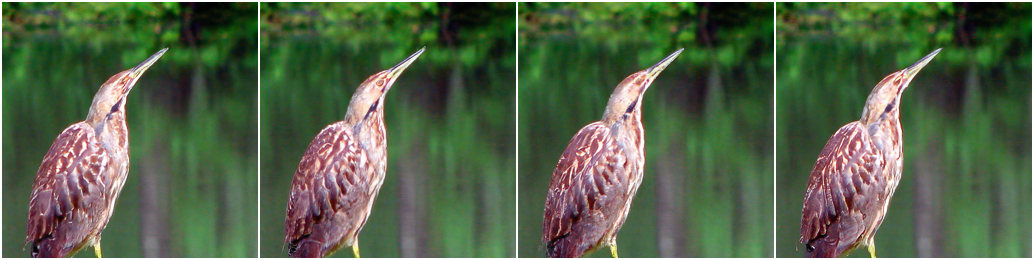} \\ 
    
    \includegraphics[height=2.8cm]{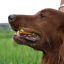}
    \hspace{2mm}
    \includegraphics[height=2.8cm]{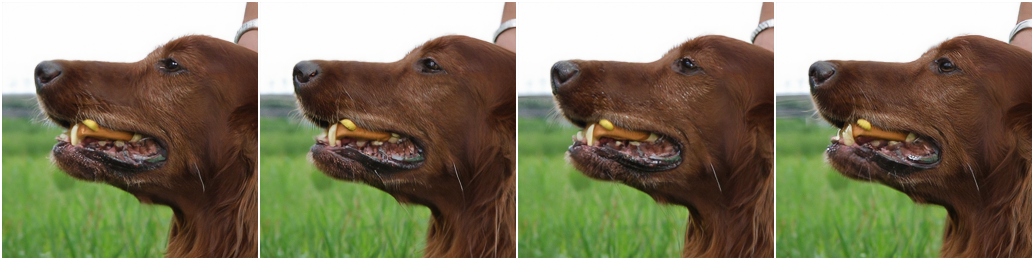} \\ 
    
    \includegraphics[height=2.8cm]{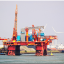}
    \hspace{2mm}
    \includegraphics[height=2.8cm]{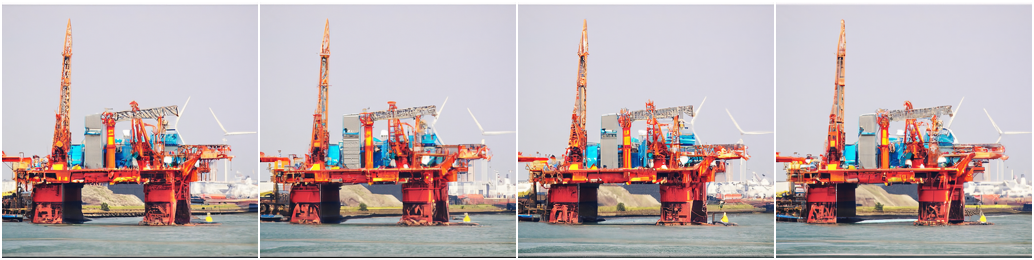} \\ 

    \includegraphics[height=2.8cm]{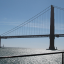}
    \hspace{2mm}
    \includegraphics[height=2.8cm]{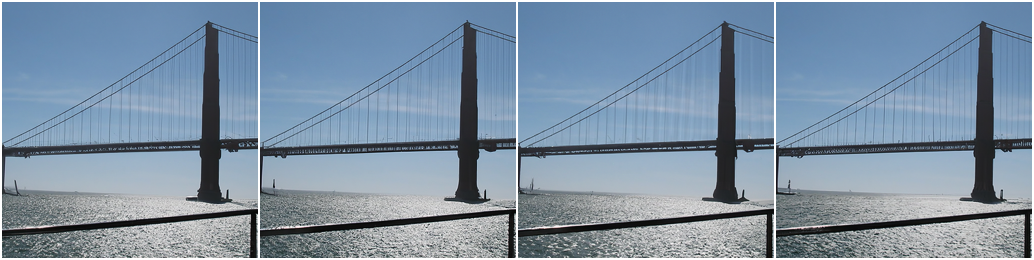}
    
    \includegraphics[height=2.8cm]{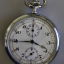}
    \hspace{2mm}
    \includegraphics[height=2.8cm]{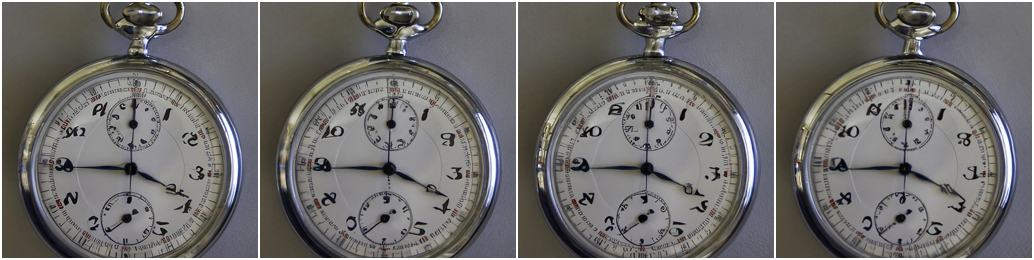}
    
    \includegraphics[height=2.8cm]{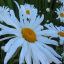}
    \hspace{2mm}
    \includegraphics[height=2.8cm]{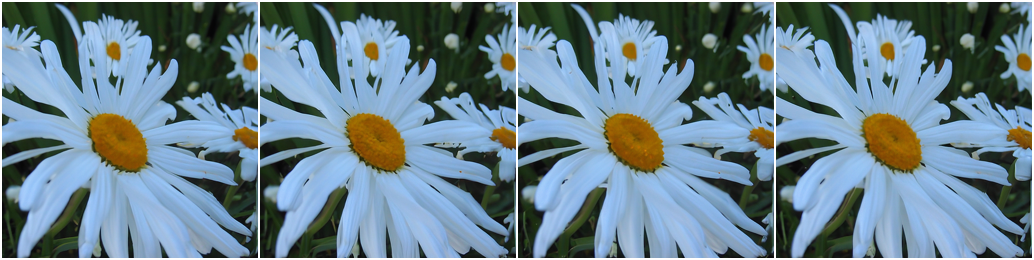} \\
    
    \caption{Image super-resolution results ($64\times64\to256\times256$) on the ImageNet dataset using our DMM. The first column is the input image and remaining columns are samples from the DMM.}
    \label{fig:imagenet_superres}
\end{figure}

We perform an additional experiment using discrete-space DMMs for a large-scale image inverse problem on the ImageNet dataset \citepappdx{russakovsky2015imagenet}. We train a DMM using CTMC noising and generative processes to perform 4-fold image super-resolution.

Each input image has $64 \times 64$ pixels and three RGB colour channels, and we aim to output images at the higher resolution of $256 \times 256$ pixels which are consistent with the input images. Our state space $\X = \{0,\dots,255\}^{3 \times 256 \times 256}$.

The noising process, reverse process parameterisation, and neural network design are the same as in Section \ref{app:mnistexperiment}, but we use a larger neural network for this task. As the starting point of our network optimisation, we utilise the pretrained network weights for continuous diffusions by \citeappdx{dhariwal2021diffusion}, but we retrain the network for our discrete-space DMM using the objective in Example \ref{ex:CTMCobjective}. The total number of neural network parameters is approximately 311.8M. We train the network using the Adam optimiser with batch size 4 and learning rate $2\times{10}^{-5}$ for an additional 200000 iterations. For sampling, we use tau-leaping with 1000 steps.

We plot the simulated super-resolution samples in Fig.\ \ref{fig:imagenet_superres} for a number of low-resolution images generated from the ImageNet validation dataset. As shown in the images, the discrete diffusion model outputs different super-resolution samples that are realistic to the eye, and coherent with the low-resolution images, demonstrating that DMMs can continue to provide high-quality posterior samples even in very high-dimensional scenarios situations where the prior $p_\data(\x)$ is unavailable and standard ABC or MCMC techniques are not available.

\subsection{Modelling distributions on $SO(3)$ using manifold diffusions}
Recall that our noising process on $SO(3)$ is Brownian motion with generator $\L = \del_t + \frac{1}{2}\Delta$. Since $SO(3)$ is compact, this converges to the uniform measure for large times; see e.g. \cite{debortoli2022riemannian}. For this process, the transition probabilities can be explicitly written as
\begin{equation}
\label{eq:analytictransitionSO3}
    q_{t|0}(\x_t| \x_0) \propto \sum_{\ell=0}^\infty (2\ell+1)e^{-\ell(\ell+1)t/2} \frac{\sin\lr{\lr{\ell + \frac{1}{2}}\alpha}}{\sin(\alpha/2)},
\end{equation}
where $\alpha = \arccos \left[2^{-1}(\Tr(\x_0^T \x_t) - 1)\right]$ is the angle between $\x_t$ and $\x_0$, and $\x_t, \x_0 \in SO(3)$ are in matrix form. For completeness, we provide the derivation of this result below in Section \ref{app:so3transitionderivation}.

Given this expression, to sample from $q_{t|0}(\x_t| \x_0)$, we follow \citeappdx{leach2022denoising} and first sample the rotation axis $v$ uniformly from the sphere $S^2 \subset \R^3$. Then, we sample the rotation angle $\alpha \in [0,\pi]$ using inverse transform sampling from the distribution
\begin{equation*}
    f_t(\alpha) = \frac{1 - \cos (\alpha)}{\pi} \sum_{\ell=0}^\infty (2\ell+1)e^{-\ell(\ell+1)t/2} \frac{\sin\lr{\lr{\ell + \frac{1}{2}}\alpha}}{\sin(\alpha/2)},
\end{equation*}
where the normalising factor $(1 - \cos(\alpha))/\pi$ is the measure on rotation angles induced by the uniform measure on $SO(3)$. For larger $t$, we find that the above series converges quickly and evaluating summation terms up to $l=5$ gives an accurate approximation. For $t < 1$, the above series converges slowly, and so we use the approximation
\begin{equation*}
    f_t(\alpha) \approx \frac{1 - \cos (\alpha)}{2 \sqrt{\pi} \sin (\alpha / 2)} \lr{\frac{t}{2}}^{-\frac{3}{2}} e^{\frac{t}{8}-\frac{\alpha^2}{2t}} \lrb{\alpha - e^{-\frac{2\pi^2}{t}}\lr{(\alpha - 2 \pi)e^{ \frac{2\pi \alpha}{t}}
    + (\alpha + 2 \pi)e^{-\frac{2 \pi \alpha}{t}}}}
\end{equation*}
from \citeappdx{leach2022denoising} instead. From the angle $\alpha$ and the axis $v = (x, y, z)$, we define the skew symmetric matrix $V$ associated to $v$ to be
\begin{equation*}
    V = \begin{pmatrix} 0 & z & -y \\ -z & 0 & x \\ y & -x & 0 \end{pmatrix}
\end{equation*}
and calculate the corresponding rotation matrix using Rodrigues' formula
\begin{equation*}
    R = I + \sin(\alpha) V + (1 - \cos(\alpha)) V^2.
\end{equation*}
Finally, we set $\x_t = R \x_0$. In this way, we can directly sample from the noising process at time $t$.

The reverse process is generated by $\K = \partial_t + s_\theta(\x, t) \cdot \nabla + \frac{1}{2}\Delta$ by Example \ref{ex:manifolddiffusion}, and the score network is parameterised as $s_\theta(\x, t)=\sum_{i=1}^3 s_\theta^i(\x, t) E_i(\x)$, using a basis $\{E_i\}_{i=1}^3$ of the tangent bundle.

We use the denoising score matching objective $\I_{\textup{DSM}}(\theta)$ to learn $\theta$ (see Section \ref{app:riemannianmanifolds}). To compute the score $\nabla \log q_{t|0}(\x_t | \x_0)$, we use automatic differentiation on Equation (\ref{eq:analytictransitionSO3}), where $\x_t, \x_0 \in \R^{3 \times 3}$ are represented in matrix form, followed by projection to the tangent space at $\x_t$. For small times, we find this can be numerically unstable, and so we use Varadhan's approximation
\begin{equation*}
    \lim_{t \rightarrow 0} \; t \nabla \log q_{t| 0}(\x_t| \x_0) =  \exp_{\x_t}^{-1}(\x_0)
\end{equation*}
for the heat kernel $q_{t|0}(\x_t | \x_0)$ at small times instead \citepappdx{debortoli2022riemannian}. 

Once we have learned the score network, we generate approximate samples from the reverse process using the Geodesic Random Walk method of \citeappdx{debortoli2022riemannian}, which corresponds to performing an Euler-Maruyama discretisation, taking Gaussian steps in the tangent space and then projecting back to the manifold using the exponential map.

\subsubsection{Derivation of analytic transition probabilities}
\label{app:so3transitionderivation}

First, we calculate the metric tensor using the quaternion chart on $SO(3)$, where the unit quaternion $w + x\i + y\j + z\k$ represents a rotation by an angle $\alpha = 2 \cos^{-1}(w)$ about the axis $(x,y,z)$, and we consider the coordinates $(x,y,z)$ to be our local chart. If $r = w + x\i + y\j + z\k$, we find the metric at $r$ by considering two small displacements $r + \d r$ and $r + \d r'$, rotating $r$ back to the identity, and then using the  fact that near the identity the metric is given by $4 \d x^2 + 4 \d y^2 + 4 \d z^2$ (where the scaling is chosen to correspond to the definition of the exponential map used by \citeappdx{debortoli2022riemannian} and \citeappdx{leach2022denoising}). Writing
\begin{align*}
    r + \d r & = (w + \d w) + (x + \d x)\i + (y + \d y)\j + (z + \d z)\k, \\
    r + \d r' & = (w + \d w') + (x + \d x')\i + (y + \d y')\j + (z + \d z')\k,
\end{align*}
where we have $w \d w + x \d x + y \d y + z \d z = 0$ and $w \d w' + x \d x' + y \d y' + z \d z' = 0$, and noting that composition of rotations corresponds to multiplication in the quaternion algebra, we have
\begin{align*}
    r^{-1} (r + \d r) & = \lr{w - x\i - y\j - z\k} \lr{(w + \d w) + (x + \d x)\i + (y + \d y)\j + (z + \d z)\k} \\
    & = 1 + \lr{- x \d w + w \d x - y \d z + z \d y} \i + \lr{- y \d w + w \d y - z \d x + x \d z} \j \\
    & \hspace{5mm} + \lr{- z \d w + w \d z - x \d y + y \d x} \k
\end{align*}
and similarly for $r^{-1} (r + \d r')$. Therefore, the metric is expressed by
\begin{equation*}
    4 \lrcb{ \lr{w + \frac{x^2}{w}} \d x + \lr{- y + \frac{xz}{w}} \d z + \lr{z + \frac{xy}{w}} \d y}^2 + \text{\textit{cyclic terms}}.
\end{equation*}
Multiplying out, collecting like terms and inspecting the coefficients of $\d x^2$, $\d x \d y$ etc., we see that
\begin{equation*}
    g_{ij} = \frac{4}{w^2} \begin{pmatrix}w^2 + x^2 & xy & xz \\ xy & w^2 + y^2 & yz \\ xz & yz & w^2 + z^2\end{pmatrix}
\end{equation*}
and we can calculate $|g| = 1/w^2$. Inverting the metric, we get
\begin{equation*}
    g^{ij} = \frac{1}{4} \begin{pmatrix} (1-x^2) & - xy & - xz \\ - xy & (1-y^2) & - yz \\ - xz & - yz & (1-z^2) \end{pmatrix}.
\end{equation*}
Now, we want to switch to using $w$ as a coordinate, and to find expressions for $\Delta f$ where $f(w)$ is a function only of $w$. To this end, we have
\begin{align*}
    \nabla f & = \frac{\del f}{\del w} \d w = - \frac{1}{w} \frac{\del f}{\del w} \lr{x \d x + y \d y + z \d z}, \\
    g^{ij} (\nabla_f)_j & = - \frac{1}{4w} \frac{\del f}{\del w} \begin{pmatrix} (1-x^2) & - xy & - xz \\ - xy & (1-y^2) & - yz \\ - xz & - yz & (1-z^2) \end{pmatrix} \begin{pmatrix} x \\ y \\ z \end{pmatrix} = - \frac{w}{4} \frac{\del f}{\del w} \begin{pmatrix} x \\ y \\ z \end{pmatrix},
\end{align*}
so
\begin{equation*}
     \Delta f = w \; \del_i\lr{\frac{1}{w} g^{ij} (\nabla f)_j} = - \frac{3w}{4} \frac{\del f}{\del w} + \frac{1 - w^2}{4}\frac{\del^2 f}{\del w^2}.
\end{equation*}

If we make the substitution $w = \cos (\alpha / 2)$, where $\alpha$ is the angle of the corresponding rotation, then $\d w = - \frac{1}{2} \sin (\alpha /2) \d \alpha$, and we get
\begin{equation*}
    \Delta f = \cot(\alpha/2) \frac{\del f}{\del \alpha} + \frac{\del^2 f}{\del \alpha^2}.
\end{equation*}
To find the transition probabilities, we must solve the Fokker--Planck equation
\begin{equation*}
    \frac{\del q}{\del t} = \frac{1}{2} \Delta q
\end{equation*}
on $SO(3)$, subject to the initial condition of a delta mass at $I$. By symmetry, we know the solution will be rotationally symmetric, so we can write the solution as $q(\alpha, t)$. Now, we look for separable solutions of the form $q(\alpha, t) = T(t)A(\alpha)$. We see that we must have
\begin{equation*}
    \frac{1}{T} \frac{\d T}{\d t} = \frac{1}{2A} \lr{\cot(\alpha/2) \frac{\d A}{\d \alpha} + \frac{\d^2 A}{\d \alpha^2}}.
\end{equation*}
Separating the two equations, we see that we require
\begin{equation*}
    \frac{\d T}{\d t} = \frac{1}{2} \lambda T, \hspace{10mm} \cot(\alpha/2) \frac{\d A}{\d \alpha} + \frac{\d^2 A}{\d \alpha^2} = \lambda A,
\end{equation*}
for some fixed $\lambda$. The first equation has solution $T(t) = e^{\lambda t/2}$, while a solution to the second is given by
\begin{equation*}
    A(\alpha) = \frac{\sin\lr{\lr{\mu + \frac{1}{2}}\alpha}}{\sin(\alpha/2)},
\end{equation*}
where $\mu$ satisfies $-\mu(\mu+1) = \lambda$. In addition, the boundary conditions force $\mu$ to be an integer. Combining these expressions, we see that the solution is of the form
\begin{equation*}
    q(\alpha, t) = \sum_{\ell=0}^\infty \beta_\ell e^{-\ell(\ell+1)t / 2} \frac{\sin\lr{\lr{\ell + \frac{1}{2}}\alpha}}{\sin(\alpha/2)}
\end{equation*}
for some coefficients $\beta_\ell$. Finally, we have the initial condition that $q(\alpha,0) = 0$ for $\alpha > 0$ and $\int_{SO(3)} q(\x, 0) f(\x) \d \nu (\x) = f(I)$ where $\nu$ is the uniform probability measure on $SO(3)$. Up to a scaling factor, this is satisfied if and only if $\beta_{\ell} \propto (2\ell + 1)$. Putting this all together, we obtain Equation (\ref{eq:analytictransitionSO3}).

\subsection{Mixture of wrapped normal distributions on $SO(3)$}
\label{app:wrappednormalexperiment}

We consider modelling a mixture of wrapped normal distributions on $SO(3)$. The wrapped normal distribution $\mathcal{N}^W (\x \;|\; \mu, \sigma^2)$ with mean $\mu$ and variance $\sigma^2$ is defined here as the transformed distribution via sampling $\w \sim \mathcal{N} (\w \;|\; 0, \sigma^2)$, where $\w \in \R^{3\times3}$, from the standard normal distribution with variance $\sigma^2$, projecting $\w$ onto the tangent space via $\mb{v} = \frac{\w - \w^T}{2}$, then applying the exponential map $\x=\exp_{\mu}(\mb{v})$ at $\mu$. While we could apply standard parametric learning methods which involve learning of $\{\mu_m, \sigma_m \}$ directly, we do not rely on the specific form of the data distribution $p_{\data}$, which allows us to model different distributions flexibly. We consider modelling of a mixture of wrapped normal distributions with $M=16$ mixtures.

We apply a time-rescaling for the noising process, which is given by $\L = \partial_t + \frac{1}{2} \beta(t) \Delta$ with the linear $\beta$ schedule given in Equation (\ref{eq:betaschedule}). Then, the reverse process is generated by $\K = \partial_t + \beta(t) s_\theta(\x, t) \cdot \nabla + \frac{1}{2}\beta(t) \Delta$. We use an MLP with 5 layers and 512 hidden units in each layer to output a vector of dimension 3 parameterising $\{ s_\theta^i(\x, t) \}_{i=1}^3$. 
We train the network using the Adam optimiser with batch size 512 and learning rate 0.0002 with a cosine annealing schedule for 100000 iterations.

We learn both the unconditional distribution $p_{\data}(\x)$ and the conditional distribution $p_{\data}(\x | m)$ when conditioned on the cluster member $m$. In the conditional case, we learn a conditional score model $s_\theta(\x,m,t)$ under the same settings.

Fig.\ \ref{fig:conditional_so3} shows the results from our conditional model for $p_\data(\x | m)$, where we compare the unwrapped distributions in the tangent space between the ground truth normal distribution and the modelled distribution of mixture member $m=1$, and plot a representative sample from our conditional model. We see that our model targets the correct mixture accurately. Our visualisations of distributions on $SO(3)$ are adapted from \cite{murphy2021implicit}.

\begin{figure}
    \centering
    \raisebox{-0.5\height}{\includegraphics[width=0.45\textwidth]{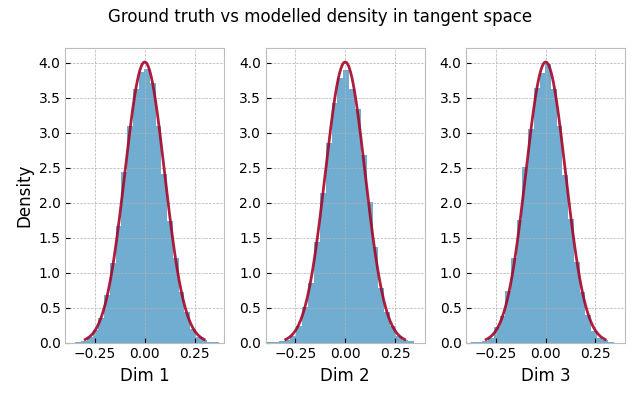}} \hspace{0.5cm}
    \raisebox{-0.5\height}{\includegraphics[width=0.43\textwidth]{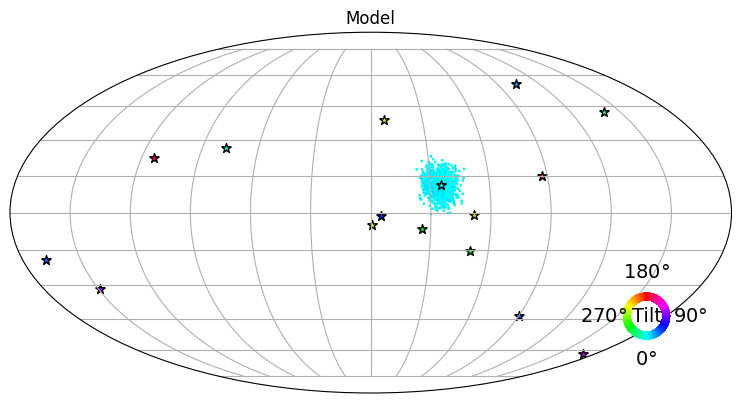}}
    \caption{(Left) Histogram of samples from our model conditioned on the mixture member $m=1$ compared to the ground truth normal density, represented in the tangent space of $SO(3)$. (Right) Conditional samples from the model for $m=1$. The axis of rotation and rotation angle are represented by position and colour respectively.}
    \label{fig:conditional_so3}
\end{figure}

We compare our method to the method of \citecont{debortoli2022riemannian}, in which the denoising diffusion model for this task is trained by simulating the forward process using the Geodesic Random Walk and using the DSM loss with Varadhan's approximation, rather than using the analytic transition densities given in Appendix \ref{app:so3transitionderivation} as we do. We compare the two methods using the learned models' test-set log-likelihood, calculated using the probability flow ODE as in \citecont{debortoli2022riemannian}, as well as the average time per training iteration. Our results are shown in Table \ref{tab:mix_so3}. We see that both methods achieve comparable log-likelihoods, but our method is about 15\% more efficient during training since having the analytic transition densities means that we can simulate the forward noising process in a single step. 

\begin{table}
\caption{\label{tab:mix_so3}Test set log-likelihood and time per training iteration for denoising models on $SO(3)$. Mean and standard deviation reported over 5 seeds.}
\centering
\fbox{%
\begin{tabular}{|c|c|c|c|c|}
\hline 
 & & & & Time per\\
 & $M=16$ & $M=32$ & $M=64$ & \;\; iteration (ms) \; \; \\
\hline 
\hline 
\citecont{debortoli2022riemannian} & 0.864±0.026 & 0.174±0.025 & -0.516±0.016 & 55.18±2.783\\
\hline 
Analytic (ours) & 0.872±0.026 & 0.175±0.025 & -0.515±0.016 & 47.23±2.134\\
\hline 
\end{tabular}}
\end{table}

\subsection{Pose estimation on the SYMSOL dataset}
\label{app:symsolexperiment}

We give details for the pose estimation task on the SYMSOL dataset. We use a similar network design for the conditional score $s_\theta(\x_t,\xiobs, t)$ as \citeappdx{murphy2021implicit}, composed of a vision recognition model for processing the input images $\xiobs$, and an MLP for outputting the score. For the vision recognition model, we utilise pretrained ResNet-50 backbone without the final fully-connected classification layer, which outputs a 2048-dimensional embedding. We next get sinusoidal positional embeddings of $\x_t$ and $t$, use linear layers to transform all embeddings into 256 dimensions and take the summed embedding. This also allows efficient computations of embeddings with a single $\xiobs$ and multiple values of $(\x_t, t)$ as the computationally expensive forward pass through the vision recognition model only needs to be taken once. Thus, we simulate a small number of $(\x_t,t)$ pairs given each pair $(\x_0,\xiobs)$ at each step for more efficient training. We finally pass the embedding into an MLP with 3 layers and 256 hidden units in each layer. 

Compared to the Implicit-PDF methodology by \citeappdx{murphy2021implicit}, which maintains a grid on $SO(3)$ and approximates the density pointwise, our DMM model directly learns a sampling method and does not require maintaining a grid. Therefore, our method is more general and not specific to $SO(3)$. For our implementation, we modify their network structure to take in the time $t$, and output the score parameterisation of dimension 3 as opposed to the unnormalised log density of dimension 1. We optimise the network using the Adam optimiser with batch size 128 and learning rate 0.0001 with a cosine annealing schedule for 100000 iterations.

We include further visualisations of the generated samples when conditioned on 2D views of different shapes in Fig.\ \ref{app:symsolresults1}. As shown in the plots, the samples generated using DMM are all close to the ground truth and cover all modes of the class of rotational symmetries. 


\begin{figure}
    \centering
    \centering
    \hspace*{-0.1cm}\raisebox{-0.5\height}{\includegraphics[width=1.5cm]{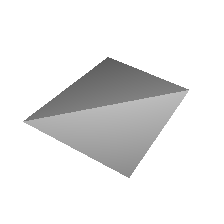}}
    \raisebox{-0.5\height}{\includegraphics[trim={2.5cm 0 0 0}, clip, width=12.5cm]{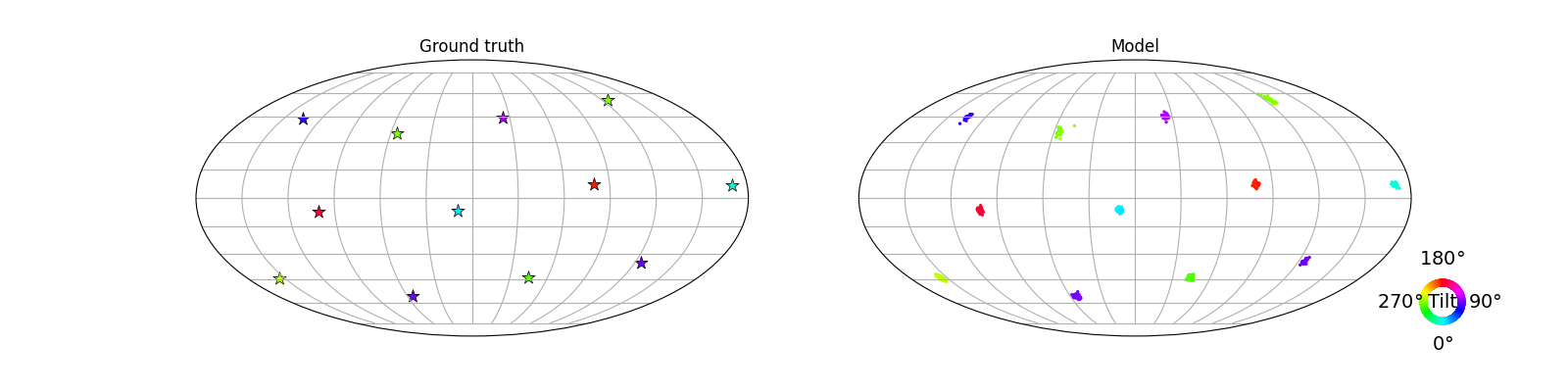}} \\ 
    \hspace*{-0.1cm}\raisebox{-0.5\height}{\includegraphics[width=1.5cm]{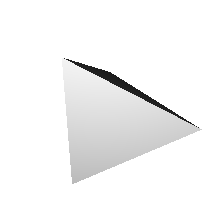}}
    \raisebox{-0.5\height}{\includegraphics[trim={2.5cm 0 0 0}, clip, width=12.5cm]{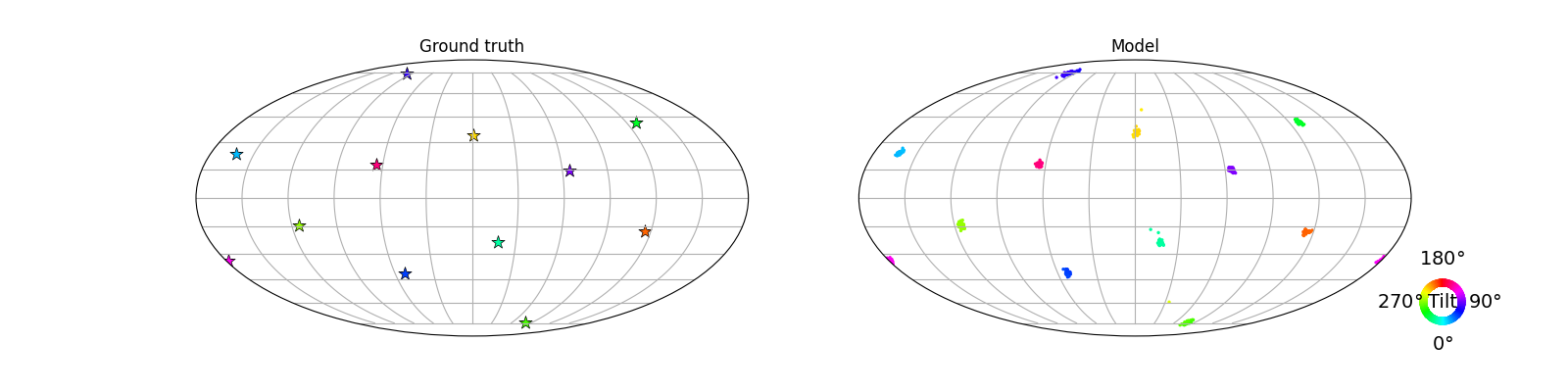}} \\ \vspace{10pt}
    \hspace*{-0.1cm}\raisebox{-0.5\height}{\includegraphics[width=1.5cm]{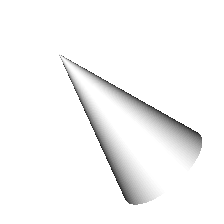}}
    \raisebox{-0.5\height}{\includegraphics[trim={2.5cm 0 0 0}, clip, width=12.5cm]{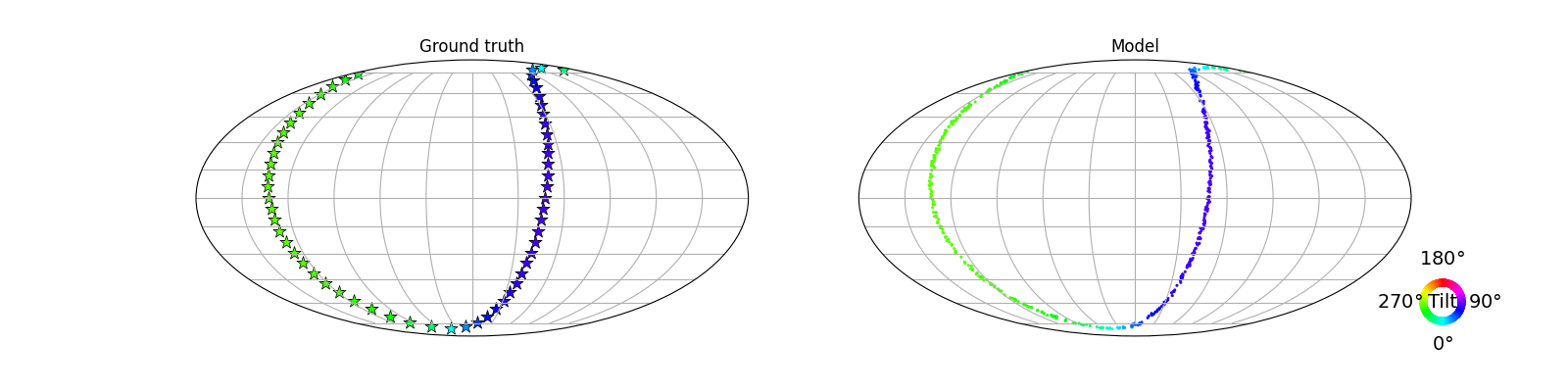}} \\
    \hspace*{-0.1cm}\raisebox{-0.5\height}{\includegraphics[width=1.5cm]{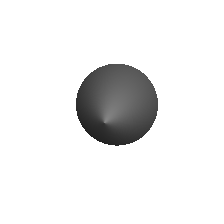}}
    \raisebox{-0.5\height}{\includegraphics[trim={2.5cm 0 0 0}, clip, width=12.5cm]{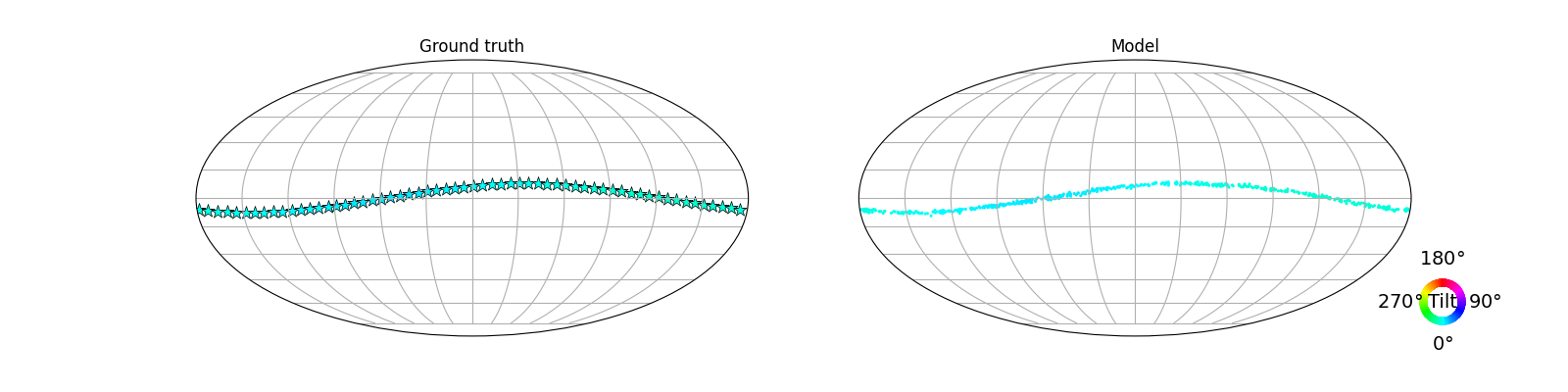}} \\
    \caption{Samples from the ground truth (plotted as stars, middle) and our pose estimation DMM (right) conditioned on 2D views of shapes (left). The axis of rotation and rotation angle are represented by position and colour respectively.}
    \label{app:symsolresults1}
\end{figure}


\subsection{Approximation of distributions over measures using Wright--Fisher diffusions}
\label{app:wrightfisherexperiments}

Finally, we evaluate the Wright--Fisher diffusion framework from Appendix \ref{app:wrightfisher} for modeling distributions over measures on a finite state space. We test our framework by attempting to model mixtures of Dirichlet distributions $p_{\data}(\x) = \frac{1}{M} \sum_{m=1}^{M} \textup{Dirichlet}(\alpha_m)$ with parameters $\alpha_m\in\R^N$. We consider $M=4$ mixtures and vary the number of dimensions $N$ of the simplex.

As in Appendix \ref{app:wrightfisher}, we use a Wright--Fisher diffusion with $q_{ij} = \vartheta_j$ for all $i \neq j$ as our noising process, and set $\vartheta_j = 3$ for all $j = 1, \dots, N$. We also apply a time rescaling to the forward process as in Equation (\ref{eq:betaschedule}). We set $\beta_{\textup{min}}=0.001$ and $\beta_{\textup{max}}$ is selected using a grid search from $0.5,1,2$. We simulate the forward diffusion process using the exact simulation algorithm of \citeappdx{jenkins2017exact}, which exploits the eigenfunction decomposition of the Wright--Fisher process transition function given in Equation (\ref{eq:WFeigenfunction}) and works by sampling from the ancestral process $A^\Theta_\infty(t)$ whose distribution is determined by the functions $\{d^\Theta_n(t):n=0,1,\dots\}$. For very small times $t$, we also use a normal approximation for simulating $A^\Theta_\infty(t)$. For more details, we refer the reader to \cite{jenkins2017exact}.

We learn the score network with the parameterisation $s_\theta^i(\p, t) = p_i \del (\log \beta(\p, t)) / \del p_i$ using the implicit score matching loss (\ref{eq:WF_ism}). We parameterise $s_\theta(\p, t)$ using an MLP with 4 layers and 512 hidden units in each layer to output a vector of dimension $N$. We train the network using the Adam optimiser with batch size 128 and learning rate 0.0001 with a cosine annealing schedule for 100000 iterations.

We visualise the results of this experiment in Fig.\ \ref{fig:mix_dir} for a 3-dimensional example. As can be seen, the DMM model is able to learn the ground truth distribution very accurately. We also report in Table \ref{tab:mix_dir_elbo} the ground truth log-likelihood of the data distribution $p_{\data}(\x)$ and the ELBO of the DMM model given by (\ref{eq:tractable_ELBO}) using the ISM loss, as the number of dimensions $N$ increases. We observe that the model's ELBO is consistently close to the true data log-likelihood, which demonstrates the scalability of the DMM model. 


\begin{table}
  \caption{\label{tab:mix_dir_elbo}True test data log-likelihood compared to the DMM model ELBO as given by (\ref{eq:tractable_ELBO}) using the ISM loss for the mixture of Dirichlet example. Mean and standard deviation reported over 5 seeds.}
  \fbox{%
    \begin{tabular}{c|c|c|c|c}
    \hline 
    Dimension of simplex & $N=3$ & $N=5$ & $N=10$ & $N=20$\\
    \hline 
    Data & 1.321±0.340 & 4.122±0.242 & 15.288±0.389 & 45.914±0.694\\
    \hline 
    Model & 1.158±0.160 & 4.017±0.208    & 15.061±0.428 & 45.494±0.698\\
    \hline 
    \end{tabular}}
\end{table}



%% file: main.bbl
\begin{thebibliography}{}

\bibitem[\protect\citeauthoryear{Cattiaux, Conforti, Gentil, and
  L{\'e}onard}{Cattiaux et~al.}{2023}]{cattiaux2021time}
Cattiaux, P., G.~Conforti, I.~Gentil, and C.~L{\'e}onard (2023).
\newblock Time reversal of diffusion processes under a finite entropy
  condition.
\newblock {\em Annales de l'Institut Henri Poincar{\'e} (B) Probabilit{\'e}s et
  Statistiques\/}~{\em 59\/}(4), 1844--1881.

\bibitem[\protect\citeauthoryear{Dong}{Dong}{2003}]{dong2003feller}
Dong, R. (2003).
\newblock {Feller Processes and Semigroups}.
\newblock {\em Lecture notes, UC Berkeley,
  https://www.stat.berkeley.edu/\raisebox{0.5ex}{\texttildelow}pitman/s205s03/lecture27.pdf\/}.

\bibitem[\protect\citeauthoryear{Durkan, Bekasov, Murray, and
  Papamakarios}{Durkan et~al.}{2019}]{durkan2019neural}
Durkan, C., A.~Bekasov, I.~Murray, and G.~Papamakarios (2019).
\newblock {Neural Spline Flows}.
\newblock {\em NeurIPS\/}.

\bibitem[\protect\citeauthoryear{Durkan, Murray, and Papamakarios}{Durkan
  et~al.}{2020}]{durkan2020contrastive}
Durkan, C., I.~Murray, and G.~Papamakarios (2020).
\newblock {On Contrastive Learning for Likelihood-free Inference}.
\newblock {\em ICML\/}.

\bibitem[\protect\citeauthoryear{Ethier and Kurtz}{Ethier and
  Kurtz}{1993}]{ethier1993fleming}
Ethier, S.~N. and T.~G. Kurtz (1993).
\newblock {Fleming--Viot Processes in Population Genetics}.
\newblock {\em SIAM Journal on Control and Optimization\/}~{\em 31}, 345--386.

\bibitem[\protect\citeauthoryear{Fearnhead and Prangle}{Fearnhead and
  Prangle}{2012}]{fearnhead2012constructing}
Fearnhead, P. and D.~Prangle (2012).
\newblock {Constructing Summary Statistics for Approximate Bayesian
  Computation: Semi-automatic Approximate Bayesian Computation}.
\newblock {\em Journal of the Royal Statistical Society: Series B (Statistical
  Methodology)\/}~{\em 74\/}(3), 419--474.

\bibitem[\protect\citeauthoryear{Greenberg, Nonnenmacher, and Macke}{Greenberg
  et~al.}{2019}]{greenberg2019automatic}
Greenberg, D.~S., M.~Nonnenmacher, and J.~H. Macke (2019).
\newblock {Automatic Posterior Transformation for Likelihood-Free Inference}.
\newblock {\em ICML\/}.

\bibitem[\protect\citeauthoryear{Jenkins and Span{\`{o}}}{Jenkins and
  Span{\`{o}}}{2017}]{jenkins2017exact}
Jenkins, P.~A. and D.~Span{\`{o}} (2017).
\newblock {Exact Simulation of the Wright{\textendash}Fisher Diffusion}.
\newblock {\em The Annals of Applied Probability\/}~{\em 27\/}(3).

\bibitem[\protect\citeauthoryear{Karatzas and Shreve}{Karatzas and
  Shreve}{1991}]{karatzas1991brownian}
Karatzas, I. and S.~E. Shreve (1991).
\newblock {\em {Brownian Motion and Stochastic Calculus}}.
\newblock Springer Science {\&} Business Media.

\bibitem[\protect\citeauthoryear{Leach, Schmon, Degiacomi, and Willcocks}{Leach
  et~al.}{2022}]{leach2022denoising}
Leach, A., S.~M. Schmon, M.~T. Degiacomi, and C.~G. Willcocks (2022).
\newblock {Denoising Diffusion Probabilistic Models on SO(3) for Rotational
  Alignment}.
\newblock {\em ICLR 2022 Workshop on Geometrical and Topological Representation
  Learning\/}.

\bibitem[\protect\citeauthoryear{LeCun, Cortes, and Burges}{LeCun
  et~al.}{2010}]{lecun2010mnist}
LeCun, Y., C.~Cortes, and C.~Burges (2010).
\newblock {MNIST handwritten digit database}.
\newblock {\em ATT Labs [Online]. Available:
  http://yann.lecun.com/exdb/mnist\/}.

\bibitem[\protect\citeauthoryear{Molchanov}{Molchanov}{1968}]{molchanov1968strong}
Molchanov, S.~A. (1968).
\newblock {Strong Feller Property of Diffusion Processes on Smooth Manifolds}.
\newblock {\em Theory of Probability \& Its Applications\/}~{\em 13}, 471--475.

\bibitem[\protect\citeauthoryear{Métivier}{Métivier}{1982}]{metivier1982semimartingales}
Métivier, M. (1982).
\newblock {\em {Semimartingales}}.
\newblock De Gruyter.

\bibitem[\protect\citeauthoryear{Palmowski and Rolski}{Palmowski and
  Rolski}{2002}]{palmowski2002technique}
Palmowski, Z. and T.~Rolski (2002).
\newblock {A Technique for Exponential Change of Measure for Markov Processes}.
\newblock {\em Bernoulli\/}~{\em 8}, 767--785.

\bibitem[\protect\citeauthoryear{Papamakarios, Sterratt, and
  Murray}{Papamakarios et~al.}{2019}]{papamakarios2019sequential}
Papamakarios, G., D.~C. Sterratt, and I.~Murray (2019).
\newblock {Sequential Neural Likelihood: Fast Likelihood-free Inference with
  Autoregressive Flows}.
\newblock {\em AISTATS\/}.

\bibitem[\protect\citeauthoryear{Prangle}{Prangle}{2020}]{prangle2020gk}
Prangle, D. (2020).
\newblock {gk: An R Package for the g-and-k and Generalised g-and-h
  Distributions}.
\newblock {\em {The R Journal}\/}~{\em 12\/}(1), 7--20.

\bibitem[\protect\citeauthoryear{Pulido}{Pulido}{2011}]{pulido2011semimartingales}
Pulido, S. (2011).
\newblock {Semimartingales and stochastic integration}.
\newblock {\em Lecture Notes, CMU,
  https://www.andrew.cmu.edu/user/calmost/pdfs/21-882-int\_lec.pdf\/}.

\bibitem[\protect\citeauthoryear{Russakovsky, Deng, Su, Krause, Satheesh, Ma,
  Huang, Karpathy, Khosla, Bernstein, Berg, and Fei-Fei}{Russakovsky
  et~al.}{2015}]{russakovsky2015imagenet}
Russakovsky, O., J.~Deng, H.~Su, J.~Krause, S.~Satheesh, S.~Ma, Z.~Huang,
  A.~Karpathy, A.~Khosla, M.~Bernstein, A.~C. Berg, and L.~Fei-Fei (2015).
\newblock {ImageNet Large Scale Visual Recognition Challenge}.
\newblock {\em International Journal of Computer Vision\/}~{\em 115\/}(3),
  211--252.

\bibitem[\protect\citeauthoryear{Schilling and Partzsch}{Schilling and
  Partzsch}{2012}]{schilling2012brownian}
Schilling, R.~L. and L.~Partzsch (2012).
\newblock {\em {Brownian Motion: An Introduction to Stochastic Processes}}.
\newblock De Gruyter.

\bibitem[\protect\citeauthoryear{Sohl-Dickstein, Battaglino, and
  Deweese}{Sohl-Dickstein et~al.}{2011}]{sohldickstein2011new}
Sohl-Dickstein, J., P.~B. Battaglino, and M.~R. Deweese (2011).
\newblock {New Method for Parameter Estimation in Probabilistic Models: Minimum
  Probability Flow}.
\newblock {\em Physical Review Letters\/}~{\em 107}.

\bibitem[\protect\citeauthoryear{Taylor}{Taylor}{2011}]{taylor2011partial}
Taylor, M.~E. (2011).
\newblock {\em {Partial Differential Equations I: Basic Theory}}.
\newblock Springer.

\bibitem[\protect\citeauthoryear{Tejero-Cantero, Boelts, Deistler, Lueckmann,
  Durkan, Gonçalves, Greenberg, and Macke}{Tejero-Cantero
  et~al.}{2020}]{tejero-cantero2020sbi}
Tejero-Cantero, A., J.~Boelts, M.~Deistler, J.-M. Lueckmann, C.~Durkan, P.~J.
  Gonçalves, D.~S. Greenberg, and J.~H. Macke (2020).
\newblock {sbi: A Toolkit for Simulation-based Inference}.
\newblock {\em Journal of Open Source Software\/}~{\em 5\/}(52), 2505.

\bibitem[\protect\citeauthoryear{Vincent, Larochelle, Bengio, and
  Manzagol}{Vincent et~al.}{2008}]{vincent2008extracting}
Vincent, P., H.~Larochelle, Y.~Bengio, and P.~A. Manzagol (2008).
\newblock {Extracting and Composing Robust Features with Denoising
  Autoencoders}.
\newblock {\em {ICML}\/}.

\bibitem[\protect\citeauthoryear{Yosida}{Yosida}{1965}]{yosida1965functional}
Yosida, K. (1965).
\newblock {\em {Functional Analysis}}.
\newblock Springer Science {\&} Business Media.

\end{thebibliography}


\begin{thebibliography}{}

\bibitem[\protect\citeauthoryear{Anderson}{Anderson}{1982}]{anderson1982reverse}
Anderson, B. D.~O. (1982).
\newblock {Reverse-time Diffusion Equation Models}.
\newblock {\em Stochastic Processes and their Applications\/}~{\em 12},
  313--326.

\bibitem[\protect\citeauthoryear{Austin, Johnson, Ho, Tarlow, and van~den
  Berg}{Austin et~al.}{2021}]{austin2021structured}
Austin, J., D.~D. Johnson, J.~Ho, D.~Tarlow, and R.~van~den Berg (2021).
\newblock {Structured Denoising Diffusion Models in Discrete State-Spaces}.
\newblock {\em NeurIPS\/}.

\bibitem[\protect\citeauthoryear{Bernton, Jacob, Gerber, and Robert}{Bernton
  et~al.}{2019}]{bernton2019approximate}
Bernton, E., P.~E. Jacob, M.~Gerber, and C.~P. Robert (2019).
\newblock {Approximate Bayesian Computation with the Wasserstein Distance}.
\newblock {\em Journal of the Royal Statistical Society: Series B (Statistical
  Methodology)\/}~{\em 81\/}(2), 235--269.

\bibitem[\protect\citeauthoryear{Brown, Mann, Ryder, Subbiah, Kaplan, Dhariwal,
  Neelakantan, Shyam, Sastry, Askell, et~al.}{Brown
  et~al.}{2020}]{brown2020language}
Brown, T., B.~Mann, N.~Ryder, M.~Subbiah, J.~D. Kaplan, P.~Dhariwal,
  A.~Neelakantan, P.~Shyam, G.~Sastry, A.~Askell, et~al. (2020).
\newblock {Language Models are Few-shot Learners}.
\newblock {\em NeurIPS\/}.

\bibitem[\protect\citeauthoryear{Campbell, Benton, De~Bortoli, Rainforth,
  Deligiannidis, and Doucet}{Campbell et~al.}{2022}]{campbell2022continuous}
Campbell, A., J.~Benton, V.~De~Bortoli, T.~Rainforth, G.~Deligiannidis, and
  A.~Doucet (2022).
\newblock {A Continuous Time Framework for Discrete Denoising Models}.
\newblock {\em NeurIPS\/}.

\bibitem[\protect\citeauthoryear{Chen, Chewi, Li, Li, Salim, and Zhang}{Chen
  et~al.}{2023}]{chen2022sampling}
Chen, S., S.~Chewi, J.~Li, Y.~Li, A.~Salim, and A.~R. Zhang (2023).
\newblock {Sampling is as Easy as Learning the Score: Theory for Diffusion
  Models with Minimal Data Assumptions}.
\newblock {\em ICLR\/}.

\bibitem[\protect\citeauthoryear{De~Bortoli}{De~Bortoli}{2023}]{debortoli2022convergence}
De~Bortoli, V. (2023).
\newblock {Convergence of Denoising Diffusion Models under the Manifold
  Hypothesis}.
\newblock {\em Transactions on Machine Learning Research\/}.

\bibitem[\protect\citeauthoryear{De~Bortoli, Mathieu, Hutchinson, Thornton,
  Teh, and Doucet}{De~Bortoli et~al.}{2022}]{debortoli2022riemannian}
De~Bortoli, V., E.~Mathieu, M.~Hutchinson, J.~Thornton, Y.~W. Teh, and
  A.~Doucet (2022).
\newblock {Riemannian Score-Based Generative Modeling}.
\newblock {\em NeurIPS\/}.

\bibitem[\protect\citeauthoryear{Dhariwal and Nichol}{Dhariwal and
  Nichol}{2021}]{dhariwal2021diffusion}
Dhariwal, P. and A.~Nichol (2021).
\newblock {Diffusion Models Beat GANs on Image Synthesis}.
\newblock {\em NeurIPS\/}.

\bibitem[\protect\citeauthoryear{Ethier and Griffiths}{Ethier and
  Griffiths}{1993}]{ethier1993transition}
Ethier, S.~N. and R.~C. Griffiths (1993).
\newblock {The Transition Function of a Fleming-Viot Process}.
\newblock {\em The Annals of Probability\/}~{\em 21}, 1571--1590.

\bibitem[\protect\citeauthoryear{Geffner, Papamakarios, and Mnih}{Geffner
  et~al.}{2023}]{geffner2023compositional}
Geffner, T., G.~Papamakarios, and A.~Mnih (2023).
\newblock {Compositional Score Modeling for Simulation-based Inference}.
\newblock {\em arXiv preprint arXiv:2209.14249\/}.

\bibitem[\protect\citeauthoryear{Goodfellow, Pouget-Abadie, Mirza, Xu,
  Warde-Farley, Ozair, Courville, and Bengio}{Goodfellow
  et~al.}{2014}]{goodfellow2014generative}
Goodfellow, I.~J., J.~Pouget-Abadie, M.~Mirza, B.~Xu, D.~Warde-Farley,
  S.~Ozair, A.~Courville, and Y.~Bengio (2014).
\newblock {Generative Adversarial Nets}.
\newblock {\em NeurIPS\/}.

\bibitem[\protect\citeauthoryear{Greenacre}{Greenacre}{2021}]{greenacre2021compositional}
Greenacre, M. (2021).
\newblock {Compositional Data Analysis}.
\newblock {\em Annual Review of Statistics and its Application\/}~{\em 8},
  271--299.

\bibitem[\protect\citeauthoryear{Gutmann and Hirayama}{Gutmann and
  Hirayama}{2011}]{gutmann2012bregman}
Gutmann, M.~U. and J.-i. Hirayama (2011).
\newblock {Bregman Divergence as General Framework to Estimate Unnormalized
  Statistical Models}.
\newblock {\em UAI\/}.

\bibitem[\protect\citeauthoryear{Ho, Jain, and Abbeel}{Ho
  et~al.}{2020}]{ho2020denoising}
Ho, J., A.~Jain, and P.~Abbeel (2020).
\newblock {Denoising Diffusion Probabilistic Models}.
\newblock {\em NeurIPS\/}.

\bibitem[\protect\citeauthoryear{Hoogeboom, Nielsen, Jaini, Forré, and
  Welling}{Hoogeboom et~al.}{2021}]{hoogeboom2021argmax}
Hoogeboom, E., D.~Nielsen, P.~Jaini, P.~Forré, and M.~Welling (2021).
\newblock {Argmax Flows and Multinomial Diffusion: Learning Categorical
  Distributions}.
\newblock {\em NeurIPS\/}.

\bibitem[\protect\citeauthoryear{Huang, Aghajohari, Bose, Panangaden, and
  Courville}{Huang et~al.}{2022}]{huang2022riemannian}
Huang, C.-W., M.~Aghajohari, A.~J. Bose, P.~Panangaden, and A.~Courville
  (2022).
\newblock {Riemannian Diffusion Models}.
\newblock {\em NeurIPS\/}.

\bibitem[\protect\citeauthoryear{Huang, Lim, and Courville}{Huang
  et~al.}{2021}]{huang2021variational}
Huang, C.-W., J.~H. Lim, and A.~Courville (2021).
\newblock {A Variational Perspective on Diffusion-Based Generative Models and
  Score Matching}.
\newblock {\em NeurIPS\/}.

\bibitem[\protect\citeauthoryear{Hyvärinen}{Hyvärinen}{2005}]{hyvarinen2005estimation}
Hyvärinen, A. (2005).
\newblock {Estimation of Non-Normalized Statistical Models by Score Matching}.
\newblock {\em Journal of Machine Learning Research\/}~{\em 6}, 695–709.

\bibitem[\protect\citeauthoryear{Hyvärinen}{Hyvärinen}{2007}]{hyvarinen2007some}
Hyvärinen, A. (2007).
\newblock {Some Extensions of Score Matching}.
\newblock {\em Computational Statistics and Data Analysis\/}~{\em 51}, 2499 –
  2512.

\bibitem[\protect\citeauthoryear{Kingma and Welling}{Kingma and
  Welling}{2014}]{kingma2014autoencoding}
Kingma, D.~P. and M.~Welling (2014).
\newblock {Auto-Encoding Variational Bayes}.
\newblock {\em ICLR\/}.

\bibitem[\protect\citeauthoryear{Lou and Ermon}{Lou and
  Ermon}{2023}]{lou2023reflected}
Lou, A. and S.~Ermon (2023).
\newblock {Reflected Diffusion Models}.
\newblock {\em ICML\/}.

\bibitem[\protect\citeauthoryear{Lueckmann, Boelts, Greenberg, Gonçalves, and
  Macke}{Lueckmann et~al.}{2021}]{lueckmann2021benchmarking}
Lueckmann, J.-M., J.~Boelts, D.~S. Greenberg, P.~J. Gonçalves, and J.~H. Macke
  (2021).
\newblock {Benchmarking Simulation-Based Inference}.
\newblock {\em AISTATS\/}.

\bibitem[\protect\citeauthoryear{Lyu}{Lyu}{2009}]{lyu2009interpretation}
Lyu, S. (2009).
\newblock {Interpretation and Generalization of Score Matching}.
\newblock {\em UAI\/}.

\bibitem[\protect\citeauthoryear{Mardia, Kent, and Laha}{Mardia
  et~al.}{2016}]{mardia2016score}
Mardia, K.~V., J.~T. Kent, and A.~K. Laha (2016).
\newblock {Score Matching Estimators for Directional Distributions}.
\newblock {\em arXiv preprint arXiv:1604.08470\/}.

\bibitem[\protect\citeauthoryear{Murphy, Esteves, Jampani, Ramalingam, and
  Makadia}{Murphy et~al.}{2021}]{murphy2021implicit}
Murphy, K.~A., C.~Esteves, V.~Jampani, S.~Ramalingam, and A.~Makadia (2021).
\newblock {Implicit-PDF: Non-Parametric Representation of Probability
  Distributions on the Rotation Manifold}.
\newblock {\em ICML\/}.

\bibitem[\protect\citeauthoryear{Nunes and Prangle}{Nunes and
  Prangle}{2015}]{nunes2015abctools}
Nunes, M.~A. and D.~Prangle (2015).
\newblock {abctools: An R Package for Tuning Approximate Bayesian Computation
  Analyses}.
\newblock {\em {The R Journal}\/}~{\em 7\/}(2), 189--205.

\bibitem[\protect\citeauthoryear{Oord, Dieleman, Zen, Simonyan, Vinyals,
  Graves, Kalchbrenner, Senior, and Kavukcuoglu}{Oord
  et~al.}{2016}]{oord2016wavenet}
Oord, A. v.~d., S.~Dieleman, H.~Zen, K.~Simonyan, O.~Vinyals, A.~Graves,
  N.~Kalchbrenner, A.~Senior, and K.~Kavukcuoglu (2016).
\newblock {WaveNet: A Generative Model for Raw Audio}.
\newblock {\em arXiv:1609.03499\/}.

\bibitem[\protect\citeauthoryear{Popov, Vovk, Gogoryan, Sadekova, and
  Kudinov}{Popov et~al.}{2021}]{popov2021grad}
Popov, V., I.~Vovk, V.~Gogoryan, T.~Sadekova, and M.~Kudinov (2021).
\newblock {Grad-tts: A Diffusion Probabilistic Model for Text-to-speech}.
\newblock {\em ICML\/}.

\bibitem[\protect\citeauthoryear{Rezende and Mohamed}{Rezende and
  Mohamed}{2015}]{rezende2015variational}
Rezende, D.~J. and S.~Mohamed (2015).
\newblock {Variational Inference with Normalizing Flows}.
\newblock {\em ICML\/}.

\bibitem[\protect\citeauthoryear{Richemond, Dieleman, and Doucet}{Richemond
  et~al.}{2022}]{richemond2022categorical}
Richemond, P.~H., S.~Dieleman, and A.~Doucet (2022).
\newblock {Categorical SDEs with Simplex Diffusion}.
\newblock {\em arXiv preprint arXiv:2210.14784\/}.

\bibitem[\protect\citeauthoryear{Saharia, Ho, Chan, Salimans, Fleet, and
  Norouzi}{Saharia et~al.}{2022}]{saharia2022image}
Saharia, C., J.~Ho, W.~Chan, T.~Salimans, D.~J. Fleet, and M.~Norouzi (2022).
\newblock {Image Super-Resolution via Iterative Refinement}.
\newblock {\em IEEE Transactions on Pattern Analysis and Machine
  Intelligence\/}, 1--14.

\bibitem[\protect\citeauthoryear{Sharrock, Simons, Liu, and Beaumont}{Sharrock
  et~al.}{2022}]{sharrock2022sequential}
Sharrock, L., J.~Simons, S.~Liu, and M.~Beaumont (2022).
\newblock {Sequential Neural Score Estimation: Likelihood-Free Inference with
  Conditional Score Based Diffusion Models}.
\newblock {\em arXiv preprint arXiv:2210.04872\/}.

\bibitem[\protect\citeauthoryear{Sohl-Dickstein, Weiss, Maheswaranathan, and
  Ganguli}{Sohl-Dickstein et~al.}{2015}]{sohldickstein2015deep}
Sohl-Dickstein, J., E.~A. Weiss, N.~Maheswaranathan, and S.~Ganguli (2015).
\newblock {Deep Unsupervised Learning Using Nonequilibrium Thermodynamics}.
\newblock {\em ICML\/}.

\bibitem[\protect\citeauthoryear{Song, Durkan, Murray, and Ermon}{Song
  et~al.}{2021}]{song2021maximum}
Song, Y., C.~Durkan, I.~Murray, and S.~Ermon (2021).
\newblock {Maximum Likelihood Training of Score-Based Diffusion Models}.
\newblock {\em NeurIPS\/}.

\bibitem[\protect\citeauthoryear{Song, Sohl-Dickstein, Kingma, Kumar, Ermon,
  and Poole}{Song et~al.}{2021}]{song2021score}
Song, Y., J.~Sohl-Dickstein, D.~P. Kingma, A.~Kumar, S.~Ermon, and B.~Poole
  (2021).
\newblock {Score-Based Generative Modeling through Stochastic Differential
  Equations}.
\newblock {\em ICLR\/}.

\bibitem[\protect\citeauthoryear{Sun, Yu, Dai, Schuurmans, and Dai}{Sun
  et~al.}{2023}]{sun2023score}
Sun, H., L.~Yu, B.~Dai, D.~Schuurmans, and H.~Dai (2023).
\newblock {Score-based Continuous-time Discrete Diffusion Models}.
\newblock {\em ICLR\/}.

\bibitem[\protect\citeauthoryear{Trippe, Yim, Tischer, Baker, Broderick,
  Barzilay, and Jaakkola}{Trippe et~al.}{2023}]{trippe2023diffusion}
Trippe, B.~L., J.~Yim, D.~Tischer, D.~Baker, T.~Broderick, R.~Barzilay, and
  T.~Jaakkola (2023).
\newblock {Diffusion Probabilistic Modeling of Protein Backbones in 3D for the
  Motif-scaffolding Problem}.
\newblock {\em ICLR\/}.

\bibitem[\protect\citeauthoryear{Vincent}{Vincent}{2011}]{vincent2011connection}
Vincent, P. (2011).
\newblock {A Connection Between Score Matching and Denoising Autoencoders}.
\newblock {\em Neural Computation\/}~{\em 23}, 1661--1674.

\bibitem[\protect\citeauthoryear{Yu, Drton, and Shojaie}{Yu
  et~al.}{2022}]{yu2022generalized}
Yu, S., M.~Drton, and A.~Shojaie (2022).
\newblock {Generalized Score Matching for General Domains}.
\newblock {\em Information and Inference: A Journal of the IMA\/}~{\em
  11\/}(2), 739--780.

\end{thebibliography}
